\documentclass[opre, nonblindrev, hyphens]{ssrn_arxiv}
\OneAndAHalfSpacedXI 

\usepackage{graphicx}
\usepackage{times}
\usepackage{natbib}
\bibpunct[, ]{(}{)}{,}{a}{}{,}%
\def\bibfont{\small}%
\usepackage{amsmath}
\usepackage{algpseudocode}
\usepackage{algorithm}
\usepackage{mathrsfs} 
\usepackage{caption}
\usepackage{mathtools}
\usepackage{booktabs}
\usepackage{appendix}
\usepackage[flushleft]{threeparttable}
\usepackage{float}
\usepackage{adjustbox}
\usepackage{varwidth}
\usepackage{tikz}
\usepackage{pgfplots}
\usepackage{braket}
\usepackage{enumerate}
\usepackage{subfigure}
\pgfplotsset{compat = 1.9}
\usepgfplotslibrary{external}
\usepackage{anyfontsize}

\usepackage{tikz-3dplot}
\tdplotsetmaincoords{60}{115}
\pgfplotsset{compat=newest}

\newcommand{\R}{\mathbb{R}}

\newcommand{\E}{\mathbb{E}}
\newcommand{\I}{\mathbb{I}}
\newcommand{\Z}{\mathbb{Z}}
\newcommand{\prob}{\mathbb{P}}
\newcommand{\cS}{\mathcal{S}}
\newcommand{\cC}{\mathcal{C}}
\newcommand{\cM}{\Delta\mathcal{M}}
\newcommand{\cR}{\mathcal{R}}
\newcommand{\cx}{\mathbf{x}^{\textup{H}}}
\newcommand{\cT}{\mathcal{T}}
\newcommand{\cE}{\mathcal{E}}
\newcommand{\cG}{\mathcal{G}}
\newcommand{\cH}{\mathcal{H}}
\newcommand{\sS}{\mathscr{S}}
\newcommand{\sM}{\mathscr{M}}
\newcommand{\bfe}{\mathbf{e}}
\newcommand{\bC}{\mathbf{C}}
\newcommand{\bV}{\mathbf{V}}
\newcommand{\bU}{\mathbf{U}}
\newcommand{\hS}{S^{\pi}}
\newcommand{\pS}{\hat{S}}
\newcommand{\hv}{\hat{v}}
\newcommand{\hu}{\hat{u}}
\newcommand{\av}{\Bar{v}}
\newcommand{\au}{\Bar{u}}
\newcommand{\Eu}{\mathbb{E}\big[\Bar{u}_{i, z}(m)\big]}

\newcommand{\Ev}{\mathbb{E}\big[\Bar{v}_{z, i}(t)\big]}

\newcommand{\ep}{\epsilon}
\newcommand{\tS}{\Tilde{S}}
\newcommand{\uM}{\Delta\widetilde{\mathcal{M}}}
\newcommand{\bR}{\gamma}

\usepackage[UKenglish,USenglish]{babel}
\usepackage{verbatim}
\TheoremsNumberedThrough     
\ECRepeatTheorems

\EquationsNumberedThrough    

\MANUSCRIPTNO{}

\begin{document}


\RUNAUTHOR{}

\RUNTITLE{Data-Driven Dynamic Assortment in Online Platforms: Learning about Two Sides}
\TITLE{Data-Driven Dynamic Assortment in Online Platforms: Learning about Two Sides}

\ARTICLEAUTHORS{%
\AUTHOR{Rahul Roy}
\AFF{IE Business School, IE University, Madrid 28006, \EMAIL{rroy@faculty.ie.edu} (RR)}
\AUTHOR{Nur Sunar, Jayashankar M. Swaminathan}
\AFF{Kenan-Flagler Business School, The University of North Carolina at Chapel Hill, NC  27599, \EMAIL{Nur\_Sunar@kenan-flagler.unc.edu} (NS); \EMAIL{msj@unc.edu} (JS)} 
} 

\ABSTRACT{We study a dynamic assortment problem on a two-sided service platform with incomplete information and heterogeneous customers in a discrete-time setting. In each period, a customer arrives seeking service, and the platform chooses an assortment of sellers to display. The customer then proposes a transaction to at most one seller in the assortment according to a multinomial logit choice model. After a fixed number of periods, sellers review the proposals they have received and each chooses at most one customer according to another multinomial logit choice model, after which the cycle repeats. A key challenge is that the platform does not know the choice-model parameters of either customers or sellers in advance. To our knowledge, this is the first study of a dynamic assortment problem in which both sides’ choice parameters are unknown. We develop a data-driven algorithm that learns these parameters while optimizing the platform’s objective over time. We evaluate performance using regret, which measures revenue loss relative to a clairvoyant benchmark that knows all parameters and customer arrivals in advance. We show that the algorithm’s worst-case regret grows polylogarithmically over time, and we derive a matching lower bound, establishing its rate optimality.

}


\KEYWORDS{assortment optimization; two-way learning; two-sided platform; multi-armed bandit; sequential decision making under uncertainty; data-driven decision making} 

\maketitle


%


\section{Introduction}
\label{introduction}


Assortment selection problems have been extensively studied in the literature, with applications ranging from physical retail environments to online platforms \textcolor{black}{\citep{rooderkerk2013optimizing, fisher2014demand, mitrofanov2024assortment, Zizhuo1, nageswaran2025buyer,  xie2025integrating, xu2025personalized}}. In the classical assortment selection problem, a decision maker selects a subset of products or services to maximize a total reward, such as revenue or profit. The total reward depends on customers' preferences over the offered items, which are typically captured through a choice model.

Two-sided online platforms are reshaping the landscape of various sectors. Unlike traditional marketplaces such as Amazon, participants on the demand side (e.g., customers), the supply side (e.g., providers), or both can be ``active'' on a two-sided platform. When participants are active, they must initiate interaction with agents on the other side of the platform to facilitate a transaction or match, which occurs only when both sides agree. Many online service marketplaces, such as EnergySage, HomeAdvisor, Freelancer, and Airbnb, are examples of such platforms with active participants.

Enhancing supply-side dynamics by offering an effective assortment of service providers is a key focus for two-sided online service platforms. Successful assortment decisions, which typically refer to the platform's choice of which service providers to showcase, are essential for maximizing transactions. For example, in EnergySage, the assortment may be interpreted as the set of solar installers shown to customers; in Freelancer and HomeAdvisor, as the set of providers displayed for a given project type; and in Airbnb, as the set of listings presented to travelers. Reports suggest that suppliers may decline service requests for various reasons, most prominently because they prefer not to serve certain customer types
\citep{edelman2017racial, levin2017airbnb, schall2024airbnb, airbnb2025decline, homeadvisor2025worth}.  In practice, rejection rates by service providers are substantial and can significantly affect platforms' financial performance. For example, using Airbnb as a context, \citet{fradkin2017search} found that in 2014, hosts rejected 42\% of guest requests, leading to a 40\% reduction in guests' transaction probability. Because providers are active on many two-sided platforms, unlike in traditional settings, understanding only customer preferences is not sufficient for making effective assortment decisions. The platform must also account for provider preferences and behaviors, since successful transactions depend not only on customer choice and supply availability but also on mutual agreement between customers and providers. A key challenge, however, is that the preferences of both customers and sellers are unknown to platform managers and must therefore be learned over time. Moreover, platforms often lack adaptive decision support systems that facilitate successful matches between supply- and demand-side participants when information about both sides' preferences is incomplete. This paper aims to bridge this gap by formulating a dynamic assortment selection problem faced by the manager of a two-sided online platform.

In our model, there are multiple customer types. In each period, a customer arrives at the platform, and the customer's type is unknown to the platform prior to arrival. In that same period, the platform manager chooses an \textit{assortment}, that is, a subset of sellers from a set of $N \in \mathbb{Z}_{+}$ sellers, to display to the customer. 
Customers are active, meaning that an arriving customer must initiate (propose) a transaction with a seller on the platform  to acquire services. This is a common feature of many marketplaces. For example, on Airbnb, guests actively reach out to hosts to book stays, and on Upwork, clients initiate contact with freelancers to request services. Upon observing the assortment, the customer proposes to at most one seller in the assortment, and the customer's choice follows a multinomial logit (MNL) model. At the end of every $K$ periods, sellers review the proposals they have received and each decides whether to match with at most one customer according to another MNL model; this cycle then repeats. In this setting, as in practice,  sellers have a specified time frame to match with a customer \citep{vrbo2025response, airbnb2025response}, and the platform manager does not filter out customer proposals, since these proposals represent potential business opportunities for both sellers and the platform; filtering them out would therefore risk losing transactions and revenue. The platform manager's objective is to maximize expected total reward over a planning horizon of $T$ periods. The challenge is that the platform does not know the choice-model parameters of either customers or sellers and must learn them dynamically over time. Our goal is therefore to devise an algorithm that minimizes the platform manager's \textit{regret} over the planning horizon $T$. Regret quantifies the platform's revenue loss relative to the clairvoyant benchmark, namely, the maximum achievable revenue if the platform knew all choice parameters and customer arrival sequences.

The literature on dynamic assortment selection problems with incomplete information has solely analyzed one-way learning scenarios, where the typical focus is to understand the preferences of customers, one side of market participants. However, that formulation assumes either non-existent or known preferences of the other side, like sellers or service providers. Our study diverges from this research stream by studying a dynamic assortment selection problem with two-way learning. Specifically, in our study, the decision maker (the platform manager) dynamically chooses assortments under incomplete information to learn about preference parameters of both customers and sellers while maximizing her total reward. To our knowledge, our paper is the first to examine a dynamic assortment selection problem with two-way  sequential learning. In the remainder of the manuscript, we refer to the described dynamic assortment selection problem with two-sided incomplete information as the platform's \emph{Two-Way Learning} problem or simply, \emph{TWL} problem. Moreover, we will use the terms \emph{policy} and \emph{algorithm} interchangeably throughout the manuscript.

Our model captures the customer and the seller preferences via MNL choice models. While assortment optimization problems have considered several choice models over the years to capture customer preferences, the MNL model \citep{mcfadden1978modeling} still remains one of the most popular parametric choice models in the literature due to its tractability, and therefore has been widely studied in the assortment optimization problems \citep{ryzin1999relationship, rusmevichientong2010dynamic, davis2013assortment, agrawal2019mnl, aznag2021mnl, miao2022online, aouad2023online, chen2024robust}. In fact, the vast majority of papers that study dynamic assortment under incomplete information about customers' choice parameters use an MNL choice model to capture customer preferences \citep{agrawal2019mnl, chen2020dynamic, aznag2021mnl, perivier2022dynamic, chen2024robust, sun2024unified}. Following the literature, we also use an MNL choice model in our formulation.

\subsection{Summary of Main Results and Contributions}
\label{contributions}

To our knowledge, our paper is the first to study a dynamic assortment selection problem where the decision maker dynamically chooses an assortment of sellers while learning about the choice parameters of both customers and sellers in an online platform. Formulating and analyzing this novel problem yield the following distinctive results.  

First, we develop an online algorithm for the TWL problem based on an upper confidence bound (UCB) approach. We show (in Theorem \ref{thm_regret_upper_bound}) that our algorithm's worst-case regret is $\mathcal{O}\big(\log^2(NT)\big)$, which is polylogarithmic in $T$.\footnote{Appendix \ref{notations} includes the formal definitions of the notations used to indicate the order of the regret in terms of $T$.} Our algorithm's polylogarithmic regret bound is a significant improvement over the worst-case regret results established in dynamic assortment selection problems with one-way learning in the literature. 

Specifically, studying a dynamic assortment selection problem with one-way learning and an MNL choice model, existing studies (\citet{agrawal2019mnl, aznag2021mnl}) identified algorithms with a worst-case regret of $\mathcal{\Tilde{O}}\big(\sqrt{T}\big).$\footnote{$\mathcal{\Tilde{O}}(\cdot)$ suppresses logarithmic factors, i.e, $\mathcal{\Tilde{O}}\big(g(T)\big) = \mathcal{O}\big(g(T)\text{polylog}(T)\big)$.} Overall, our study improves the worst-case regret bounds of these studies by identifying a polylogarithmic worst-case regret. This improvement is particularly noteworthy given the complexities inherent in the dynamic assortment selection problem with two-way learning, where assortment decisions must be made without knowing the preference parameters of either customers or sellers. Apart from this, we show (in Theorem \ref{thm_regret_lower_bound}) that there exists an instance of the TWL problem such that the lower bound of any admissible policy's worst-case regret is $\widetilde \Omega\big(\log^2(NT)\big)$. These matching upper and lower regret bounds establish the rate optimality\footnote{Appendix \ref{notations} includes the definition of rate optimality of an admissible policy.} of our algorithm.

In \S \ref{numerical_exp}, we present a simulation study to evaluate the performance of our algorithm for several scenarios of the TWL problem. We show that our algorithm consistently performs well in all scenarios. To assess how algorithms designed for dynamic assortment selection under a one-way learning framework perform in a two-way learning setup, we adapt the UCB-based algorithm from \citet{agrawal2019mnl} to our setting. We find that our algorithm performs substantially better than the one proposed by \citet{agrawal2019mnl} for all of the considered problem instances. This finding highlights that algorithms developed for one-way learning problems are ineffective in two-way learning scenarios, underscoring the need to develop new algorithms tailored to two-way learning frameworks. Furthermore, for a fixed set of sellers, we numerically study the relationship between the platform's total reward  and the upper bound of the assortment size, i.e., the maximum number of sellers that can be displayed to a customer. Our study reveals that increasing the upper bound of the assortment size beyond a threshold does not yield any considerable improvement in the platform's total reward. 

Finally, in \S \ref{extensions}, we analyze two extensions of our base model to assess the robustness of our algorithm's performance. In the first extension, we show (in Theorem \ref{thm_info_transparency}) that the performance of our algorithm is robust to the type and amount of customer-specific information the platform manager shares with the sellers. In our second extension, we establish (in Theorem \ref{thm_cust_reward}) that the performance of our algorithm is unaffected by whether the platform's reward is dependent or independent of the type of matched customer.

\subsection{Related Literature}
\label{literature}

There is a rich literature on assortment optimization across a wide range of contexts. Many studies in this literature focus on static optimization problems; see, e.g., \citet{ryzin1999relationship, talluri2004revenue, davis2013assortment, gallego2014constrained, li2015d, desir2022capacitated}. Our paper contributes to the stream on dynamic assortment optimization. In this stream, several studies examine settings in which customer preferences are unknown to the decision maker a priori and must be learned over time while maximizing total reward \textcolor{black}{\citep[e.g.,][]{rusmevichientong2010dynamic, saure2013optimal, agrawal2019mnl, aznag2021mnl, miao2022online, foussoul2023mnl, chen2024robust, li2025online}}. Existing work on dynamic assortment selection with incomplete information, spanning contexts such as retail product purchases and online advertising, has focused on one-way learning settings in which customer preferences are the only unknowns to be learned. Moreover, studies within this stream of literature focusing on one-way learning under an MNL choice model has developed algorithms with a worst-case regret of $\mathcal{\Tilde{O}}\big(\sqrt{T}\big)$. For instance, while \citet{aznag2021mnl} and \citet{miao2022online} prove a worst-case regret of $\mathcal{O}\big(\sqrt{T}\big)$, \citet{agrawal2019mnl} propose a UCB-based algorithm where they show that the upper bound on the regret of the algorithm is  $\mathcal{O}\big(\sqrt{NT\log(NT)}\big)$, where $N$ is the number of products. Similarly, \citet{agrawal2017thompson} develop another algorithm for the aforementioned one-way learning problem that achieves a worst-case regret bound of $\mathcal{O}\big(\sqrt{NT}\log(TK)\big)$, where $K \leq N$ is the upper bound on the assortment size. Our paper differs from this stream of literature in several important ways. First, we study a two-way learning problem in which an online platform dynamically learns the preferences of both customers and sellers. Second, unlike several prior studies that assume homogeneous customers (e.g., \citet{saure2013optimal, agrawal2017thompson, agrawal2019mnl, aznag2021mnl}), we allow for heterogeneous customer types. Methodologically, our work is most closely related to \citet{agrawal2019mnl, aznag2021mnl} and \citet{miao2022online}, whose policies follow an ``exploration--exploitation'' approach, rather than the ``explore--then--exploit'' approach used by \citet{rusmevichientong2010dynamic} and \citet{saure2013optimal}. However, unlike these studies, our model features a substantially more complex payoff-realization structure motivated by online service platforms in practice. In particular, platform payoffs are not realized immediately upon customer choice. Instead, the platform manager receives a payoff only when sellers make decisions at the end of each fixed-length recurring cycle. Consequently, in every period, the manager must optimize assortments under incomplete information and without immediate payoff feedback. This feature makes the problem significantly more challenging and necessitates a new approach to formulating and analyzing the TWL problem.
 
Overall, the two-way learning (TWL) problem we develop is fundamentally distinct from existing dynamic assortment selection problems with one-way learning, such as the MNL-Bandit problem studied by \citet{agrawal2019mnl}. Attempting to reduce the former to the latter, even under various assumptions and simplifications, proves infeasible due to the inherent complexity of the TWL problem. Even in scenarios involving homogeneous customers, known seller preferences, and single-period epochs, the payoff dynamics of the TWL problem differ significantly from those of the MNL-Bandit problem. In the latter, the platform receives immediate payoff following the customer’s decision. In contrast, in our context, although the customer extends a proposal to the seller, the payoff is contingent upon the seller's subsequent response. This adds a layer of stochasticity to the payoff calculation, as the outcome after the customer's decision is influenced by the uncertainty about the seller's reaction. Thus, even the simplest form of the TWL problem remains considerably more complex than the MNL-Bandit framework due to the interplay of customer and seller choices within a probabilistic framework.

\textcolor{black}{Our study is also related to the multi-armed bandit literature; see, e.g., \citet{robbins1952some, katehakis1987multi, auer2002finite, auer2002nonstochastic, slivkins2019introduction, bastani2020online}, and \citet{simchi2025simple} for important foundational work and illustrative examples}. Our work is more closely related to those by \citet{gai2012combinatorial, chen2013combinatorial, combes2015combinatorial, harrison2015investment, chen2016combinatorial}, and \citet{jourdan2021efficient} who consider a \emph{combinatorial bandit} framework. \textcolor{black}{However, unlike our study that proposes a novel two-way learning framework, the aforementioned studies examine the combinatorial bandit problem via a one-way learning framework.} Furthermore, unlike some studies who consider a linear reward function with respect to the unknown parameters in the problem \citep{rusmevichientong2010linearly, gai2012combinatorial}, our objective function is nonlinear in the unknown problem parameters: customers' and sellers' MNL parameters. 

While our primary focus is on dynamic assortment selection in two-sided platforms, the dynamics of learning in online markets, including two-sided markets, have attracted attention in various contexts \textcolor{black}{\citep[e.g.,][]{sarkar2021bandit, liu2021bandit, johari2021matching, li2022rate, wang2022bandit, avadhanula2022bandits, feng2023asymptotically, birge2024interfere, singhvi2025online}}. However, none of the prior work investigates a two-sided platform operator's two-way dynamic learning problem in a combinatorial framework as we do. \citet{sarkar2021bandit} adopt a multi-armed bandit (MAB) approach to devise investment strategies for a centralized lending platform, facilitating matches between lenders and borrowers. In their model, borrowers' preferences are assumed to be fixed, while lenders adapt based on matching outcomes in each period. \citet{liu2021bandit} explore a decentralized two-sided matching market, where players on one side of the market learn their preferences over time through repeated interactions. \citet{li2022rate} investigate a market  within an MAB framework, assuming known job-seeker preferences but unknown agent preferences. Furthermore, \citet{singhvi2025online} study sequential learning on online platforms where user preferences are unknown, explicitly addressing the sample selection bias that arises when outcomes are only observed for users who actively engage with a recommendation. Notably, these studies and others like \citet{johari2021matching} and \citet{wang2022bandit} primarily address one-sided uncertainty in preferences. 

\subsection{Organization}
\label{organization}

The remainder of our paper is structured as follows. \S \ref{problem} presents a formal description of the TWL problem, which is followed by an elaborate discussion of our proposed algorithm in \S \ref{policy}. \S \ref{upper_bound} presents our first key result, Theorem \ref{thm_regret_upper_bound}, along with supporting lemmas, and provides an outline of Theorem \ref{thm_regret_upper_bound}. \S \ref{lower_bound}, presents our second key result, Theorem \ref{thm_regret_lower_bound}, along with supporting lemmas.  \S \ref{numerical_exp} presents our simulation study that evaluates the performance of our algorithm across various instances of the two-way learning problem, and compares our algorithm's performance with that of another related UCB-based algorithm. Moreover, \S \ref{numerical_exp} investigates how the upper bound on the assortment size affects the platform's total reward. \S \ref{extensions} discusses two extensions of our algorithm and establishes its robustness under those extensions via Theorems \ref{thm_info_transparency} and \ref{thm_cust_reward}. Finally, \S \ref{conclusions}  includes concluding remarks.

\section{Problem Formulation}
\label{problem}

Consider a two-sided platform that dynamically selects a set of sellers---an \textit{assortment}---to be displayed to each arriving customers over $T$ periods. The set of sellers on the marketplace is $\cS \doteq \set{1, \ldots, N}$. At the beginning of each period $t \in \cT \doteq \set{1, \ldots, T}$, a customer arrives at the platform. There are different customers types, and the set of customer types is denoted by $\cC$.  Upon customer arrival, customer's type $c_t \in \cC$ is revealed to the platform manager. Yet, the sequence of customer arrivals is unknown to the platform a priori. Customer types can be defined in various ways. Practitioners typically form these types based on observable characteristics, such as zipcodes. Therefore, the set of customer types, $\cC$, is known to the platform manager. Observing $c_t$, the platform manager chooses an assortment of sellers $S_t \subseteq \cS$ and shows this assortment to the customer arrived in period $t$ ($t^{\text{th}}$ customer). The size of the assortment  offered to an arriving customer must satisfy $|S_t| \leq B \leq N$, where $B$ is a constant, implying a capacity constraint for the platform manager. Seeing the offered assortment $S_t$, the customer either extends a proposal to a seller in $S_t$ or leaves the platform without making any proposal. Customer $t$'s choice within assortment $S_t$ is a random variable $\zeta_t(S_t)$ such that  the probability that the customer $t$ of type $c_t = z$ chooses seller $i$ from $S_t$ at time $t$ is:
\begin{align}
    p_{z, i}(S_t) \doteq \prob(\zeta_t(S_t) = i \mid c_t = z) = 
    \begin{cases}
        \dfrac{v_{z, i}}{1 + \sum_{j \in S_t}v_{z, j}} & \text{for} \;\; i \in S_t,\\
        0 & \text{otherwise},
    \end{cases} \label{eq:proposal_prob}
\end{align}
where $v_{z, j}$ represents the choice parameter of customer type $z$ for seller $j$.  Note that customers' choices follow an MNL model. The caveat is the matrix of the MNL parameters of the customers, that is, $\bV \doteq [v_{z, i}]_{|\cC| \times N}$,  is unknown to the platform manager. 

Every $K \geq 1$ period, seller  $i \in \cS$ reviews all customer proposals received during that time. Consequently, we split the time horizon $T$ into $M$ \textit{epochs}, each consisting of $K$ periods. Without loss of generality, we assume $T= MK$. The set of epochs over $T$ periods is denoted by $\sM \doteq \set{1, \ldots, M}$. At the end of each epoch  $m$, all sellers assess the proposals received during that epoch (i.e., over the preceding $K$ periods) and decide whether to match with a customer. If a seller chooses to match with any customer in epoch $m \in \sM$, the platform receives a reward $\gamma_m$, where $\gamma_m \in (0, \bR]$ and $\bR \in \R_+$ (In \S \ref{customer_based_reward}, we allow the reward $\gamma_m$ to depend on both epoch $m \in \sM$ and customer type $z \in \cC$, i.e., $\gamma_m$ takes the form $\gamma_{m, z}$). Customer proposals received in epoch $m$ are lost before the start of the next epoch $m + 1$. Figure \ref{fig:timeline} depicts the sequence of events on the platform over $T = MK$ periods. 

\begin{figure}[htbp]
\begin{center}
\includegraphics[width=0.95\linewidth]{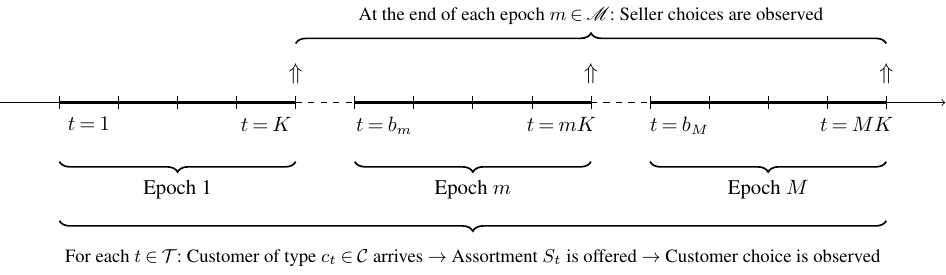}
\caption{\scriptsize Sequence of Events on Online Platform}
\label{fig:timeline}
\end{center}
\end{figure}

We now explain sellers' preferences. Define $b_m$ as the first period of each epoch $m$, i.e., $b_m \doteq (m-1)K + 1$. Take any time period $t \in \set{b_m, \ldots, mK}$ in epoch $m$  and let $\bC_i(t) \in \mathbb{Z}_{+}^{1\times|\cC|}$ be the row vector of customer proposals received by seller $i$ starting from $b_m$ up to $t$. Each coordinate of $\bC_i(t)$ corresponds to the number of proposals by a different customer type in $\cC$. Let $\bC(t) = (\bC_i(t))_{i \in \cS}$ be the matrix comprising of the vectors $\bC_i(t)$ for all $i \in \cS$. The matrix $\bC(t)$ is of size $N \times |\cC| $, and each row of $\bC(t)$ corresponds to a seller $i \in \cS$, while each column of $\bC(t)$ corresponds to a customer type in $\cC$. We henceforth refer to $\bC(t)$ as the \emph{proposal matrix} at time $t$.

At time $t$, suppose that a customer of type $c_t \in \cC$ arrives at the platform. If the customer sends a proposal to some seller $j \in S_t$, that only leads to a change in $\bC_{j}(t-1)$ at most by 1 at the coordinate $c_t$. Therefore, $\bC_{j}(t) = \bC_{j}(t-1) + \bfe_{c_t} \cdot \I\big[\zeta_t(S_t) = j\big]$, where $\bfe_{c_t}$ is the unit basis vector at coordinate $c_t$. Then, the proposal matrix $\bC(t) = \big(\bC_{i}(t)\big)_{i \in \cS}$ evolves as follows:
\begin{align}
    \bC_i(t) =  
    \begin{cases}
        \bC_{i}(t-1) + \bfe_{c_t} & \text{if} \;\; \I\big[\zeta_t(S_t) = i\big],\\
        \bC_{i}(t-1) & \text{otherwise}.
    \end{cases} \label{eq:state_transition}
\end{align}

Seller $i$'s choice within the vector $\bC_i(t)$ at time $t$ is represented by a random variable $\xi_i(\bC_i(t))$ such that  the probability that seller $i$ chooses a customer of type $c$ from $\bC_i(t)$ at time $t$ is $\phi_{i, c}(\bC_i(t))$. Define $n_{c, i}(t) \doteq (\bC_i(t))_{c}$ where $(\bC_i(t))_{c}$ represents the $c^{\text{th}}$ \textcolor{black}{coordinate} of $\bC_i(t)$. Then, we capture sellers' preferences using an MNL model, where $\phi_{i, c}(\bC_i(t))$ is given by: 
\begin{align}
    \phi_{i, c}(\bC_i(t)) \doteq \prob(\xi_i(\bC_i(t)) = c) = 
    \begin{cases}
        \dfrac{n_{c, i}(t)u_{i, c}}{1 + \sum_{z \in \cC}n_{z, i}(t)u_{i, z}} & \text{for} \;\; n_{c, i}(t)  > 0,\\
        0 & \text{otherwise}.
    \end{cases} \label{eq:single_match_prob}
\end{align}

In \eqref{eq:single_match_prob}, $u_{i, z}$ represents the preference parameter of seller $i$ for customer type $z$, and $n_{z, i}(t) = (\bC_i(t))_{z}$ represents the number of proposals received by seller $i$ from customer of type $z$ up to time $t$. As in the matrix of customers' choice parameters $\bV$, the platform manager also has no prior information about $\bU \doteq [u_{i, z}]_{N \times |\cC|}$, i.e., the matrix of parameters in sellers' choice models. Notably, the seller's choice probability, as defined by \eqref{eq:single_match_prob}, aligns with standard formulations in existing literature \citep{aouad2023online}.

The total reward received by the platform at the end of each epoch $m$ is given by:
\begin{align}
    \cR(m, \bU, \bV) &= \gamma_m\sum_{i \in \cS} \E_{\xi} \left[\I\big[\xi_i(\bC_i(mK)) \in \bC_i(mK) \big] \mid \bC_i(mK) \right] \nonumber\\
    &= \gamma_m\sum_{i \in \cS}\bigg\{\dfrac{\sum_{z \in \cC}n_{z, i}(mK)u_{i, z}}{1 + \sum_{z \in \cC}n_{z, i}(mK)u_{i, z}}\bigg\}, \label{eq:reward_epoch_simple} 
\end{align}
where $\I\big[\xi_i(\bC_i(mK)) \in \bC_i(mK)\big]$ indicates that seller $i$ matches with any customer within $\bC_i(mK)$ at the end of epoch $m$.

\textcolor{black}{We can simplify \eqref{eq:reward_epoch_simple} by introducing new functions $\cM(\cdot)$ and $\cM_{i}(\cdot)$ defined below. Recall that at the beginning of period $t$, seller $i$'s proposal vector is $\bC_i(t-1)$, and customer type $c_t$ is observed. Based on $\bC_i(t-1)$, seller $i$'s probability to match with a customer of type $z \in \cC$ is $\phi_{i, z}(\bC_i(t-1))$, while her total probability to match with any customer is $\phi_i(\bC_i(t-1)) = \sum_{z \in \cC}\phi_{i, z}(\bC_i(t-1))$.} Recall also that the platform offers an assortment $S_t$ to the arriving customer of type $c_t\in \cC$ in period $t$. Using these, and given the observed customer type $c_t$, we define:
\begin{align}
    &\cM_i(S_{t}, \bC(t-1), c_{t}) \doteq \E_{\zeta}\bigg[\phi_i\big(\bC_i(t-1) + \bfe_{c_{t}}\cdot\I\big[\zeta_t(S_t) = i\big]\big) - \phi_i\big(\bC_i(t-1)\big) \bigg] \label{eq:deltaM_i} \\
    &\cM(S_{t}, \bC(t-1), c_{t}) \nonumber\\
    &\doteq \sum_{i \in S_t}\cM_i(S_{t}, \bC(t-1), c_{t}) \nonumber \\  
    & = \sum_{i \in S_t}\E_{\zeta}\bigg[\phi_i\big(\bC_i(t-1) + \bfe_{c_{t}}\cdot\I\big[\zeta_t(S_t) = i\big]\big) - \phi_i\big(\bC_i(t-1)\big) \bigg]  \nonumber \\ 
    &= \sum_{i \in S_t}p_{c_t, i}(S_t)\bigg\{\phi_i\big(\bC_i(t-1) + \bfe_{c_{t}}\big) - \phi_i\big(\bC_i(t-1)\big)\bigg\}\nonumber \\
    &= \sum_{i \in S_t}\bigg[\dfrac{v_{c_t, i}}{1 + \sum_{j \in S_t} v_{c_t, j}}\bigg]\bigg[\dfrac{\sum_{z \in \cC}n_{z, i}(t-1)u_{i, z} + u_{i, c_t}}{1 + \sum_{z \in \cC}n_{z, i}(t-1)u_{i, z} + u_{i, c_t}} - \dfrac{\sum_{z \in \cC}n_{z, i}(t-1)u_{i, z}}{1 + \sum_{z \in \cC}n_{z, i}(t-1)u_{i, z}} \bigg] \nonumber \\
    &= \sum_{i \in S_t}\bigg[\dfrac{v_{c_t, i}u_{i, c_t}}{(1 + \sum_{j \in S_t} v_{c_t, j})(1 + \sum_{z \in \cC}n_{z, i}(t)u_{i, z})(1 + \sum_{z \in \cC}n_{z, i}(t-1)u_{i, z})}\bigg] \nonumber \\
    &= \sum_{i \in S_t}\bigg[\dfrac{v_{c_t, i}r_{i, c_t}(t)}{1 + \sum_{j \in S_t}v_{c_t, j}}\bigg], \label{eq:deltaM}
\end{align}
where 
\begin{align}
    r_{i, c_t}(t) \doteq \dfrac{u_{i, c_t}}{(1 + \sum_{z \in \cC}n_{z, i}(t)u_{i, z})(1 + \sum_{z \in \cC}n_{z, i}(t-1)u_{i, z})} <1. \label{eq:rit}
\end{align}
 
The term $\cM(S_t,\bC(t-1),c_t)$ can be interpreted as the marginal increase in expected matches at the end of period $t$, relative to expected matches at the end of period $t-1$, when assortment $S_t$ is offered at time $t$ to an arriving customer of type $c_t$, given proposal matrix $\bC(t-1)$. 

Lemma \ref{lemma_exp_reward} (in \S \ref{supplementary_lemmas}) establishes that the total reward in epoch $m$ equals $\gamma_m$ times $\cM(S_t,\bC(t-1),c_t)$ for all periods in the epoch. The total reward received by the platform in epoch $m$ is therefore given by:
\begin{align}
    \cR(m, \bU, \bV) = \gamma_m\sum_{t = b_m}^{mK} \cM(S_t, \bC(t-1), c_t).\label{eq:reward_epoch}
\end{align}

Recall that the platform manager does not have any prior knowledge of the preference parameters of both customers' and sellers' MNL models. Additionally, the sequence of customer arrivals is unknown to the platform manager a priori. In this setting, the platform manager adaptively offers an assortment of sellers to each arriving customer (of a certain type) over a planning horizon while dynamically learning the MNL parameters of both customers and sellers, with an objective to maximize the platform's total reward. We refer to this sequential assortment selection problem under incomplete information as the platform's \emph{Two-Way Learning} Problem, or simply \emph{TWL} problem.

\subsection{Clairvoyant Policy: A Dynamic Programming Approach}
\label{dynamic_prog}

Assume that there is a clairvoyant who knows $v_{z, i}$, $u_{i, z}$ for all $i \in \cS, z \in \cC$, and the sequence of customer arrivals $\set{c_t}_{t=b_m}^{mK}$ a priori for each epoch $m$. For a given epoch $m$, let $\bC(t-1) = \big(\bC_{i}(t-1)\big)_{i \in \cS}$ be the proposal matrix at time $t-1$. We define the platform's value function  at time $t$ in epoch $m$ under a given policy  as follows:
\begin{align}
    V_m\big(t, \bC(t-1)\big) &= \E\bigg[\gamma_m\sum_{l = t}^{mK}\cM(S_{l}, \bC(l-1), c_{l}) \mid t, \bC(t-1)\bigg] \nonumber\\
    &= \gamma_m \cM(S_{t}, \bC(t-1), c_{t}) + \E\bigg[V_m\big(t+1, \bC(t)\big) \mid t,  \bC(t-1)\bigg],
    \label{eq:value}
\end{align}
where the expectation is taken under the policy, and proposal matrix at $t$, given by $\bC(t)$, is a function of $\bC(t-1)$, $c_{t}$, and $S_{t}$, as defined in \eqref{eq:state_transition}. In \eqref{eq:value}, $V_m\big(t, \bC(t-1)\big)$ represents the platform's remaining reward measured at the beginning of period $t$ in epoch $m$ when the state is $\big(t, \bC(t-1)\big)$. Then, in the clairvoyant setting, the maximum value function at time $t$ (in epoch $m$) satisfies the following Bellman equation:
\begin{align}
    V_m^*\big(t, \bC(t-1)\big) &= \max_{S_{t} \in \sS} \Bigg\{\gamma_m\cM(S_{t}, \bC(t-1), c_{t}) + \E\bigg[V_m^*\big(t+1, \bC(t)\big) \mid t, \bC(t-1)\bigg]\Bigg\}, \label{eq:optimal_value}
\end{align}
where \(\sS \doteq \set{S \subseteq \cS \mid |S| \leq B}\) is the set of feasible assortments.  The term $ V_m^*$ is called the \textit{optimal clairvoyant value function},  and the policy \{$S^{\ast}_{t}, t=1,2,\ldots, T\}$ where  $S_t^*$ is defined in \eqref{eq:optimal_assortment_oracle} will be referred to as the optimal clairvoyant policy, or \textit{the clairvoyant policy}, in short.
\begin{align}
    S_t^* &= \argmax_{S_{t} \in \sS} V_m^*\big(t, \bC(t-1)\big) \nonumber\\ 
    &= \argmax_{S_{t} \in \sS} \Bigg\{\gamma_m\cM(S_{t}, \bC(t-1), c_{t}) + \E\bigg[V_m^*\big(t+1, \bC(t)\big) \mid t, \bC(t-1)\bigg]\Bigg\}.\label{eq:optimal_assortment_oracle}
\end{align}

Recall that at the end of an epoch, all unaccepted customer proposals are lost. Therefore, $V_m\big(mK+1, \bC(mK)\big) = 0$, implying that:
\begin{align}
    &V_m^*\big(mK, \bC(mK-1)\big) \nonumber\\ 
    &= \max_{S_{mK} \in \sS} \bigg\{\gamma_m\cM(S_{mK}, \bC(mK-1), c_{mK})\bigg\} \nonumber \\ 
    &= \max_{S_{mK} \in \sS} \left\{\sum_{i \in S_{mK}}\dfrac{\gamma_m v_{c_{mK}, i}u_{i, c_{mK}}}{(1 + \sum_{j \in S_{mK}}v_{c_{mK}, j})(1 + \sum_{z \in \cC}n_{z, i}(mK)u_{i, z})(1 + \sum_{z \in \cC}n_{z, i}(mK-1)u_{i, z})}\right\}.
    \label{eq:boundary_condition}
\end{align}

Using these, the clairvoyant can recursively solve \eqref{eq:optimal_value} to obtain $V_m^*\big(t, \bC(t-1)\big)$, and therefore $S_t^*$, for all $t$ in $m \in \sM$. 

\subsection{Performance Evaluation via Regret}
\label{regret}

\textcolor{black}{We now study the original incomplete information setting described in \S \ref{problem}.} Consider a period $t \in \set{b_m, \ldots, mK}$ within an epoch $m$. Let \(\{c_l\}_{l=1}^t\), \(\{\bC(l)\}_{l=1}^{t-1}\), \(\{\hS_l\}_{l=1}^{t-1}\), and \(\{\zeta_l(\hS_l)\}_{l=1}^{t-1}\) denote, respectively, the observed sequences of customer arrivals, proposal matrices, offered assortments, and customer choices under policy \(\pi\) up to and including the beginning of period \(t\). In addition, let $\left\{\left\{\xi_i\left(\bC_i(\Tilde{m}K)\right)\right\}_{i \in \cS}\right\}_{\Tilde{m}=1}^{m-1}$ denote the observed sequence of seller choices up to and including epoch $m-1$.  Then, the filtration associated with these sequences up to and including \textcolor{black}{the beginning of} period $t$ in epoch $m$ is given by:
\begin{align}
    \mathscr{H}_{t} = \left\{\mathscr{H}_{0}, \sigma\left(\left\{\left\{c_l\right\}_{l=1}^t, \left\{\bC(l), \hS_l, \zeta_l(\hS_l)\right\}_{l=1}^{t-1}, \left\{\left\{\xi_i\left(\bC_i(\Tilde{m}K)\right)\right\}_{i \in \cS}\right\}_{\Tilde{m}=1}^{m-1}\right\}\right)\right\}. \label{eq:filtration}
\end{align}
where $\mathscr{H}_{0}$ indicates any prior information available to the platform manager at time $t=0$, and $\sigma(\set{\cdot})$ represents the sigma-algebra set over the sequence $\set{\cdot}$. The sequence of functions $\pi = \set{\pi_1, \ldots \pi_T}$ is called an admissible policy when $\pi_t$ is adapted to $\mathscr{H}_{t}$, that is,
\begin{align}
    \pi_t : \mathscr{H}_{t} \rightarrow \set{\hS_t \subseteq \cS \mid |\hS_t| \leq B}. \label{eq:policy_def}
\end{align}

The value function under an admissible policy $\pi$ is given by:
\begin{align}
    V^{\pi}_m\big(t, \bC(t-1)\big) &= \E\bigg[\sum_{l = t}^{mK}\gamma_m\cM(\hS_{l}, \bC(l-1), c_{l}) \mid \mathscr{H}_t \bigg], \quad \forall t \in [(m-1) K +1, mK]. \label{eq:vf_policy_pi} 
\end{align}

Platform manager's objective is to find an admissible policy $\pi$ that maximizes the value function over $M$ epochs or equivalently $T$ periods. Recall that the notation $b_m$ refers to the first period in an epoch $m$, i.e., $b_m \doteq (m-1)K + 1$. Using \eqref{eq:vf_policy_pi}, the value function under the policy $\pi$ at the start of epoch $m$ (i.e., at time $t= b_{m}$) is given by:
\begin{align}
    V^{\pi}_m\big(b_m, \bC(b_m-1)\big)
    &= \E\left[\sum_{t = b_m}^{mK}\gamma_m \cM(\hS_t, \bC(t-1), c_t) \mid \mathscr{H}_{b_m}\right]. \label{eq:value_policy}
\end{align}

\textcolor{black}{Note that in \eqref{eq:value_policy}, the parameters $v_{z, i}$ and $u_{i, z}$ for all $i \in \cS, z \in \cC$ as well as the sequence of customer arrivals are unknown a priori.} We therefore can no longer use the clairvoyant policy to solve the problem. To identify a well-performing admissible policy $\pi$ under such incomplete information, we use the concept of \emph{regret} to evaluate policy performance. 
The regret represents the platform’s expected total revenue loss relative to the clairvoyant revenue — the maximum revenue achievable if the platform knew all choice parameters and customer arrival sequences. Our objective is to develop a policy $\pi$ that minimizes the regret over $T$ periods. By \eqref{eq:optimal_value} and \eqref{eq:value_policy}, the platform's regret of implementing the policy $\pi$ under incomplete information scenario is given by: 
\begin{align}
    Reg_{\pi}(T) &= \E\left[\sum_{m=1}^M\left\{ V_m^*\big(b_m, \bC(b_m-1)\big) - V^{\pi}_m\big(b_m, \bC(b_m-1)\big)\right\}\right] \nonumber \\
    &= \sum_{m=1}^M \E\left[\E\left[\sum_{t=b_m}^{mK}\gamma_m \cM(S_t^*, \bC(t-1), c_t) \mid b_m, \bC(b_m-1) \right]\right] \nonumber\\
    &- \sum_{m=1}^M \E\left[\E\left[\sum_{t=b_m}^{mK}\gamma_m \cM(\hS_t, \bC(t-1), c_t) \mid \mathscr{H}_{b_m} \right]\right] \nonumber \\
    &= \sum_{m=1}^M \E \left[\sum_{t = b_m}^{mK}\gamma_m \left\{\cM(S_t^*, \bC(t-1), c_t) - \cM(\hS_t, \bC(t-1), c_t)\right\}\right], \label{eq:regret_def}
\end{align}
where $S_t^*$ is the assortment offered in period $t$ under the optimal clairvoyant policy discussed in \S \ref{dynamic_prog}. 

Previous research on dynamic assortment selection under incomplete information has largely focused on settings in which the platform learns only customer preferences, without accounting for the preferences of other agents, such as sellers. As a result, the objective in this literature (e.g., expected revenue) is typically modeled solely as a function of unknown customer preferences, under the assumption that offering an assortment yields an immediately observable choice and realized payoff.

Our setting differs in ways that make existing stochastic MAB approaches for assortment selection unsuitable. Although recent work in assortment optimization (e.g., \cite{agrawal2019mnl, aznag2021mnl}) extends MAB frameworks to accommodate substitution effects and dependent arms, these methods remain designed for one-sided learning environments with immediate feedback. By contrast, our setting requires the platform to learn both customer and seller preferences, while seller-side outcomes are realized only at the end of each epoch.

This combination of two-sided learning and delayed feedback is incompatible with the exploration logic underlying existing assortment algorithms. In one-sided settings, a common strategy is to offer the same assortment repeatedly until a stopping condition is met, such as a customer choosing the outside option, in order to construct unbiased estimates of customer preferences. If this strategy were applied in our environment, the platform would maintain a fixed assortment over a sequence of customer arrivals within a $K$-period epoch. While this would facilitate learning about customer preferences for the displayed sellers, it would simultaneously halt learning on the seller side, because sellers outside the offered assortment would receive no visibility and hence no proposals.

This limitation becomes particularly costly at the end of the epoch. If the displayed sellers ultimately reject the accumulated proposals, the platform would discover that it had devoted an entire epoch to a poorly chosen assortment. It would then incur a double loss: failing to generate successful matches while also learning nothing about the acceptance behavior of excluded sellers. As a result, customer-side learning would outpace seller-side learning, substantially slowing the joint learning process and leading to inflated regret.

To overcome these inherent limitations, our approach departs from traditional MAB frameworks by explicitly modeling the structural dependencies of the two-sided platform under asynchronous feedback. By leveraging the concept of marginal increase in expected matches, introduced in \S \ref{problem}, we develop a novel online algorithm tailored for dynamic assortment selection with two-way learning. Specifically, our proposed algorithm enables the platform to dynamically estimate the underlying multinomial logit parameters of both customers and sellers, exploiting the inherent substitution effects to efficiently explore the action space. It bridges the feedback gap by utilizing granular, period-level customer choices while seamlessly integrating the delayed, epoch-level payoff realizations. This dual-layered updating process allows the platform to simultaneously minimize regret and refine its assortment strategy based on the maximum available information at any given time. We refer to our algorithm as ``Two-Way Learning via UCB," or ``TWL-UCB" for short. The next section provides a formal exposition of this algorithm.
 
\section{The TWL-UCB Algorithm}
\label{policy}

The TWL-UCB algorithm we design dynamically selects assortments and updates Upper Confidence Bound (UCB) estimates of the unknown choice parameters for both customers and sellers. The policy consists of three main steps. We next provide an overview of these steps and formalize each of them later in this section. In Step I, for each period, the platform offers an assortment that maximizes the estimated marginal increase in expected matches, leveraging UCB estimates of customer and seller preferences for that period. In Step II, for each period, after observing customer decisions, the platform updates UCB estimates for customers by tracking how often they send proposals to sellers when included in offered assortments. In Step III, at the end of each epoch, sellers who receive proposals decide whether to match with a customer, allowing the platform to update UCB estimates for sellers based on observed matching behavior.

These updates ensure that the algorithm dynamically refines its estimates of both customer and seller preferences, balancing exploration and exploitation to improve assortment decisions over time. The approach effectively adapts to uncertainty in preferences, enabling efficient and near-optimal assortment selection while minimizing regret. We now formalize Steps I through III mentioned above. The pseudo code of our policy, which also demonstrates these steps, is available in Algorithm \ref{algo:two-sided-mnl-bandit}.

Consider epoch $m \in \sM$ and let $t$ be any time period  in epoch $m$, i.e., $t \in \set{b_m, \ldots, mK}$.  Define $v_{z, i, t-1}^{\textup{UCB}}$ and $u_{i, z, m-1}^{\textup{UCB}}$ as the UCB estimates \textcolor{black}{of $v_{z, i}$ and $u_{i, z}$, respectively,}  at the ends of period $t-1$ and epoch $m-1$, respectively, for all $i \in \cS$ and $z \in \cC$. 

\smallskip 

\noindent \underline{\textbf{Step I: Offer Assortment $\hat{S}_{t}$.}} After observing the arriving customer type $c_{t}$ in period $t$, using  the most recently available UCB estimates $v_{c_{t}, i, t-1}^{\textup{UCB}}$ and $u_{i, c_{t}, m-1}^{\textup{UCB}}$, the platform offers the assortment $\hat{S}_{t}$ that maximizes $\uM(S_t, \bC(t-1), c_t )$, that is,  \textcolor{black}{the marginal increase in expected matches in period $t$ computed using the UCB estimates $v_{c_{t}, i, t-1}^{\textup{UCB}}$ and $u_{i, c_{t}, m-1}^{\textup{UCB}}$}, to the customer in period $t$.
Formally,
\begin{align}
    \pS_{t} & \doteq  \argmax_{S_{t} \in \sS}\bigg\{\gamma_m \uM(S_t, \bC(t-1), c_t )\bigg\}, \label{eq:optimistic_assortment} 
\end{align}
\begin{align}
    \uM(S_t, \bC(t-1), c_t) \doteq \sum_{i \in S_{t}}\dfrac{v_{c_{t}, i, t-1}^{\textup{UCB}}r_{i, c_{t}}^{\textup{UCB}}(b_m)}{1 + \sum_{i \in S_{t}}v_{c_t, i, t-1}^{\textup{UCB}}},  \label{eq:deltaM_ucb}
\end{align}
and 
\begin{align}
    r_{i, c_t}^{\textup{UCB}}(b_m) \doteq \dfrac{u_{i, c_t, m-1}^{\textup{UCB}}}{(1 + u_{i, c_t, m-1}^{\textup{UCB}})}.  \label{eq:ucb_reward}
\end{align}

Lemma \ref{lemma_exp_deltaM_monotone} (in \S \ref{supplementary_lemmas}) establishes that $\uM(\pS_t, \bC(t-1), c_t)$ serves as an upper bound on the marginal increase in expected matches under the clairvoyant policy, $\cM(S_t^*, \bC(t-1), c_t)$, when the UCB estimates upper-bound the true parameters. Lemma \ref{lemma_ucb} shows that the UCB estimates, $v_{z, i, t}^{\textup{UCB}}$ and $u_{i, z, m}^{\textup{UCB}}$, indeed upper-bound the true parameter values $v_{z, i}$ and $u_{i, z}$, respectively, for all $i \in \cS, z \in \cC$, \textcolor{black}{with high probabilities as $t$ grows}.  Thus, $\uM(\pS_t, \bC(t-1), c_t) \geq \cM(S_t^*, \bC(t-1), c_t)$ with a high probability over time.  

Note that \eqref{eq:optimistic_assortment} is a single-period capacitated assortment optimization problem under the MNL model, and there are efficient polynomial-time algorithms to solve it (see, e.g., \citet{rusmevichientong2010dynamic, davis2013assortment}). 

\smallskip 

\noindent \underline{\textbf{Step II: Update UCB Estimates for Customers.}} Suppose that $c_{t} = z \in \mathcal{C}$. Based on the assortment $\pS_t$, the arriving customer of type $z$ either sends a proposal to a seller or leaves the platform without contacting any seller in period $t$. Let $\omega_{i, z}(t)$ represent the number of times, up to and including period $t$, that the assortment offered based on \eqref{eq:optimistic_assortment} to a customer of type $z$ includes seller $i$. That is, 
\begin{align}
    \omega_{i, z}(t) &=   \sum_{l \leq t} \I\big[c_l = z, i \in \pS_l\big] . \label{eq:assortment_count}
\end{align}

Moreover, let $\hv_{z, i}(t)$ be the number of times a customer of type $z$ sends a proposal to seller $i$ up to and including period $t$. Formally,

\begin{align}
    \hv_{z, i}(t) &=  \sum_{l \leq t}\I\big[c_l = z,  \zeta_l(\pS_l) = i\big] \label{eq:proposal_count}.
\end{align}

Subsequently, we define $\av_{z, i}(t)$ as the fraction of time, up to and including period $t$, that a customer of type $z$ sends a proposal to seller $i$, considering only the periods during which seller $i$ is included in the assortment offered to type-$z$ customers.
\begin{align}
    \av_{z, i}(t) &= \dfrac{\hv_{z, i}(t)}{\omega_{i, z}(t)}. \label{eq:avg_proposal_count}
\end{align}

Using \eqref{eq:avg_proposal_count},  we update the UCB estimate for type-$z$ customer \textcolor{black}{at the end of} period $t$ 
, $v_{z, i, t}^{\textup{UCB}}$, as: 
\begin{align}
    v_{z, i, t}^{\textup{UCB}} &= \av_{z, i}(t) + \dfrac{C\log(Nt)}{\omega_{i, z}(t)} + \sqrt{\dfrac{C\log(Nt)}{\omega_{i, z}(t)}\av_{z, i}(t)},  \label{eq:vucb}
\end{align}
where $C \in \R_+$ \textcolor{black}{is a constant to be fixed before running the algorithm.} 
\smallskip 

\noindent \underline{\textbf{Step III:  Update UCB Estimates for Sellers when $t = mK$.}} \ At the end of epoch $m$ (i.e., $t = mK$), each seller who received at least one customer proposal in epoch $m$ decides whether to match with any customer or remain unmatched. Let $\mathscr{T}_{z, i}(m)$ be the set of all epochs up to and including epoch $m$ during which at least one customer of type $z \in \cC$ arrives and sends a proposal to seller $i$. We define $\theta_{z, i}(m)$ as the cardinality of $\mathscr{T}_{z, i}(m)$. Formally,
\begin{align}
    \mathscr{T}_{z, i}(m) &= \set{\Tilde{m} \leq m \mid \exists \; t \in \set{(\Tilde{m}-1)K + 1, \ldots, \Tilde{m}K} : c_t = z, \zeta_t(\pS_t) = i}, \label{eq:proposal set}\\
    \theta_{z, i}(m) &= |\mathscr{T}_{z, i}(m)|. \label{eq:proposal_set_size}
\end{align}

Let $\hu_{i, z}(m)$ be the number of times seller $i \in \cS$ matches with a customer of type $z \in \cC$ from the available customer proposals up to and including epoch $m$:
\begin{align}
    \hu_{i, z}(m) &= \sum_{\Tilde{m} \leq m}\I\big[\xi_i(\bC_i(\Tilde{m}K)) = z\big] \label{eq:match_count}.
\end{align}

Based on \eqref{eq:proposal_set_size} and \eqref{eq:match_count}, we define $\au_{i, z}(m)$ as the fraction of epochs, up to and including epoch $m$, in which seller $i$ matches with a customer of type $z$, considering only the epochs during which at least one customer of type $z \in \cC$ arrives and sends a proposal to seller $i$.
\begin{align}
    \au_{i, z}(m) &= \dfrac{\hu_{i, z}(m)}{\theta_{z, i}(m)}. \label{eq:avg_match_count}
\end{align}

Using \eqref{eq:avg_match_count} and for $C \in \R_+$, we update the UCB estimates for all sellers at the end of epoch $m$, $u_{i, z, m}^{\textup{UCB}}$, as: 
\begin{align} 
    u_{i, z, m}^{\textup{UCB}} &= \au_{i, z}(m) + \dfrac{C\log(Nm)}{\theta_{z, i}(m)} + \sqrt{\dfrac{C\log(Nm)}{\theta_{z, i}(m)}\au_{i, z}(m)}. \label{eq:uucb} 
\end{align}

\smallskip 

\noindent \textbf{Note on Estimates.} We establish the consistency of our mean estimators, $\av_{z,i}(t)$ and $\au_{i,z}(m)$, in Theorem \ref{thm_consistent_estimator}. The statement and proof of Theorem \ref{thm_consistent_estimator} are deferred to \S \ref{consistent_estimators}. In \eqref{eq:vucb} and \eqref{eq:uucb}, to define $v_{z,i,t}^{\textup{UCB}}$ and $u_{i,z,m}^{\textup{UCB}}$, we use the dynamic confidence radii
\[
\frac{C\log(Nt)}{\omega_{i,z}(t)}+\sqrt{\frac{C\log(Nt)}{\omega_{i,z}(t)}\,\av_{z,i}(t)}
\quad\text{and}\quad
\frac{C\log(Nm)}{\theta_{z,i}(m)}+\sqrt{\frac{C\log(Nm)}{\theta_{z,i}(m)}\,\au_{i,z}(m)},
\]
centered at $\av_{z,i}(t)$ and $\au_{i,z}(m)$, respectively. These confidence radii are sharper than the standard confidence radii commonly used in the literature, namely,
\[
\sqrt{\frac{C\log(Nt)}{\omega_{i,z}(t)}}
\quad\text{for } \av_{z,i}(t),
\qquad\text{and}\qquad
\sqrt{\frac{C\log(Nm)}{\theta_{z,i}(m)}}
\quad\text{for } \au_{i,z}(m).
\]

\begin{algorithm}
\caption{(TWL-UCB Algorithm)}\label{algo:two-sided-mnl-bandit}
\begin{algorithmic}[1]
\begin{small}
\State \textbf{initialization:} $v_{z, i, 0}^{\textup{UCB}} = u_{i, z, 0}^{\textup{UCB}} = 1$ for all $i \in \cS, z \in \cC$
\State \textbf{parameters:} $K \in \Z_+, M \in \Z_+, C \in \R_+$ 
\State set $t = 1, m = 1$
\While{$m \leq M$}
\While{$t \leq mK$}
\State observe $c_t$, and suppose that $c_t = z$ 
\State compute $\pS_t = \argmax_{S_t \in \sS}\bigg\{\sum_{i \in S_t}  \dfrac{\gamma_m v_{z, i, t-1}^{\textup{UCB}}u_{i, z, m-1}^{\textup{UCB}}}{(1 + \sum_{i \in S_t}v_{z, i, t-1}^{\textup{UCB}})(1 + u_{i, z, m-1}^{\textup{UCB}})}\bigg\}$
\State offer $\pS_t$ to $c_t = z$, and update $\omega_{i, z}(t) = \sum_{l \leq t}\I\big[c_l = z, i \in \pS_l\big]$
\State observe $\zeta_t(\pS_t)$, and update $\hv_{z, i}(t) = \sum_{l \leq t}\I\big[c_l = z, \zeta_l(\pS_l) = i\big]$ and $\av_{z, i}(t) = \dfrac{\hv_{z, i}(t)}{\omega_{i, z}(t)}$
\State update $v_{z, i, t}^{\textup{UCB}} = \av_{z, i}(t) + \dfrac{C\log(Nt)}{\omega_{i, z}(t)} + \sqrt{\dfrac{C\log(Nt)}{\omega_{i, z}(t)}\av_{z, i}(t)}$
\If{$t = mK$}
\State update $\mathscr{T}_{z, i}(m) = \set{\Tilde{m} \leq m \mid \exists \; t \in \set{(\Tilde{m}-1)K + 1, \ldots, \Tilde{m}K} : c_t = z, \zeta_t(\pS_t) = i} \; \forall i \in \cS, z \in \cC$
\State update $\theta_{z, i}(m) = |\mathscr{T}_{z, i}(m)|$, and observe $\xi_i(\bC_i(mK))$
\State update $\hu_{i, z}(m) = \sum_{\Tilde{m} \leq m}\I\big[\xi_i(\bC_i(\Tilde{m}K)) = z\big]$ and $\au_{i, z}(m) = \dfrac{\hu_{i, z}(m)}{\theta_{z, i}(m)}$ 
\State update $u_{i, z, m}^{\textup{UCB}} = \au_{i, z}(m) + \dfrac{C\log(Nm)}{\theta_{z, i}(m)} + \sqrt{\dfrac{C\log(Nm)}{\theta_{z, i}(m)}\au_{i, z}(m)}$
\State $t = t + 1$, $m = m + 1$
\Else
\State $t = t + 1$
\EndIf
\EndWhile
\EndWhile
\end{small}
\end{algorithmic}
\end{algorithm}

\section{Upper Bound on the Performance of  Proposed TWL-UCB Algorithm}
\label{upper_bound}

The remainder of the paper focuses on a setting with Assumption \ref{assumption:mnl}, a standard assumption about MNL parameters in the literature. Several studies have adopted this assumption \citep{agrawal2017thompson, agrawal2019mnl, aznag2021mnl, gao2021assortment, chen2024robust, zhu2026unified}. 

Assumption \ref{assumption:mnl} implies that no-proposal and no-match options are the most likely outcomes for customers and sellers, respectively. This assumption holds in various real-world scenarios, particularly in online platforms \citep{shopify2024abandonment, baymard2025abandonment}. As a result, it is widely used in the literature.

\begin{assumption}\label{assumption:mnl}
The MNL parameters $v_{z, i}$ and $u_{i, z}$ for all $i \in \cS$ and $z \in \cC$ satisfy the following:
\begin{enumerate}
    \item $v_{z, i} \leq v_{z, 0} = 1$.
    \item $u_{i, z} \leq u_{i, 0} = 1$.
\end{enumerate}
\end{assumption}

Based on this, Theorem \ref{thm_regret_upper_bound} identifies an upper bound on the regret of the TWL-UCB algorithm (Algorithm \ref{algo:two-sided-mnl-bandit}) defined in 
Section \ref{policy}.

\begin{theorem}[Upper Bound on Regret the TWL-UCB  Algorithm] \label{thm_regret_upper_bound}
For any $\bU$ and $\bV$, under Assumption \ref{assumption:mnl}, there exists a constant $\alpha \in \R_+$ such that the regret of the TWL-UCB algorithm  at the end of $T$ periods is bounded from above as:
\[Reg_{_{\textup{TWL-UCB}}}(T) \leq \alpha \log^2(NT).\]
\end{theorem}

\textcolor{black}{Theorem \ref{thm_regret_upper_bound} demonstrates that the increase in expected loss in revenue over time, under our proposed algorithm relative to the clairvoyant policy, grows at a rate no worse than $\alpha \log^2(NT)$. Thus, regret grows sublinearly in the horizon $T$. Specifically, the ratio of the regret upper bounds at horizons $T+1$ and $T$ is
\[
\frac{\log^2(N(T+1))}{\log^2(NT)},
\]
which converges to $1$ as $T$ becomes large.}

\subsection{Proof Outline for Theorem \ref{thm_regret_upper_bound}}
\label{ub_proof_outline}

In this section, we outline the steps involved in proving Theorem \ref{thm_regret_upper_bound}. The proofs of Theorem \ref{thm_regret_upper_bound}, along with all mentioned supporting lemmas in this section, are available in \S \ref{ec_thm_1}. 

To prove this theorem, we first establish two key lemmas: Lemmas \ref{lemma_ucb} and \ref{lemma_optimal_value_bound}. Lemma \ref{lemma_ucb} proves that, the estimates $v_{z, i, t}^{\textup{UCB}}$ and $u_{i, z, m}^{\textup{UCB}}$ serve as upper bounds on the true parameter values $v_{z, i}$ and $u_{i, z}$, respectively, with a high probability, for all $z \in \cC$ and $i \in \cS$. Furthermore, Lemma \ref{lemma_ucb} identifies upper bounds for $v_{z, i, t}^{\textup{UCB}}$ and $u_{i, z, m}^{\textup{UCB}}$ in terms of $\Eu$ and $\Ev$, respectively, with a high probability.

\begin{lemma}\label{lemma_ucb}
Consider any $t \in \set{(m-1)K + 1, \ldots, mK}$, where $m \in \set{1, \ldots, M}$. Corresponding to $t$, consider $\omega_{i, z}(t), \av_{z, i}(t), v_{z, i, t}^{\textup{UCB}}, \theta_{z, i}(m), \au_{i, z}(m)$, and $u_{i, z, m}^{\textup{UCB}}$ defined in \eqref{eq:assortment_count}, \eqref{eq:avg_proposal_count}, \eqref{eq:vucb}, \eqref{eq:proposal_set_size}, \eqref{eq:avg_match_count}, and \eqref{eq:uucb}, respectively. Then there exist a set $\{\beta_j\}_{j=1}^{4} \in \R_+$  such that the following results hold.
\begin{enumerate}
    \item[\emph{(A.1)}] $\prob\bigg(v_{z, i, t}^{\textup{UCB}} \geq v_{z, i}\bigg) \geq 1- \dfrac{2}{Nt}$ \ for all $i \in \cS, z \in \cC$.
    \item[\emph{(A.2)}] $\prob\bigg(v_{z, i, t}^{\textup{UCB}} - \Ev \leq \beta_1\dfrac{\log(Nt)}{\omega_{i, z}(t)} + \beta_2\sqrt{\dfrac{\log(Nt)}{\omega_{i, z}(t)}\Ev}\bigg) \geq 1- \dfrac{2}{Nt}$ \ for all $i \in \cS, z \in \cC$.
    \item[\emph{(B.1)}] $\prob\bigg(u_{i, z, m}^{\textup{UCB}} \geq u_{i, z}\bigg) \geq 1- \dfrac{2}{Nm}$ \ for all $i \in \cS, z \in \cC$.
    \item[\emph{(B.2)}] $\prob\bigg(u_{i, z, m}^{\textup{UCB}} - \Eu \leq  \beta_3\dfrac{\log(Nm)}{\theta_{z, i}(m)} + \beta_4\sqrt{\dfrac{\log(Nm)}{\theta_{z, i}(m)}\Eu} \bigg) \geq 1- \dfrac{2}{Nm}$ \ for all $i \in \cS, z \in \cC$.
\end{enumerate}
\end{lemma}

Using Lemma \ref{lemma_ucb}, we partition the sample space into two mutually exclusive and collectively exhaustive regimes to analyze the regret of the TWL-UCB algorithm: the high probability regime and the low probability regime. Identifying these probability regimes allows us to decompose the problem of upper bounding regret into bounding regret within each regime separately.  

To bound the regret in each epoch under the low probability regime, we utilize the upper bound on the difference between the total marginal increase in expected matches under the clairvoyant policy and our policy. Specifically, we compute the following quantity for each epoch \(m\) and show that it is upper bounded by a constant \textcolor{black}{times $1/m$:}
\begin{align*}
    \sum_{t=b_m}^{mK} \left\{\cM(S_t^*, \bC(t-1), c_t) - \cM(\pS_t, \bC(t-1), c_t)\right\} = \mathcal{O}\left(\dfrac{1}{m}\right). 
\end{align*}

When aggregated over all epochs in the low probability regime, this results in an overall regret bound of \(\mathcal{O}(\log T)\).   

Under the high probability regime,  using Lemmas \ref{lemma_rucb_lower_bound} and \ref{lemma_exp_deltaM_monotone} (in \S \ref{supplementary_lemmas}), we first derive an upper bound on  $\cM(S_t^*, \bC(t-1), c_t) - \cM(\pS_t, \bC(t-1), c_t)$ for each period \( t \). That is, we show that:
\[
\cM(S_t^*, \bC(t-1), c_t) - \cM(\pS_t, \bC(t-1), c_t) \leq \uM(\pS_t, \bC(t-1), c_t) - \cM(\pS_t, \bC(t-1), c_t).
\]

Lemma \ref{lemma_optimal_value_bound} presents an upper bound on $\uM(\pS_t, \bC(t-1), c_t) - \cM(\pS_t, \bC(t-1), c_t)$ for each $t$. When the upper bound in  Lemma \ref{lemma_optimal_value_bound} is aggregated over all time periods, the resulting upper bound on $\sum_{m=1}^M\sum_{t=b_m}^{mK}\gamma_m\left\{\uM(\pS_t, \bC(t-1), c_t) - \cM(\pS_t, \bC(t-1), c_t)\right\}$
is found to be \( \mathcal{O}(\log^2(NT)) \) in the high probability regime. 

Finally, we combine the upper bounds on the (conditional) regret in the low and high probability regimes to derive the overall regret over the planning horizon, thereby proving Theorem \ref{thm_regret_upper_bound}.

\begin{lemma}\label{lemma_optimal_value_bound}
For any epoch $m$ and time period $t \in \set{(m-1)K + 1, \ldots, mK}$, define 
\begin{align*}
&q(\omega_{i, c_t}(t), v_{c_t, i}) \doteq \bigg(\beta_1\dfrac{\log(Nt)}{\omega_{i, c_t}(t)} + \beta_2\sqrt{\dfrac{\log(Nt)}{\omega_{i, c_t}(t)}v_{c_t, i}}\bigg)  \quad  \text{and} \\
&q(\theta_{c_t, i}(m), u_{i, c_t}) \doteq \bigg(\beta_3\dfrac{\log(Nm)}{\theta_{c_t, i}(m)} + \beta_4\sqrt{\dfrac{\log(Nm)}{\theta_{c_t, i}(m)}u_{i, c_t}}\bigg),
\end{align*}
where  $\{\beta_j\}_{j=1}^{4} \in \mathbb{R}_+$ are as defined in Lemma \ref{lemma_ucb}, and $\omega_{i, c_t}(t)$ and $\theta_{c_t, i}(m)$ are defined in \eqref{eq:assortment_count} and \eqref{eq:avg_match_count}, respectively. Then, with probability of at least $1 - \dfrac{8}{m}$, we have
\begin{align*}
    &\uM(\pS_t, \bC(t-1), c_t) - \cM(\pS_t, \bC(t-1), c_t)\\ 
    &\leq \sum_{i \in \pS_t}\bigg(q(\omega_{i, c_t}(t), v_{c_t, i}) + q(\theta_{c_t, i}(m), u_{i, c_t}) + q(\omega_{i, c_t}(t), v_{c_t, i})q(\theta_{c_t, i}(m), u_{i, c_t})\bigg),
\end{align*}
where $\pS_t$, $\cM(S_t, \bC(t-1), c_t)$, and $\uM(S_t, \bC(t-1), c_t)$ are given by \eqref{eq:optimistic_assortment}, \eqref{eq:deltaM}, and \eqref{eq:deltaM_ucb}, respectively
\end{lemma}  

\section{Lower Bound on the Regret of Any Admissible Policy}
\label{lower_bound} 

\textcolor{black}{In this section, we establish a lower bound on the worst-case regret of any admissible policy $\pi$ for the TWL problem. To do so, we first define the worst-case (expected) regret of an admissible policy $\pi$.
}Suppose that $S_t^*$ and $\hS_t$ denote the assortments offered  at time $t$ under the clairvoyant policy and an admissible policy $\pi$, respectively.  Using \eqref{eq:regret_def}, we define the worst-case regret under policy 
$\pi$ as follows:
\begin{align}
    \sup_{(\bU, \bV)} Reg_{\pi, (\bU, \bV)}(T) = \sup_{(\bU, \bV)} \sum_{m=1}^M\E_{(\bU, \bV)}\left[\sum_{t=b_m}^{mK}\gamma_m \left\{\cM(S_t^*, \bC(t-1), c_t) - \cM(\hS_t, \bC(t-1), c_t)\right\} \right], \label{eq:wc_regret}
\end{align}
where $b_m \doteq (m-1)K + 1$. To find the lower bound on the worst-case regret, we first create an instance of the TWL problem  by constructing an adversarial parameterization. This instance imposes a capacity constraint of the form $|S_t| \leq B \leq N/2$, where $N$ denotes the number of sellers, and $B$ represents the upper bound on the size of the assortment $S_t$ for all $t$.

Theorem \ref{thm_regret_lower_bound} identifies a lower bound on the worst-case regret of any admissible policy $\pi$ that offers an assortment $\hS_t$ at each $t \in \cT$ under the adversarial instance of the TWL problem. 

\begin{theorem}[Lower Bound on Regret] \label{thm_regret_lower_bound}
There exist a constant $\mathcal{B} \in \R_+$ and an instance of the TWL problem with $B \leq N/2$ such that under Assumption \ref{assumption:mnl}, the worst-case regret of any admissible policy $\pi$ at the end of $T$ periods is lower bounded by $\mathcal{B} \log^2(NT)$. That is,
\begin{align}
    \sup_{(\bU, \bV)} Reg_{\pi, (\bU, \bV)}(T) \geq \mathcal{B} \log^2(NT).
\end{align}
\end{theorem}

\textcolor{black}{Theorem \ref{thm_regret_lower_bound} shows that no policy designed for the TWL problem can achieve a regret smaller than $\mathcal{B} \log^2 (NT)$ in the worst-case. This lower bound serves as a fundamental benchmark for evaluating the performance of any policy, as no policy can outperform it across all possible parameterizations. Since the upper bound on regret for our TWL-UCB algorithm (Theorem \ref{thm_regret_upper_bound}) aligns with this benchmark, up to constant factors, it confirms that the algorithm we design is \emph{rate-optimal}.}

\subsection{Proof Outline for Theorem \ref{thm_regret_lower_bound}}
\label{lb_proof_outline}

To establish the lower bound on the worst-case regret of any policy, we first create an instance of the TWL problem, where we also construct an adversarial parameterization $\Theta_{\overline{S}}$, as will be explained in \eqref{eq:param_space}. {\color{black} For this problem instance, we identify the ``single best assortment" that maximizes the total reward over the entire planning horizon, given the complete knowledge of the customer and provider preferences and the sequence of customer arrivals. In other words, the ``single best assortment" is that one assortment that returns the highest total reward over the planning horizon under the clairvoyant environment. We show that this ``single best assortment" must be of maximum capacity $B$. We refer to this policy of offering a ``single best assortment" for the entire planning horizon as the ``single best assortment policy." Next, we use this ``single best assortment" to transform the problem of lower bounding the worst-case regret of any policy $\pi$ relative to the clairvoyant policy into the problem of lower bounding its regret relative to the ``single best assortment policy," as will be shown in \eqref{eq:weak_regret_bound}. We henceforth refer to the latter regret as the \emph{weak regret} \citep{auer2002nonstochastic}, as will be defined in \eqref{eq:weak_regret}. Finally, we derive a lower bound on the weak regret under the constructed problem instance to prove Theorem \ref{thm_regret_lower_bound}. Below we delineate the construction of the problem instance and the proof outline for Theorem \ref{thm_regret_lower_bound}.
}

\smallskip

\noindent \textbf{Adversarial Parameterization and Problem Instance.}  
Choose any arbitrary $\overline{S} \subseteq \cS$. Let $\bV_{\overline{S}}$ denote the matrix of choice parameters for all customer types $z \in \cC$, and let $\bU_{\overline{S}}$ denote the matrix of choice parameters for all sellers $i \in \cS$. \textcolor{black}{Then $v_{z, i} \in \bV_{\overline{S}}$ for all $z \in \cC$ and $u_{i, z} \in \bU_{\overline{S}}$ for all $i \in \cS$ are given by: }  
\begin{align}
    &v_{z, i}= 
    \begin{cases}
        v(1+\ep_z) & \text{if} \;\; i \in {\overline{S}},\\
        v & \text{if} \;\; i \notin {\overline{S}},
    \end{cases}\label{eq:param_space}\\
    &u_{i, z} = u \;\; \text{for all} \;\; z \in \cC, \nonumber
\end{align}
where $v \in (0, 0.5]$, $\ep_z \in (0, 1]$ for all $z \in \cC$, and $u \in (0, 1]$. We henceforth denote the parameterization $(\bU_{\overline{S}}, \bV_{\overline{S}})$ by $\Theta_{\overline{S}}$. \textcolor{black}{Under this parameterization, $\overline{S}$ is the ``single best assortment" (as formally proved in Lemma \ref{lemma_optimal_assortment_size} in \textcolor{black}{\S \ref{supplementary_lemmas}})}. The probability and expectation under $\Theta_{\overline{S}}$ are denoted by $\prob_{\overline{S}}$ and $\E_{\overline{S}}$, respectively. In addition, we set $K = 1$ such that $T = M$. We consider $\gamma_m$ in \eqref{eq:reward_epoch} to be of the following form:
\begin{align}
    \gamma_m = \left(1+u\right). \label{eq:gamma_epoch}
\end{align}

\textcolor{black}{Intuitively, this parameterization acts as an adversary by creating a ``needle in a haystack" environment designed to maximally confuse any learning policy. By setting the choice parameters of the sellers in $\overline{S}$ to be only marginally better than those in other assortments (controlled by an \textcolor{black}{arbitrarily} small $\epsilon_z$), the adversary makes $\overline{S}$ statistically nearly indistinguishable from any suboptimal assortment $S \neq \overline{S}$. This \textcolor{black}{infinitesimal} gap is deliberate: if $\ep_z$ were significantly higher, a policy could quickly identify the suboptimal assortments and cease exploration. Instead, by keeping the difference infinitesimal, the adversary forces any algorithm to spend an excessive amount of time, and thereby incur substantial worst-case regret, repeatedly testing suboptimal assortments just to gather enough data to confidently detect the true optimal set.}

{\color{black} For ease of exposition in the remainder of this section, we denote the marginal increase in expected matches corresponding to any assortment $S \subseteq \cS$ as:
\begin{align}
\cM(S) \doteq \cM(S, \bC(t-1), c_t)  = \sum_{i \in S}\cM_i(S, \bC(t-1), c_t). \label{eq:delta_M_simplified}
\end{align}
}

\noindent \textbf{Lower Bound on Worst-Case Regret.}  Consider any admissible policy $\pi$. Let $\hS_t \in \sS$ be the assortment offered in period $t$ under the admissible policy $\pi$, where $\sS \doteq \set{S \subseteq \cS \mid |S| \leq B}$ is the set of feasible assortments. Let $Reg_{\pi, \Theta_{\overline{S}}}(T)$ denote the weak regret of policy $\pi$ under the specific parameterization $\Theta_{\overline{S}}$. As defined earlier in \S\ref{lb_proof_outline}, this regret is measured relative to the best single assortment, that is, the assortment that maximizes the total reward over the planning horizon among all time-invariant assortments:
\begin{align}
    Reg_{\pi, \Theta_{\overline{S}}}(T) &= \max_{S \in \sS}\sum_{m=1}^M \gamma_m \E_{\overline{S}}\left[\sum_{t=b_m}^{mK} \left\{\cM(S) - \cM(\hS_t)\right\}\right],\label{eq:weak_regret}
\end{align}
where $\cM(\cdot)$ is defined in \textcolor{black}{\eqref{eq:delta_M_simplified}}. 

As established in Lemma \ref{lemma_optimal_assortment_size} (\textcolor{black}{in \S \ref{supplementary_lemmas}}), under the parameterization $\Theta_{\overline{S}}$, the single best assortment that maximizes the total reward over the planning horizon among all time-invariant assortments is $\overline{S}$. Moreover, Lemma \ref{lemma_optimal_assortment_size} shows that ${\overline{S}}$ must be of capacity $B$. Consequently, we can resolve the maximization problem and express the weak regret directly as:
\begin{align}
    Reg_{\pi, \Theta_{\overline{S}}}(T) = \sum_{m=1}^M \gamma_m \E_{\overline{S}}\left[\sum_{t=b_m}^{mK} \left\{\cM(\overline{S}) - \cM(\hS_t)\right\}\right]. \label{eq:weak_regret_v2}
\end{align}

\textcolor{black}{Define $\sS_B \doteq \set{S \subseteq \cS \mid |S| = B}$ as the set of all assortments $S$ of size $B$. Since $\sup_{(\bU, \bV)} Reg_{\pi, (\bU, \bV)}(T)$ represents the worst-case regret relative to the clairvoyant policy across all possible parameterizations, it must be lower-bounded by the maximum regret over our constructed family of adversarial parameterizations $\set{\Theta_{\overline{S}} \mid \overline{S} \in \sS_B}$.} Therefore, $\sup_{(\bU, \bV)} Reg_{\pi, (\bU, \bV)}(T)$ can be lower bounded by the maximum weak regret:
\begin{align}  
    \sup_{(\bU, \bV)} Reg_{\pi, (\bU, \bV)}(T) \geq \max_{\overline{S} \in \sS_B} Reg_{\pi, \Theta_{\overline{S}}}(T) & = \max_{\overline{S} \in \sS_B} \sum_{m=1}^M\gamma_m\E_{\overline{S}}\left[\sum_{t=b_m}^{mK}\bigg\{\cM({\overline{S}}) - \cM(\hS_t)\bigg\} \right]. \label{eq:weak_regret_bound}
\end{align}

 Clearly, the cardinality of the set $\sS_B$ is: $|\sS_B| = {N \choose B}$. For each $\hS_t$, we can define an assortment $S_t \in \sS_B$ such that $S_t \supseteq \hS_t$. By construction and  \textcolor{black}{the monotonicity of the $\cM(S)$ function with respect to $|S|$ as shown in the proof of Lemma \ref{lemma_inst_regret_suboptimal}}, the sequence $\set{\hS_t}_{t\leq T}$ suffers a higher regret than $\set{S_t}_{t \leq T}$. Therefore, we get:
\begin{align}
    \sup_{(\bU, \bV)} Reg_{\pi, (\bU, \bV)}(T) \geq \max_{\overline{S} \in \sS_B} Reg_{\pi, \Theta_{\overline{S}}}(T) &= \max_{\overline{S} \in \sS_B}\sum_{m=1}^M\gamma_m\E_{\overline{S}}\left[\sum_{t=b_m}^{mK}\bigg\{\cM({\overline{S}}) - \cM(\hS_t)\bigg\} \right]\nonumber \\ 
    &\geq \max_{\overline{S} \in \sS_B}\sum_{m=1}^M\gamma_m\E_{\overline{S}}\left[\sum_{t=b_m}^{mK}\bigg\{\cM({\overline{S}}) - \cM(S_t)\bigg\} \right]. \label{eq:weak_regret2}
\end{align}   

To find the lower bound in \eqref{eq:weak_regret2}, first Lemma \ref{lemma_inst_regret_suboptimal} derives a lower bound on the period-$t$ (one period) regret due to offering assortment $S_t \in \sS_B$ in period $t$ under the parameterization $\Theta_{\overline{S}}$. 

\begin{lemma}\label{lemma_inst_regret_suboptimal}
    Let $\overline{S} \in \sS_B$ be an arbitrary assortment, and $\Theta_{\overline{S}}$ be the corresponding parameterization. Suppose that $c_t = z \in \cC$ is the arriving customer type at time $t$ in epoch $m$. Then for any assortment $S_t \in \sS_B$, we have the following:
    \[\gamma_m\left(\cM(\overline{S}) - \cM(S_t)\right) \geq \dfrac{uv\ep_z(B - |S_t \cap \overline{S}|)}{(1 + B)^2}, \;\; \text{where $\ep_z \in (0, 1]$}.\]
\end{lemma}

{\color{black} We note that the inequality in Lemma \ref{lemma_inst_regret_suboptimal} is a deterministic statement. As established in the proof of Lemma \ref{lemma_optimal_assortment_size}, our constructed adversarial parameterization forces the reward parameter $r_{i, c_t}(t)$, defined in \eqref{eq:rit}, to be constant over time. In particular, we have $r_{i, c_t}(t) = \dfrac{u}{1+u}$ for all $t$ under $\Theta_{\overline{S}}$. Consequently, the marginal increase in expected matches completely loses its dependence on the state vector $\bC(t-1)$, collapsing into a deterministic, static function for any fixed assortment and customer type.}

{\color{black}
Next, we use Lemma \ref{lemma_inst_regret_suboptimal} to connect the weak regret to the expected number of times specific sellers are offered. Notice that the size of the intersection can be rewritten as a sum of indicator variables: $|S_t \cap \overline{S}| = \sum_{i \in \overline{S}} \I[i \in S_t]$. Consequently, summing the instantaneous regret bound over an epoch yields:
\begin{align*}
    \sum_{t=b_m}^{mK} (B - |S_t \cap \overline{S}|) = KB - \sum_{i \in \overline{S}} \sum_{t=b_m}^{mK} \I[i \in S_t] = KB - \sum_{i \in \overline{S}} \eta_i(m),
\end{align*}
where $\eta_i(m) \doteq \sum_{t=b_m}^{mK}\I[i \in S_t]$ is the number of times seller $i$ is offered in epoch $m$. This algebraic equivalence dictates that to lower bound the regret, we must upper bound $\E_{\overline{S}}[\eta_i(m)]$ for the sellers $i \in \overline{S}$. To bound $\E[\eta_i(m)]$, we introduce a neighboring parameter space $\Theta_{\tS}$ as discussed next.} 

Define $\sS_{B-1}^{-i} \doteq \set{\tS \subset \cS \mid i \notin \tS, |\tS| = B-1}$ as the set of all assortments of size $B-1$ such that seller $i$ is not included in the assortment. Thus, assortments ${\overline{S}} \in \sS_B$ and $\tS \in \sS_{B-1}^{-i}$ differ only in seller $i$, i.e., $\tS = {\overline{S}} \setminus \set{i}$. Let $\Theta_{\tS}$ be the parameterization corresponding to each $\tS \in \sS_{B-1}^{-i}$. This means that $\Theta_{\tS}$ represents a \emph{neighboring} parameter space that differs from the parameter space $\Theta_{\overline{S}}$ for ${\overline{S}} \in \sS_B$ for exactly one seller $i$. Using these constructs, Lemma \ref{lemma_pinskers_ineq} shows that the difference between the expected number of times a seller $i \in \cS$ has been offered in an assortment under neighboring parameterizations $\Theta_{\overline{S}}$ and $\Theta_{\tS}$ in each epoch can be upper bounded by the Kullback-Leibler (\textup{KL}) divergence between the corresponding distributions, $\prob_{\overline{S}}$ and $\prob_{\tS}$.

\begin{lemma}\label{lemma_pinskers_ineq}
    Suppose that ${\overline{S}} \in \sS_B$ and $\tS \in \sS_{B-1}^{-i}$ such that $\tS = {\overline{S}} \setminus \set{i}$, where $i \in \cS$. Let $\Theta_{\overline{S}}$ and $\Theta_{\tS}$ be the parameter spaces corresponding to ${\overline{S}}$ and $\tS$, respectively. Moreover, let $\eta_i(m) \doteq \sum_{t=b_m}^{mK}\I[i \in S_t]$  denote the number of times a seller $i \in \cS$ has been included in an assortment $S_t \in \sS_B$ for all $t$ within an epoch $m$. 
    Then, we have:
    \[\bigg|\E_{\tS}[\eta_{i}(m)] - \E_{\overline{S}}[\eta_{i}(m)]\bigg| \leq K\sqrt{2\textup{KL}_m\big(\prob_{\tS}\parallel \prob_{\overline{S}}\big)} = K\sqrt{2\textup{KL}_m\big(\prob_{\tS} \parallel \prob_{\tS \cup \set{i}}\big)}, \]
    where $\textup{KL}_m\big(\prob_{\tS} \parallel \prob_{\overline{S}}\big)$ is KL divergence between $\prob_{\tS}$ and $\prob_{\overline{S}}$ in epoch $m$.
\end{lemma}

\begin{lemma}\label{lemma_kl_divergence_neighbor}
    Suppose that ${\overline{S}} \in \sS_B$ and $\tS \in \sS_{B-1}^{-i}$ such that $\tS = {\overline{S}} \setminus \set{i}$, where $i \in \cS$. Let $\Theta_{\overline{S}}$ and $\Theta_{\tS}$ be the parameter spaces corresponding to ${\overline{S}}$ and $\tS$, respectively. If $\prob_{\overline{S}}$ and $\prob_{\tS}$ are probability distributions under $\Theta_{\overline{S}}$ and $\Theta_{\tS}$, respectively, then the \textup{KL} divergence between $\prob_{\tS}$ and $\prob_{\overline{S}}$ for each epoch $m$ is given by: 
    \[\textup{KL}_m\big(\prob_{\tS} \parallel \prob_{\overline{S}}\big) = \textup{KL}_m\big(\prob_{\tS} \parallel \prob_{\tS \cup \set{i}}\big) \leq \dfrac{(B+1)(2B + 9)}{4}\max_{z \in \cC}\ep_{z}^2.\]
\end{lemma}

Combining Lemmas \ref{lemma_pinskers_ineq} and \ref{lemma_kl_divergence_neighbor}, we find an upper bound on the expected number of times a seller $i \in \cS$ has been offered in an assortment. Finally, we appropriately set $u, v, $ and $\ep_z$ for all $z \in \cC$ to complete the proof of Theorem \ref{thm_regret_lower_bound}. The details of the proof of Theorem \ref{thm_regret_lower_bound} and the supporting lemmas are available in \S \ref{ec_thm_2}. 

\section{Simulation Studies}
\label{numerical_exp}

In this section, we conduct a simulation study for several instances of the TWL problem. For our simulation study, we construct four different instances of the preference parameters $(\bU, \bV)$, and compare the performance of our TWL-UCB algorithm to that of the UCB-based algorithm from \citet{agrawal2019mnl} (Algorithm 1 in \citet{agrawal2019mnl}) for all instances of the preference parameters. In all these instances, we observe that our proposed algorithm performs better than the one proposed in \citet{agrawal2019mnl}. Finally, we numerically evaluate how the value function under our proposed algorithm varies with the upper bound on the assortment size when MNL parameters are drawn from a uniform distribution. We observe that the value function is sublinear in the upper bound on the assortment size. More specifically, we find that after the upper bound reaches a threshold, the value function exhibits significantly stronger diminishing returns. 

\subsection{Regret}
\label{online_regret}

For our simulation, we consider a platform with $N = 10$ sellers and an upper bound of $B = 4$ sellers on the assortment. There are two customer types, i.e., $ \cC = \{z_1, z_2\}$, where each customer type arrives on the platform with a probability 0.5. Each epoch consists of $K = 1$ period, and the planning horizon is set to $M = 200,000$ epochs, resulting in a total length of $T = MK = 200,000$ periods. \textcolor{black}{We restrict the epoch length to $K=1$ to align our simulation with the standard instantaneous-feedback environment of the MNL-Bandit problem. This structural alignment ensures a fair baseline comparison with the UCB-based algorithm from \citet{agrawal2019mnl}. Specifically, setting $K=1$ precludes our policy from batching proposals across periods, thereby neutralizing any algorithmic advantage our method might derive from delayed matching. Consequently, evaluating our two-sided learning algorithm under this strict single-period matching requirement provides a strictly conservative assessment of its relative performance advantage over policies optimized for instantaneous rewards.}

For each customer type $z \in \cC$, we define the preference parameter as follows:
\begin{align}
    v_{z, i}= 
    \begin{cases}
        0.3 + \ep & \text{if} \;\; z = z_1, i \in \set{1, 2, 3, 4},\\
        0.3 + \ep & \text{if} \;\; z = z_2, i \in \set{7, 8, 9, 10},\\
        0.3 & \text{for all other cases}.
    \end{cases} \label{eq:customer_params}
\end{align}

In \eqref{eq:customer_params}, we vary $\ep$ to generate different preference parameter instances for customers. For our simulation study, we consider the following values: $\epsilon \in \set{0.2, 0.3, 0.4, 0.5}$. In addition, we set $u_{i, z} = 1$ and $v_{z, 0} = u_{i, 0} = 1$ for all $i, z$. 

Figure \ref{fig:bb_regret} presents the regret of the TWL-UCB algorithm relative to a clairvoyant algorithm for different values of $\epsilon$, based on 25 independent simulations, each run for a total horizon of $T = 200,000$ periods. \textcolor{black}{It is evident that our algorithm's regret consistently exhibits sublinear growth with time across all considered values of $\epsilon$ in our setup.} The algorithm's consistent performance across all parameterizations underscores its robustness in diverse instances of the TWL problem. Additionally, Figure \ref{fig:bb_regret_ub} demonstrates that our algorithm's  regret is upper-bounded by $\mathcal{O}\big(\log^2(NT)\big)$.

We also observe from Figure \ref{fig:bb_regret} that the regret diminishes as $\epsilon$ increases. This trend is attributed to the accelerated learning of unknown parameters for all customer types as $\ep$ increases, leading to a progressively smaller loss in the expected platform reward over time.

\begin{figure}[htbp]
\begin{center}
\subfigure[Regret for $\ep \in \set{0.2, 0.3, 0.4, 0.5}$ \label{fig:bb_regret}]{\includegraphics[width=0.495\linewidth]{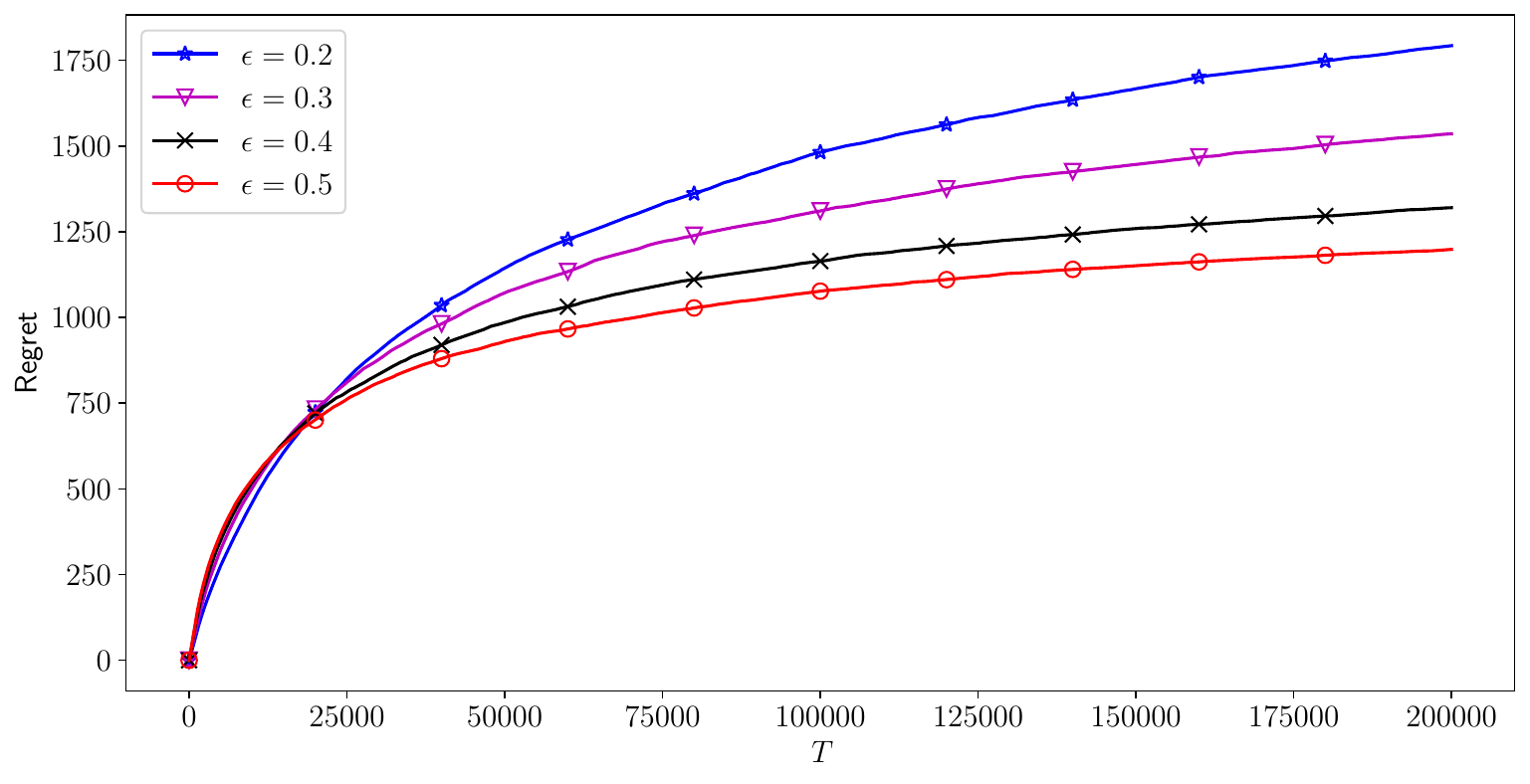}}
\subfigure[Regret vs. $\mathcal{O}\left(\log^2(NT)\right)$ \label{fig:bb_regret_ub}]{\includegraphics[width=0.495\linewidth]{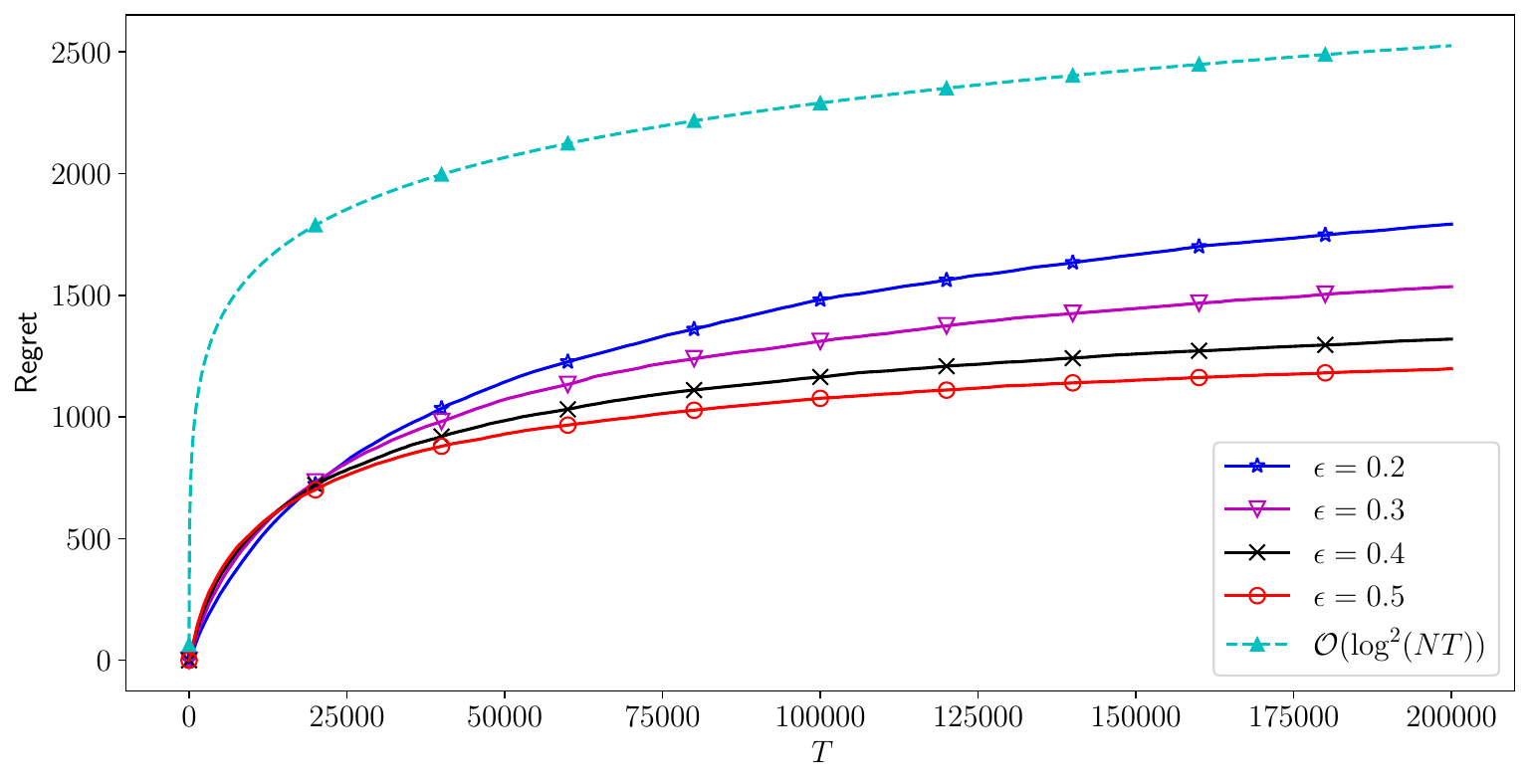}}
\caption{\scriptsize Regret of TWL-UCB Algorithm}
\label{fig:twl_regret}
\end{center}
\end{figure}

\begin{figure}[htbp]
\begin{center}
\subfigure[${\ep = 0.2}$ \label{fig:ep_a}]{\includegraphics[width=0.495\linewidth]{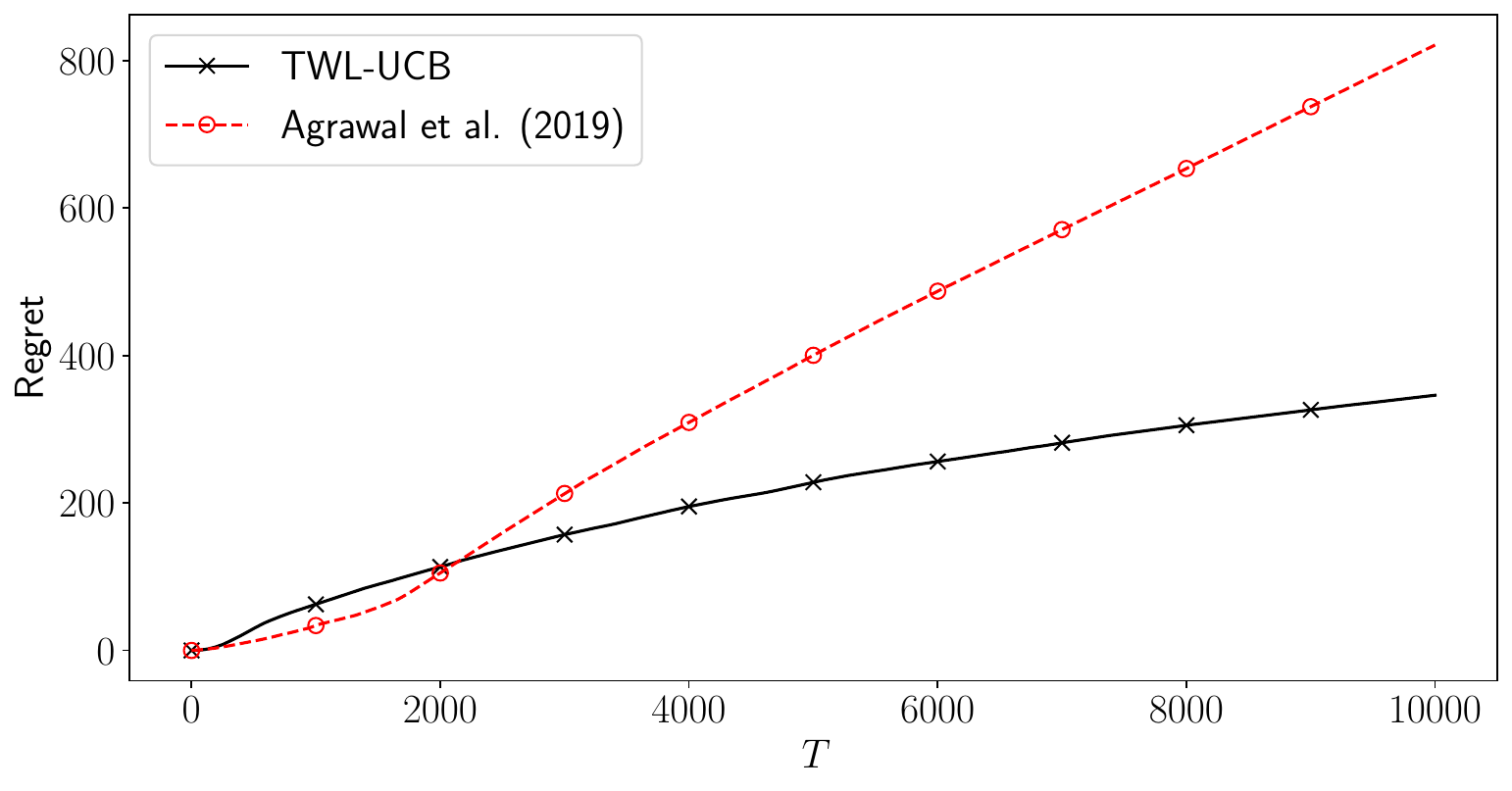}}
\subfigure[${\ep = 0.3}$ \label{fig:ep_b}]{\includegraphics[width=0.495\linewidth]{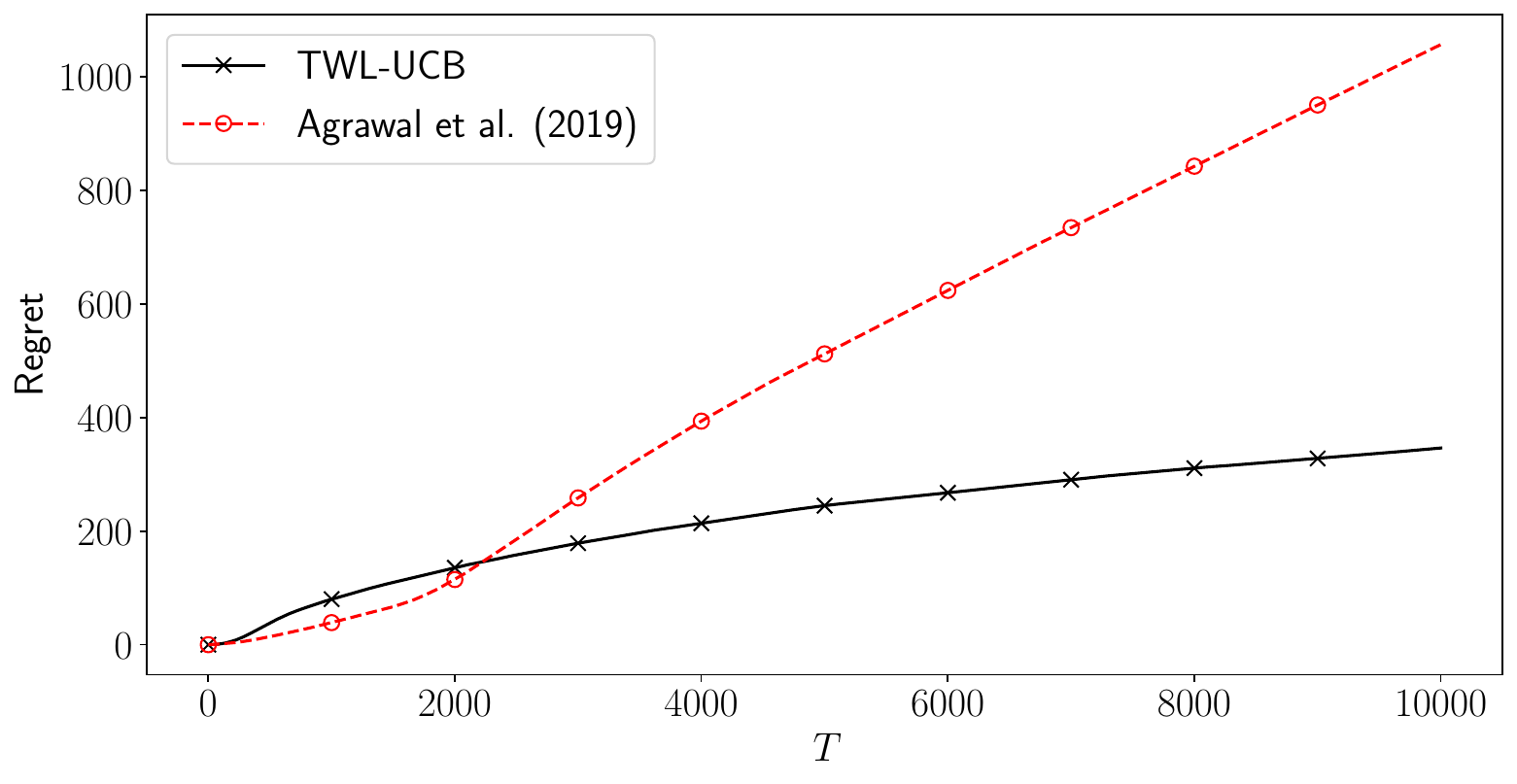}}
\subfigure[${\ep = 0.4}$ \label{fig:ep_c}]{\includegraphics[width=0.495\linewidth]{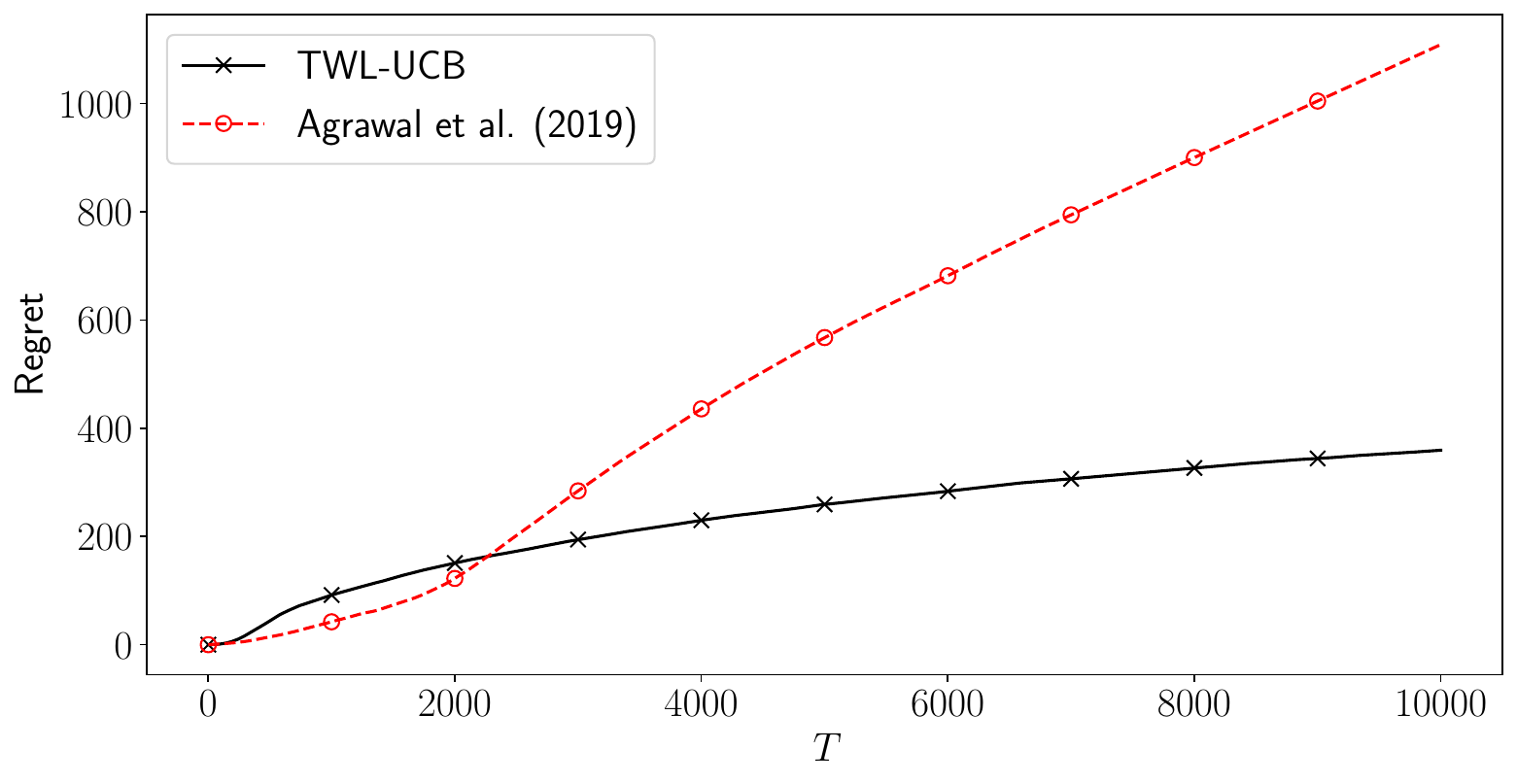}}
\subfigure[${\ep = 0.5}$ \label{fig:ep_d}]{\includegraphics[width=0.495\linewidth]{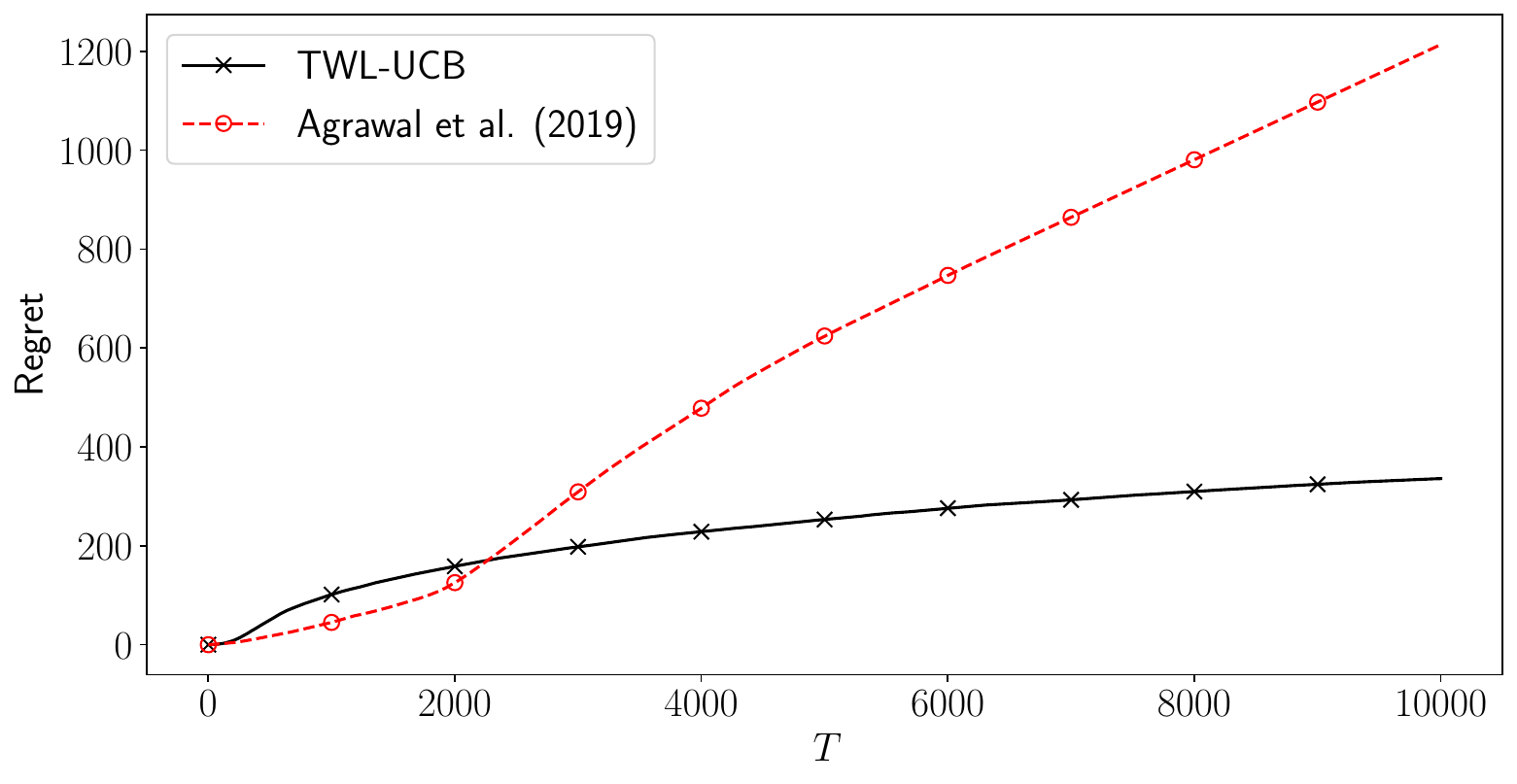}}
\caption{\scriptsize Performance Comparison of Algorithms: TWL-UCB vs. \citet{agrawal2019mnl}}
\label{fig:bb_vs_mnl}
\end{center}
\end{figure}

\subsection{Performance Comparison of Algorithms}
\label{performance_comparison}

The TWL-UCB algorithm leverages an exploration-exploitation approach to dynamically learn the preference parameters of both customers and sellers while maximizing the total reward over a specific time horizon in a two-sided platform. As mentioned in \S \ref{literature}, our setting is related to that of \citet{agrawal2019mnl}, though the problem we study differs considerably from theirs. To further emphasize the differences, in this section, we compare the performance of our TWL-UCB algorithm to that of the UCB-based algorithm proposed by \citet{agrawal2019mnl} for all parameter instances. \textcolor{black}{Note that the algorithm of \citet{agrawal2019mnl}, developed for one-way learning formulation of dynamic assortment selection, assumes homogeneous customers; so, to make a meaningful comparison, we consider only one customer type in this section to compare the performance of the TWL-UCB algorithm with that of Algorithm 1 in \citet{agrawal2019mnl}.}

\textcolor{black}{Figure \ref{fig:bb_vs_mnl} illustrates the performances of TWL-UCB and Algorithm 1 in \citet{agrawal2019mnl} for the same parameter instances as explained in \S \ref{online_regret} with homogeneous customers, i.e., $z_1 = z_2$.  In addition, we set $M = 10,000$, and therefore $T = MK = 10,000$, in this section. Across all the problem instances, the TWL-UCB algorithm achieves a superior performance (significantly lower regret) than the algorithm of \citet{agrawal2019mnl}. This finding further emphasizes that algorithms designed for dynamic assortment selection problems with incomplete information on one side (customers) perform poorly in setups that distinctly require learning about two sides (sellers and customers).} 

\textcolor{black}{We note that the UCB-based algorithm proposed by \citet{agrawal2019mnl} performs well when transactions depend solely on customer choices, which is the focus of their study. Their algorithm updates assortments only when a customer does not select a product, allowing the algorthm to gradually learn the optimal assortment and achieve sublinear regret over time. In contrast, we consider a setting where transactions require mutual agreement between customers and sellers. Relying solely on customer not selecting a product to update assortments slows down learning, ultimately leading to substantially higher regret, as shown in Figure \ref{fig:bb_vs_mnl}. The two-way dependency in our problem necessitates more frequent assortment updates to maximize mutual agreement in each period. Our TWL-UCB algorithm addresses this by dynamically updating assortments in each period based on historical behaviors of both customers and sellers, ensuring faster learning and improved performance.}

\subsection{Optimal Bound on Assortment Size}
\label{optimal_assortment_size}

In this section, we investigate how the upper bound on the assortment size impacts the value function in the TWL problem (as defined in \textcolor{black}{\eqref{eq:value}}). For our analysis, we set $N = 100$, $K = 1$, $M = \set{1000, 2000}$, $\cC = \set{z_1, z_2}$, and we vary $B$ from 1 through $N$ in steps of 1 to examine its impact on the value function. In addition, we let $u_{i, z} = 1$ for all $i \in \cS$ and $z \in \cC \cup \set{0}$, where $z=0$ corresponds to the case of an outside option or no choice on the platform. We sample $v_{z, i}$ for all $z \in \cC, i \in \cS$ from a uniform distribution within the interval $[0.3, 0.3 + \ep]$, where we run 25 independent simulations with $\ep = 0.3$ for a total of $T = MK = \set{1000, 2000}$ time periods. Finally, we set $v_{z, 0} = 1$ for all $z \in \cC$.  

Figure \ref{fig:vf_vs_B} shows the variation of the value functions for both clairvoyant and TWL-UCB algorithms with assortment bound $(B)$ for $T = 1000$ (Figure \ref{fig:vf_B_a}) and $T = 2000$ (Figure \ref{fig:vf_B_b}). Our analysis reveals a prominent pattern in the value function, showing diminishing returns as the assortment bound increases. Notably, there is no significant improvement in the value function beyond an assortment bound of 30. This suggests that expanding assortments beyond a certain threshold does not yield substantial gains for the platform manager. For platform managers looking to optimize assortment strategies, this insight is crucial. It highlights the need to balance assortment size with associated costs and benefits. By identifying this threshold and recognizing the diminishing returns of assortment expansion, managers can make more informed decisions. They can focus on optimizing assortments within the range where significant gains in total rewards are achievable. This, in turn, maximizes the platform’s overall performance and profitability. 

\begin{figure}[htbp]
\begin{center}
\subfigure[${T = 1000}$ \label{fig:vf_B_a}]{\includegraphics[width=0.495\linewidth]{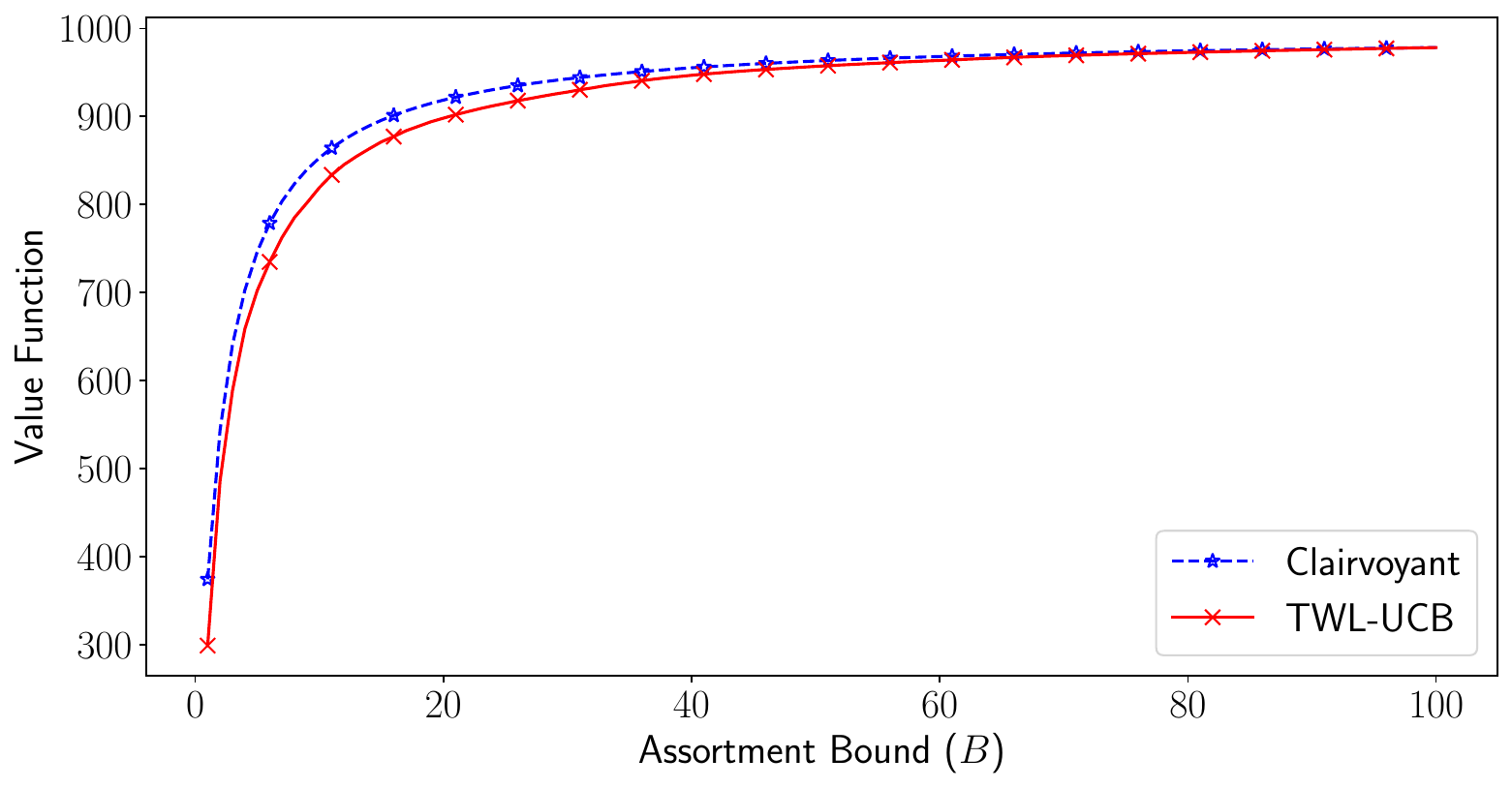}}
\subfigure[${T = 2000}$ \label{fig:vf_B_b}]{\includegraphics[width=0.495\linewidth]{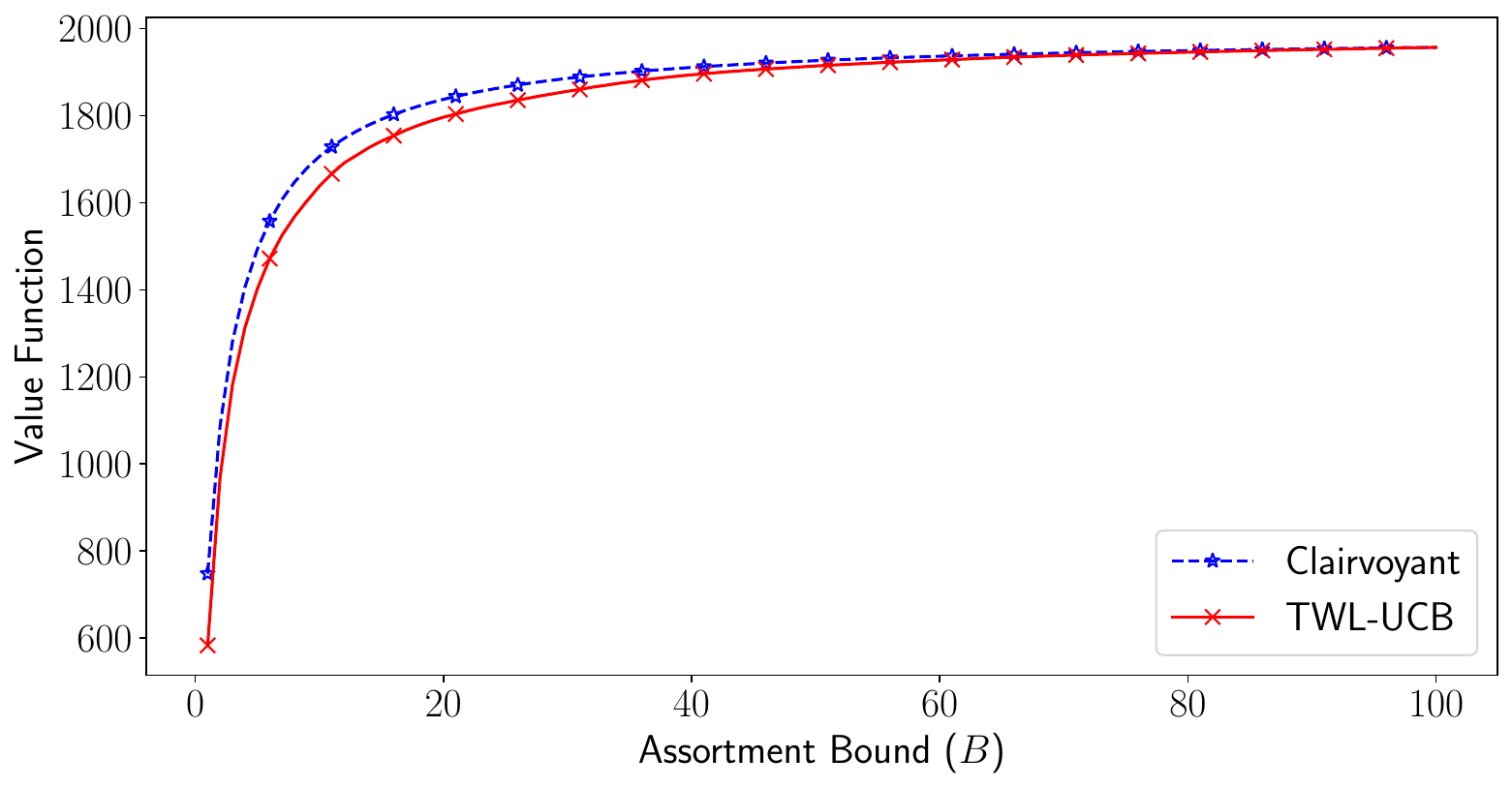}}
\caption{\scriptsize Value Function vs. Assortment Bound (${B}$)}
\label{fig:vf_vs_B}
\end{center}
\end{figure}

\begin{figure}[htbp]
\begin{center}
\subfigure[${T = 1000}$ \label{fig:dvf_B_a}]{\includegraphics[width=0.495\linewidth]{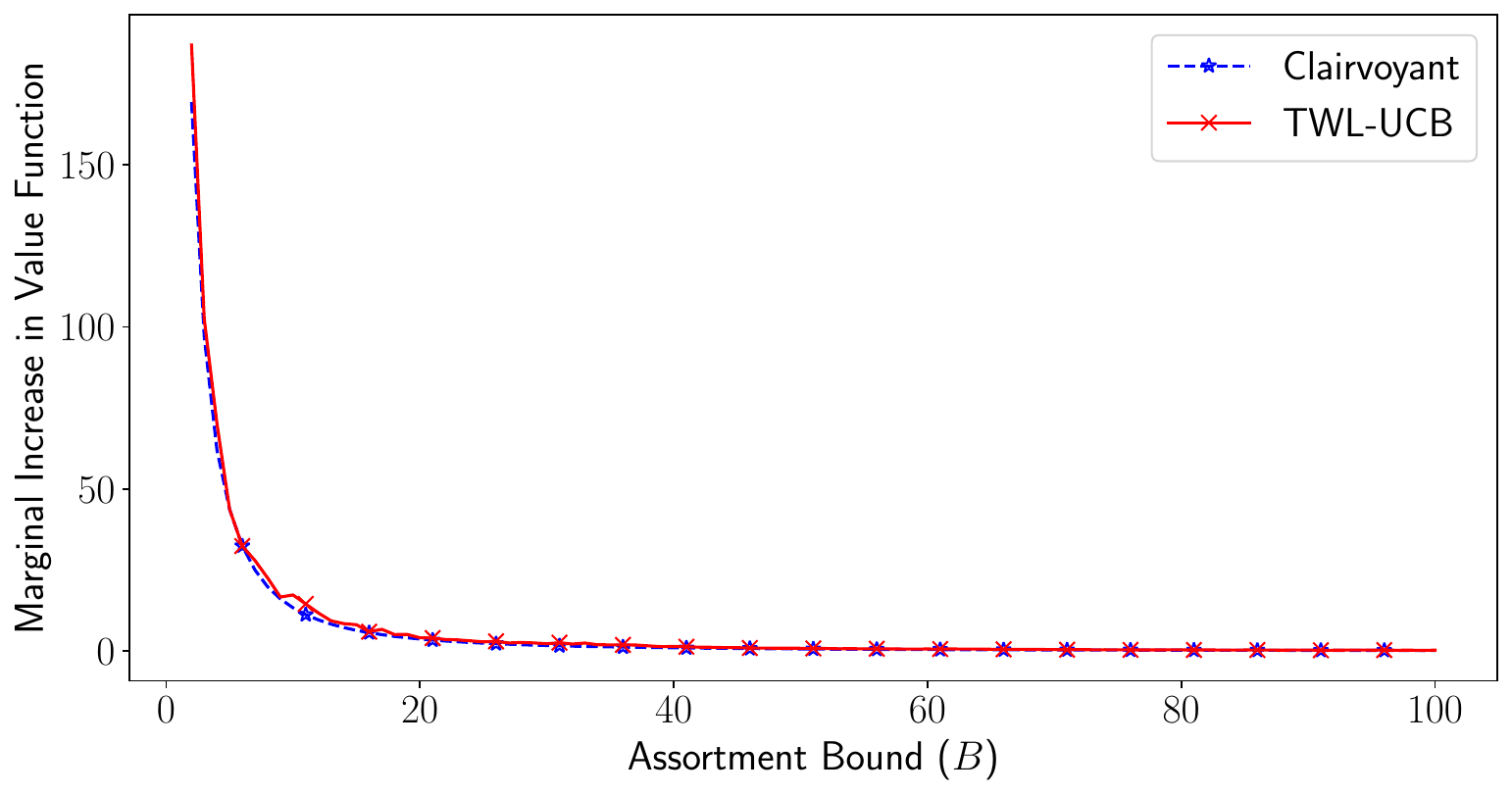}}
\subfigure[${T = 2000}$ \label{fig:dvf_B_b}]{\includegraphics[width=0.495\linewidth]{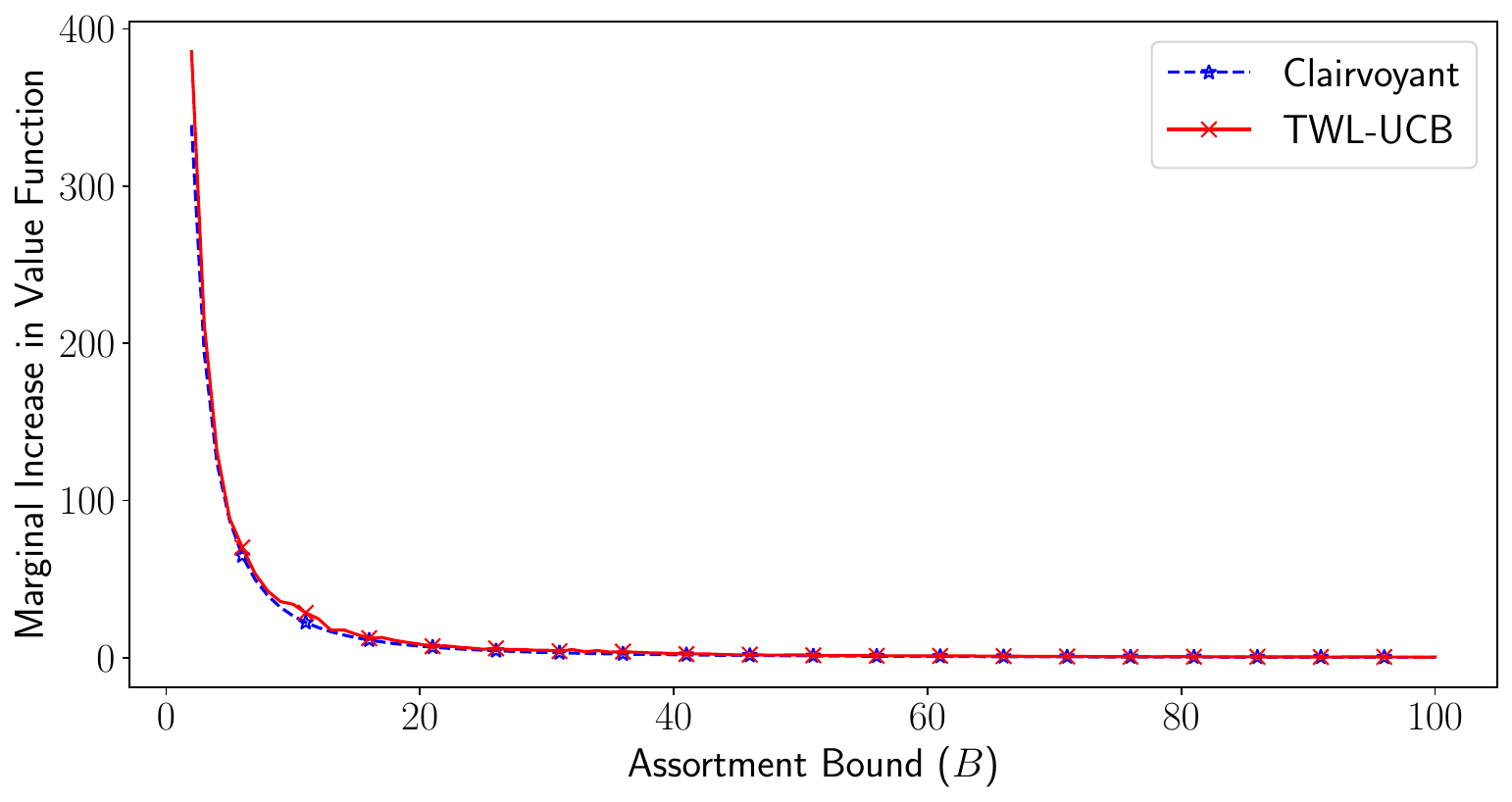}}
\caption{\scriptsize Marginal Increase in Value Function vs. Assortment Bound (${B}$)}
\label{fig:dvf_vs_B}
\end{center}
\end{figure}

To strengthen the aforementioned insight, we examine the variation of the marginal increase in the platform's value function with the assortment bound. Figure \ref{fig:dvf_vs_B} illustrates that there is no significant change in the marginal increase in the value function once the assortment bound surpasses a specific threshold. Figures \ref{fig:dvf_B_a} and \ref{fig:dvf_B_b} suggest that the threshold on the assortment bound lies around 30, consistent with the findings depicted in Figure \ref{fig:vf_vs_B}. This analysis reinforces our earlier conclusion that there are diminishing returns to scale with the assortment bound. Beyond a certain threshold, increasing the assortment size does not yield significant additional benefits in terms of the total reward. Understanding this threshold is crucial for platform managers in optimizing assortment strategies to maximize the platform's overall performance and profitability.

\section{Extensions}
\label{extensions}

In this section, we introduce and analyze two extensions of the TWL problem described in \S \ref{problem}. These extensions establish the robustness of our TWL-UCB algorithm in handling diverse scenarios in online platform operations.

\subsection{Information Transparency}
\label{info_transparency}

This section demonstrates the robustness of our proposed policy to various degrees of information transparency. Specifically, we will show that Theorem \ref{thm_regret_upper_bound} holds true irrespective of whether the platform manager shares complete or partial customer-specific information with the sellers. This finding underscores the resilience of our algorithm to different levels of information disclosure, enhancing its applicability and effectiveness in diverse online platform environments.

In this section, we explore an extension of our problem where customers have multiple features and the platform manager only reveals a fraction of the customer-specific information to the sellers alongside the proposal matrix. This extension introduces a layer of complexity to the problem, as sellers now have access to limited information about the customers they target. The decision to selectively disclose customer-specific information poses several challenges and trade-offs. On one hand, providing limited information to sellers may help protect customer privacy and prevent exploitation of sensitive data. On the other hand, it may also hinder sellers' ability to tailor their offerings effectively and deliver personalized experiences to customers. In the remainder of this section, we investigate the implications of this extended problem formulation and explore any potential impact on the performance of our algorithm. 

In an MNL choice model, the preference parameters of a seller $i \in \cS$ for customer type $z \in \cC$ is of the form: $u_{i, z} = \exp{(\mathbf{x}_{i, z}^{\top}\boldsymbol{\vartheta})}$, where $\mathbf{x}_{i, z} \in \R^{d}$ is the contextual feature vector of customer of type $z$ for provider $i$ and $\boldsymbol{\vartheta} \in \R^{d}$ is the parameter vector. Let $x_{i, z}^j$ be the $j$-the component of the feature vector $\mathbf{x}_{i, z}$, where $j \in \set{1, \ldots, d}$. Without loss of generality, we assume that $x_{i, z}^j = 0$ for any $j$ indicates that the corresponding feature is hidden from the providers in $\mathbf{x}_{i, z}$.  \textcolor{black}{In this revised framework, the platform manager either conceals certain customer-specific features from the providers or reveals all features while sharing information about customer proposals.} All other modeling elements are the same as the ones described in \S \ref{problem}. Based on these, Theorem \ref{thm_info_transparency} presents our result in this extended setting.    

\begin{theorem}\label{thm_info_transparency}
    The regret of the TWL-UCB algorithm for the TWL problem under partial information transparency is given by:
    \[Reg{_{_\textup{TWL-UCB}}}(T) = \mathcal{O}\left(\log^2(NT)\right).\]
\end{theorem}

Theorem \ref{thm_info_transparency} implies that the performance of our policy is robust to variations in the level of information transparency, whether the platform manager shares complete or partial customer-specific information with the sellers. In other words, under Assumption \ref{assumption:mnl}, the efficacy of our proposed policy in minimizing regret is not compromised by the extent to which customer-specific information is disclosed to the sellers. This robustness ensures the effectiveness of our policy across different scenarios of information transparency, offering a versatile solution for dynamic assortment selection in two-sided platforms. For a detailed proof of the robustness of our proposed policy with respect to information transparency, we direct the readers to \S \ref{ec_info_robust}. 

\subsection{Customer-Dependent Reward}
\label{customer_based_reward}

We now consider an extension in which rewards depend on customer type. We show that Theorem \ref{thm_regret_upper_bound} continues to hold in this setting, implying that the regret performance of our algorithm is unaffected by customer-dependent rewards.

In the original TWL problem (\S \ref{problem}), the platform receives a reward $\gamma_m \in (0, \bR]$ for each match between a customer and a seller in epoch $m$. While rewards may vary across epochs, they are assumed to be independent of customer type. Here, we extend the model to allow rewards to depend on the type of customer matched, thereby capturing customer heterogeneity in match values.

Throughout this section, we adopt a customer-dependent reward structure expressed as $\gamma_{m, z} \in (0, \Hat{\gamma}]$ for $\Hat{\gamma} \in \mathbb{R}_+$, where $\gamma_{m, z}$ represents the platform's reward from a match involving customer type $z$ in epoch $m$. Other than the reward structure, all  modeling elements are the same as the ones described in \S \ref{problem}. We now present our final theoretical result, which establishes the robustness of our algorithm's performance under a customer-dependent reward structure.

\begin{theorem}\label{thm_cust_reward}
    Under a customer-dependent reward structure of the form $\gamma_{m, z} \in \R_+$ for all $m \in \sM$ and $z \in \cC$, the regret of the TWL-UCB algorithm for the TWL problem is given by:
    \[Reg{_{_\textup{TWL-UCB}}}(T) = \mathcal{O}\left(\log^2(NT)\right).\]
\end{theorem}

The proof of this result is available in \S \ref{ec_customer_reward_robust}. 

\section{Conclusions}
\label{conclusions}

Assortment optimization has long been a central topic in the operations research / management science literature, attracting significant attention because of its broad applicability across a wide range of sectors. In this study, we introduce a novel formulation of the dynamic assortment selection problem for two-sided platforms with ``active'' customers under incomplete information about the preferences of both sides. We refer to this problem as the “Two-Way Learning” (TWL) problem. The TWL problem differs fundamentally from earlier studies on dynamic assortment selection, in which the decision-maker typically focused only on learning demand or the choice parameters of customers. In contrast, in the TWL problem, the decision-maker must learn the choice parameters of both customers and sellers while simultaneously optimizing an objective function. This distinctive feature gives rise to a two-way sequential learning problem, as opposed to prior settings that involved learning on only one side. Through this study, we aim to uncover the complexities inherent in dynamic assortment selection under two-way learning and to develop insights and methodologies that enable platform managers to make informed decisions in real time.

For the TWL problem, we propose a UCB-based online algorithm designed to effectively address the ``exploration-exploitation" trade-off within the sequential learning paradigm. Our algorithm achieves a polylogarithmic upper bound on the regret, given by $\mathcal{O}\big(\log^2(NT)\big)$, on its performance with respect to that of a clairvoyant policy. Furthermore, we show that any admissible policy for the TWL problem incurs a lower bound of at least $\widetilde \Omega\big(\log^2(NT)\big)$ on the worst-case regret. Based on the above results, we establish the rate optimality of our algorithm. We strengthen our theoretical findings with simulation studies, which reveal three important findings. First, in our studies, the TWL-UCB algorithm  consistently perform well  across several instances of the two-way learning problem. Second, we establish that algorithms designed for dynamic assortment selection problems with incomplete information on one side (customers) perform poorly in our setup that distinctly requires learning about two sides (sellers and customers). Third, we observe that beyond a threshold of the assortment bound, a platform manager does not receive any measurable gain in total rewards. In other words, offering larger assortments to customers may not yield meaningful gains for an online platform.

Finally, we analyze the following two extensions of our setting and establish the robustness of our algorithm, as measured via its regret, under those extensions: First, we show that the algorithm's performance is not affected even if the platform manager shares only partial information of the customers with the sellers on the online platform. Second, we prove that the algorithm's performance does not get impacted even if the reward received upon a successful match depends on the type of customer matched. A more detailed discussion of our findings and the contributions of our paper is provided in the first three pages of the manuscript.

\smallskip

\begin{APPENDICES}
\section{Standard Notations}\label{notations}

\begin{definition}\label{big-o}
    Suppose that $f(x)$ and $g(x)$ are two non-negative functions of  of $x \geq 0$. We write $f(x) = \mathcal{O}\left(g(x)\right)$ if and only if there exist constants $x_0$ and $C_1 > 0$ such that $f(x) \leq C_1g(x)$ for all $x \geq x_0$. Intuitively, this means that $f$ grows no faster than $g$.   
\end{definition}

\begin{definition}\label{big-omega}
    Suppose that $f(x)$ and $g(x)$ are two non-negative functions of  of $x \geq 0$. We write $f(x) = \widetilde \Omega\left(g(x)\right)$ if and only if there exist constants $x_0$ and $C_2 > 0$ such that $f(x) \geq C_2g(x)$ for all $x \geq x_0$. Intuitively, this means that $f$ grows at least as fast as $g$.
\end{definition}

\begin{definition}\label{big-theta}
    Suppose that $f(x)$ and $g(x)$ are two non-negative functions of  of $x \geq 0$. We write $f(x) = \widetilde \Theta\left(g(x)\right)$ if $f(x)$ is both $\mathcal{O}\left(g(x)\right)$ and $\widetilde \Omega\left(g(x)\right)$. Intuitively, this means that $f$ grows at the same rate as $g$.
\end{definition}

\begin{definition}\label{rate-optimality}
    Suppose that $Reg_{_\mathcal{A}}(T)$ denotes the regret of an admissible policy $\mathcal{A}$ for a problem $\mathcal{P}$ over $T$ periods. The policy $\mathcal{A}$ is rate-optimal if $Reg_{_\mathcal{A}}(T) = \widetilde \Theta\left(h(T)\right)$, where $h(T)$ is the lower bound on the worst-case regret of any admissible policy for the problem $\mathcal{P}$ over $T$ periods.
\end{definition}

\end{APPENDICES}


{
\SingleSpacedXI
\bibliographystyle{informs2014} 
\bibliography{2-myrefs} 
}	

%
%
%


\ECSwitch

\ECHead{Electronic Companion for ``Data-Driven Dynamic Assortment in Online Platforms: Learning about Two Sides"}

\section{Supplementary Lemmas}
\label{supplementary_lemmas}

\begin{lemma}[\citet{kleinberg2008multi}, Lemma 4.9]\label{lemma_kleinberg_2008}
    Consider $n$ i.i.d random variables $X_1, \ldots, X_n$ on $[0, 1]$. Let $\mu$ be their mean, and let $X$ be their average. Then for any $\alpha > 0$ the following holds:
    \begin{align*}
        \Pr\big[|X-\mu| < r(\alpha, X) < 3r(\alpha, \mu)\big] > 1 - e^{-\Omega(\alpha)}, \; \text{where} \; r(\alpha, x) = \dfrac{\alpha}{n} + \sqrt{\dfrac{\alpha x}{n}}. 
    \end{align*}
\end{lemma}

{\bf Proof of Lemma \ref{lemma_kleinberg_2008}:}\label{proof_lemma_kleinberg_2008} 
For completeness, we provide the proof of Lemma \ref{lemma_kleinberg_2008} below.\footnote{In \citet{kleinberg2008multi}, $\Pr[\cdot]$ is used to denote probability rather than the $\prob(\cdot)$ operator utilized throughout this manuscript.} 

We use two Chernoff Bounds for a sum of $n$ independent Poisson trials from \citet{mitzenmacher2017probability}\footnote{Theorem 4.4 in \citet[~p.69]{mitzenmacher2017probability} and Corollary 4.6 in \citet[~p.71]{mitzenmacher2017probability}} to prove Lemma \ref{lemma_kleinberg_2008}.
\begin{itemize}
    \item[(CB1)] \(\Pr\big[|X-\mu| > \delta \mu \big] < 2e^{-n\mu\delta^2/3}\) for any $\delta \in (0, 1)$.
    \item[(CB2)] \(\Pr\big[X > a \big] < 2^{-an}\) for any $a > 6\mu$.
\end{itemize}

Case 1: $\mu \geq \alpha/6n$. Let $\delta = 0.5\sqrt{\alpha/6n\mu}$. Using (CB1), we get \(|X-\mu| < \delta\mu \leq \mu/2\) with probability at least $1-e^{-\Omega(\alpha)}$. Substituting the value of $\delta$, we get the following with probability at least $1-e^{-\Omega(\alpha)}$,
\[|X-\mu| < \dfrac{1}{2}\sqrt{\alpha\mu/n} \leq \sqrt{\alpha X/n} \leq r(\alpha, X) < \dfrac{3}{2}r(\alpha, \mu).\]

Case 2: $\mu < \alpha/6n$. Let $a = \alpha/n$, Using (CB2), we get $X < \alpha/n$ with probability at least $1-2^{-\Omega(\alpha)}$. Therefore, we have the following with probability at least $1-2^{-\Omega(\alpha)}$,
\[|X-\mu| < \dfrac{\alpha}{n}< r(\alpha, X) <(1 + \sqrt{2})\dfrac{\alpha}{n} < 3r(\alpha, \mu).\]

Combining Cases 1 and 2, we get: 
\[\Pr\big[|X-\mu| < r(\alpha, X) < 3r(\alpha, \mu)\big] > 1 - e^{-\Omega(\alpha)}.\]  

\hfill $\square$ 

\begin{lemma}[\citet{chen2018note}, Lemma 3]\label{lemma_chen_2018} 
    Suppose $P$ is a categorical distribution with parameters $p_0, \ldots, p_J$, meaning that $\prob(X=j) = p_j$ for all $j = 0, \ldots, J$, and $Q$ is a categorical distribution with parameters $q_0, \ldots, q_J$. Suppose also $p_j = q_j + \varepsilon_j$ for all $ j = 0, \ldots, J$. Then 
    \[\textup{KL}(P \parallel Q) \leq \sum_{j=0}^J \dfrac{\varepsilon_j^2}{q_j}. \]
\end{lemma}

{\bf Proof of Lemma \ref{lemma_chen_2018}:}\label{proof_lemma_chen_2018} For completeness, we provide the proof of Lemma \ref{lemma_chen_2018} below. 

Given $P$ and $Q$, the KL divergence is given by: 
\begin{align*}
    \textup{KL}(P\parallel Q) &= \sum_{j = 0}^J p_j \log\dfrac{p_j}{q_j}\\ 
    &= \sum_{j = 0}^J (q_j + \varepsilon_j) \log\dfrac{q_j + \varepsilon_j}{q_j}\\
    &\leq \sum_{j = 0}^J (q_j + \varepsilon_j) \dfrac{\varepsilon_j}{q_j} \tag*{(using \mbox{$\log(1 + x) \leq x \; \forall x > -1$})} \\
    &= \sum_{j = 0}^J\varepsilon_j + \sum_{j = 0}^J \dfrac{\varepsilon_j^2}{q_j} = \sum_{j = 0}^J \dfrac{\varepsilon_j^2}{q_j}. \tag*{(using \mbox{$\sum_{j = 0}^J\varepsilon_j = 0$} since \mbox{$\sum_{j = 0}^Jp_j = \sum_{j = 0}^Jq_j = 1$})}
\end{align*}

\hfill $\square$

\begin{lemma}[\citet{agrawal2019mnl}, Lemma A.3 (Optimistic Estimates)]\label{lemma_agrawal_2019}
    Assume $0 \leq w_i\leq v_i^{\textup{UCB}}$ for all $i = 1, \ldots, n$.  Suppose $S$ is an optimal assortment when the MNL parameters are given by $\mathbf{w}$. Then, $R(S, \mathbf{v}^{\textup{UCB}}) \geq R(S, \mathbf{w})$. 
\end{lemma}

{\bf Proof of Lemma \ref{lemma_agrawal_2019}:}\label{proof_lemma_agrawal_2019} For completeness , we provide the proof of Lemma \ref{lemma_agrawal_2019} below. 

In Lemma A.3 (Optimistic Estimates) of \citet{agrawal2019mnl}, $R(S, \mathbf{w})$ is given by: 
\begin{align*}
    R(S, \mathbf{w}) = \sum_{i \in S}\dfrac{w_ir_i}{1 + \sum_{i \in S}w_i},
\end{align*}
where $\mathbf{w}$ represents the MNL parameters of the customers.

If $S$ is the clairvoyant's optimal assortment under the MNL parameters $\mathbf{w}$, then for any feasible assortment $\Tilde{S}$, we have:
\begin{align}
    &R(S, \mathbf{w}) \geq \sum_{i \in \Tilde{S}} \dfrac{w_ir_i}{1 + \sum_{i \in \Tilde{S}}w_i} \nonumber \\
    \iff \quad &R(S, \mathbf{w})\bigg(1 + \sum_{i \in \Tilde{S}}w_i\bigg) \geq \sum_{i \in \Tilde{S}}w_ir_i \nonumber\\
    \iff \quad &R(S, \mathbf{w}) \geq \sum_{i \in \Tilde{S}}w_i\bigg(r_i - R(S, \mathbf{w})\bigg). \label{eq:r_optimal}
\end{align}

In \eqref{eq:r_optimal}, $\sum_{i \in \Tilde{S}}w_i\bigg(r_i - R(S, \mathbf{w})\bigg)$ is maximum when $r_i \geq R(S, \mathbf{w})$ for all $i \in \Tilde{S}$. Clearly, when $\sum_{i \in \Tilde{S}}w_i\big(r_i - R(S, \mathbf{w})\big)$ is maximum, then $\Tilde{S} = S$. Thus, for optimal assortment $S$, we have the following:
\begin{align}
    &r_j \geq R(S, \mathbf{w}), \forall j \in S \nonumber\\
    \iff \quad &r_j\big(1+\sum_{i \in S}w_i\big) \geq \sum_{i \in S}w_ir_i, \forall j \in S \nonumber\\
    \iff \quad &r_j \geq \sum_{i \in S} w_i\big(r_i - r_j\big) = \sum_{i \in S\setminus \set{j}}w_i\big(r_i - r_j\big), \forall j \in S. \label{eq:r_j}
\end{align}

Let $R(S, \mathbf{w}^j)$ be the expected revenue under optimal assortment $S$ when the MNL parameter of some arbitrary $j \in S$ is increased from $w_j$ to $v_j^{\textup{UCB}}$. Let $S_{-j} = S \setminus \set{j}$. Then, we have the following: 
\begin{align}
    &R(S, \mathbf{w}^j) - R(S, \mathbf{w}) \nonumber \\
    &= \dfrac{\sum_{i \in S_{-j}}w_ir_i + v_j^{\textup{UCB}}r_j}{1 + \sum_{i \in S_{-j}}w_i + v_j^{\textup{UCB}}} - \dfrac{\sum_{i \in S_{-j}}w_ir_i + w_jr_j}{1 + \sum_{i \in S_{-j}}w_i + w_j} \nonumber\\
    &= \dfrac{(v_j^{\textup{UCB}} - w_j)r_j(1 + \sum_{i \in S_{-j}}w_i)}{(1 + \sum_{i \in S_{-j}}w_i + v_j^{\textup{UCB}})(1 + \sum_{i \in S_{-j}}w_i + w_j)} - \dfrac{(v_j^{\textup{UCB}}-w_j)\sum_{i \in S_{-j}}w_ir_i}{(1 + \sum_{i \in S_{-j}}w_i + v_j^{\textup{UCB}})(1 + \sum_{i \in S_{-j}}w_i + w_j)} \nonumber\\
    &= \dfrac{(v_j^{\textup{UCB}} - w_j)(r_j - \sum_{i \in S_{-j}}w_i(r_i - r_j))}{(1 + \sum_{i \in S_{-j}}w_i + v_j^{\textup{UCB}})(1 + \sum_{i \in S_{-j}}w_i + w_j)} \geq 0. \label{eq:r_ucb_monotone}\\
    &\text{(using $r_j \geq \sum_{i \in S_{-j}}w_i\big(r_i - r_j\big)$ from \eqref{eq:r_j})} \nonumber
\end{align}
Since $j \in S$ is arbitrarily chosen in \eqref{eq:r_ucb_monotone}, inequality \eqref{eq:r_ucb_monotone} is true for all $j \in S$. Therefore, we get $R(S, \mathbf{v}^{\textup{UCB}}) \geq R(S, \mathbf{w})$. 

\hfill $\square$ 

\begin{lemma}\label{lemma_exp_reward}
    Suppose that $m \in \set{1, \ldots, M}$ is any epoch of length $K$ periods with a corresponding reward $\gamma_m$. Let $S_t$ be the assortment offered to an arriving customer in period $t$ and $\bC(t-1)$ be the corresponding proposal matrix at the beginning of $t$. If $\cM(S_t, \bC(t-1), c_t)$ is the marginal increase in expected matches at time $t$, then, for any instances $\bU$ and $\bV$ of the TWL problem, the total reward in epoch $m$ is given by:
    \begin{align*}
        \cR(m, \bU, \bV) = \gamma_m\sum_{t = (m-1)K + 1}^{mK} \cM(S_t, \bC(t-1), c_t),
    \end{align*}
    where $\cM(S_t, \bC(t-1), c_t)$ is given by \eqref{eq:deltaM}. 
\end{lemma}

{\bf Proof of Lemma \ref{lemma_exp_reward}:} \label{proof_lemma_exp_reward} Recall that sellers evaluate their respective proposals at the end of each epoch $m \in \set{1, \ldots, M}$, where each epoch is of length $K$ periods. We denote seller $i$'s choice at the end of epoch $m$, i.e., at time $mK$ by a random variable $\xi_i(\bC_i(mK))$, where $\bC_i(mK)$ is the vector of customer proposals received by seller $i$ in epoch $m$. The total number of matches at the end of epoch $m$ is therefore given by:
\begin{align*}
    \mathcal{M}(mK) = \sum_{i \in \cS} \I\big[\xi_i(\bC_i(mK)) \in \bC_i(mK)\big], 
\end{align*}
where $\I\big[\xi_i(\bC_i(mK)) \in \bC_i(mK) \big]$ indicates that seller $i$ matches with any customer within $\bC_i(mK)$. 

If sellers were to evaluate their proposals at the end of $mK - 1$ instead of $mK$, then the total number of matches at the end of $mK - 1$ would be:
\begin{align*}
    \mathcal{M}(mK - 1) = \sum_{i \in \cS} \I\big[\xi_i(\bC_i(mK - 1)) \in \bC_i(mK-1)\big].
\end{align*}

We can therefore calculate the marginal increase in the number of matches as follows:
\begin{align}
    &\mathcal{M}(mK) - \mathcal{M}(mK - 1) \nonumber \\ 
    &= \sum_{i \in \cS} \I\big[\xi_i(\bC_i(mK)) \in \bC_i(mK)\big] - \sum_{i \in \cS} \I\big[\xi_i(\bC_i(mK - 1)) \in \bC_i(mK-1)\big] \nonumber \\
    &= \sum_{i \in \cS} \bigg[\I\big[\xi_i\big(\bC_i(mK - 1) + \bfe_{c_{mK}}\cdot\I[\zeta_{mK}(S_{mK}) = i ]\big) \in \bC_i(mK)\big] - \I\big[\xi_i\big(\bC_i(mK - 1)\big) \in \bC_i(mK-1)\big] \bigg], \label{eq:marginal_inc_matches}
\end{align}
where $\bfe_{c_{mK}}$ is the unit basis vector at coordinate $c_{mK}$, and $\zeta_{mK}(S_{mK})$ is a random variable that represents customer's choice at time $mK$ when an assortment $S_{mK}$ is offered. Using \eqref{eq:marginal_inc_matches}, we therefore have the following:
\begin{align}
    \E\bigg[\mathcal{M}(mK) - \mathcal{M}(mK - 1)\bigg] &= \E\left[\E\bigg[\mathcal{M}(mK) - \mathcal{M}(mK - 1) \mid \zeta_{mK}(S_{mK})\bigg]\right] \nonumber \\
    &= \sum_{i \in \cS}\E\bigg[\phi_i\big(\bC_i(mK - 1) + \bfe_{c_{mK}}\I\big[\zeta_{mK}(S_{mK}) = i \big]\big) - \phi_i(\bC_i(mK-1))\bigg] \nonumber  \\
    &= \sum_{i \in S_{mK}} p_{c_{mK}, i}(S_{mK})\big[\phi_i(\bC_i(mK)) - \phi_i(\bC_i(mK-1))\big] \nonumber\\
    &= \cM(S_{mK}, \bC(mK-1), c_{mK}). \label{eq:marginal_inc_exp_matches} \\
    &\text{(using the definition of $\cM(S_t, \bC(t-1), c_t)$ from \eqref{eq:deltaM})} \nonumber
\end{align}

Let $b_m \doteq (m-1)K + 1$. Using \eqref{eq:marginal_inc_exp_matches}, we get the following:
\begin{align*}
    &\E\bigg[\mathcal{M}(mK) - \mathcal{M}(mK - 1)\bigg] = \sum_{i \in S_{mK}} p_{c_{mK}, i}(S_{mK})\big[\phi_i(\bC_i(mK)) - \phi_i(\bC_i(mK-1))\big], \\
    &\E\bigg[\mathcal{M}(mK-1) - \mathcal{M}(mK - 2)\bigg] = \sum_{i \in S_{mK-1}} p_{c_{mK-1}, i}(S_{mK-1})\big[\phi_i(\bC_i(mK-1)) - \phi_i(\bC_i(mK-2))\big],\\
    &\vdots\\
    &\E\bigg[\mathcal{M}(b_m) - \mathcal{M}(b_m-1)\bigg] = \sum_{i \in S_{b_m}} p_{c_{b_m}, i}(S_{b_m})\big[\phi_i(\bC_i(b_m)) - \phi_i(\bC_i(b_m-1))\big].
\end{align*}

Adding all the above equations and using $\mathcal{M}(b_m-1) = 0$, we have
\begin{align}
    \E\bigg[\mathcal{M}(mK)\bigg] &= \sum_{t = b_m}^{mK}\E\bigg[\mathcal{M}(t) - \mathcal{M}(t - 1)\bigg] \nonumber \\ 
    &= \sum_{t = b_m}^{mK} \sum_{i \in S_t} p_{c_t, i}(S_t)\big[\phi_i(\bC_i(t)) -  \phi_i(\bC_i(t-1))\big] \nonumber \\
    &= \sum_{t = b_m}^{mK}\cM(S_t, \bC(t-1), c_t). \label{eq:total_exp_matches} \\
    &\text{(using the definition of $\cM(S_t, \bC(t-1), c_t)$ from \eqref{eq:deltaM})} \nonumber
\end{align}

Using \eqref{eq:total_exp_matches}, the total reward at the end of epoch $m$ is therefore given by:
\begin{align*}
    \cR(m, \bU, \bV) = \gamma_m\E\bigg[\mathcal{M}(mK)\bigg] = \gamma_m\sum_{t = b_m}^{mK}\cM(S_t, \bC(t-1), c_t).
\end{align*}

\hfill $\square$ 

\begin{lemma}\label{lemma_alpha_lower_bound}
    Consider a random variable $X \in [0, 1]$. Let $Q(\alpha, X) = C_1\dfrac{\alpha}{n} + C_2\sqrt{\dfrac{\alpha}{n}X}$, where $\alpha = \mathcal{O}(\log{(t)})$, $n$ is a positive integral random variable, and $C_1, C_2 > 0$. If $\prob\bigg(|X - \E\big[X\big]| > Q(\alpha, X)\bigg) < p$, then we have the following:
    \begin{align}
        n \leq C_1\alpha = C_1\mathcal{O}(\log{(t)}), \;\; \text{with probability at least $1-p$}. \nonumber
    \end{align}
\end{lemma}

{\bf Proof of Lemma \ref{lemma_alpha_lower_bound}:}\label{proof_lemma_alpha_lower_bound} With probability less than $p$, we have the following:
\begin{align}
    &|X - \E\big[X\big]| > Q(\alpha, X) = C_1\dfrac{\alpha}{n} + C_2\sqrt{\dfrac{\alpha}{n}X} > C_1\dfrac{\alpha}{n} \nonumber\\
    \implies \quad &n > C_1\dfrac{\alpha}{|X - \E\big[X\big]|} > C_1\alpha. \label{eq:alpha_upper_bound} \\
    &\text{(using \mbox{$|X - \E\big[X\big]| \leq 1$})} \nonumber
\end{align}

Using \eqref{eq:alpha_upper_bound}, we get $\prob\bigg(n \leq C_1\alpha = C_1\mathcal{O}(\log{(t)}) \bigg) \geq 1 - p$.

\hfill $\square$ 

\begin{lemma}\label{lemma_rucb_lower_bound}
    Let $b_m \doteq (m-1)K + 1$ be the first period in epoch $m$. Suppose that $u_{i, z, m}^{\textup{UCB}} \geq u_{i, z}$ for all $i \in \cS$ and $z \in \cC$, where $u_{i, z, m}^{\textup{UCB}}$ is given by \eqref{eq:uucb}. Then for all $i \in \cS$, $z \in \cC$, and $t \in \set{b_m, \ldots, mK}$, we have the following:
    \begin{align*}
        r_{i, z}^{\textup{UCB}}(b_m) \geq r_{i, z}(b_m) \geq r_{i, z}(t), 
    \end{align*}
    where $r_{i, z}(t)$ and $r_{i, z}^{\textup{UCB}}(b_m)$ are given by \eqref{eq:rit} and \eqref{eq:ucb_reward}, respectively, and $r_{i, z}(b_m) = r_{i, z}(t = b_m) = \dfrac{u_{i, z}}{1+u_{i, z}}$.
\end{lemma}

{\bf Proof of Lemma \ref{lemma_rucb_lower_bound}:} \label{proof_lemma_rucb_lower_bound} We know that $r_{i, z}(t)$ is given by: 
\[r_{i, z}(t) = \dfrac{u_{i, z}}{(1 + \sum_{c \in \bC_i(t-1) + \bfe_{z}}n_{c, i}(t)u_{i, c})(1 + \sum_{c \in \bC_i(t-1)}n_{c, i}(t-1)u_{i, c})}.\]

Clearly, $r_{i, z}(t)$ decreases with $n_{c, i}(t)$ for all $c \in \cC$ and $i \in \cS$. Therefore, we have the following:
\begin{align}
    r_{i, z}(b_m) = \dfrac{u_{i, z}}{1+u_{i, z}} = \max_{t \in \set{b_m, \ldots, mK}}\big\{r_{i, z}(t)\big\} \geq r_{i, z}(t). \label{eq:r_max}
\end{align}

Using \eqref{eq:r_max}, we have:
\begin{align}
    &\dfrac{\partial r_{i, z}(b_m)}{\partial u_{i, z}} = \dfrac{1}{(1+u_{i, z})} - \dfrac{u_{i, z}}{(1+u_{i, z})^2} = \dfrac{1}{(1+u_{i, z})^2} > 0. \label{eq:r_max_monotone}
\end{align}

Using \eqref{eq:r_max}, \eqref{eq:r_max_monotone}, $u_{i, z, m-1}^{\textup{UCB}} \geq u_{i, z}$, and $r_{i, z}^{\textup{UCB}}(b_m) = u_{i, z, m-1}^{\textup{UCB}}/(1+u_{i, z, m-1}^{\textup{UCB}})$, we get:
\begin{align*}
    r_{i, z}^{\textup{UCB}}(b_m) \geq r_{i, z}(b_m) \geq r_{i, z}(t).
\end{align*}

\hfill $\square$

\begin{lemma}\label{lemma_exp_deltaM_monotone}
    Suppose that $u_{i, z, m}^{\textup{UCB}} \geq u_{i, z}$ and $v_{z, i, t}^{\textup{UCB}} \geq v_{z, i}$ for all $i \in \cS$ and $z \in \cC$, where $v_{z, i, t}^{\textup{UCB}}$ and $u_{i, z, m}^{\textup{UCB}}$ are given by \eqref{eq:vucb} and \eqref{eq:uucb}, respectively. If $S_t^*$ is the optimal assortment under the clairvoyant policy, and $\pS_t$ is the assortment offered under the TWL-UCB algorithm in period $t$ in an epoch $m$, then we have the following: 
    \begin{align*}
        \uM(\pS_t, \bC(t-1), c_t) \geq \uM(S_t^*, \bC(t-1), c_t) \geq \cM(S_t^*, \bC(t-1), c_t),
    \end{align*}  
    where $\cM(S_t, \bC(t-1), c_t)$ and $\uM(S_t, \bC(t-1), c_t)$ are given by \eqref{eq:deltaM} and \eqref{eq:deltaM_ucb}, respectively.
\end{lemma}

{\bf Proof of Lemma \ref{lemma_exp_deltaM_monotone}:} \label{proof_lemma_exp_deltaM_monotone} Suppose that the customer type in period $t$ is $c_t = z$. Recall from \eqref{eq:deltaM} and \eqref{eq:reward_epoch} in \S \ref{problem} and \eqref{eq:deltaM_ucb} in \S \ref{policy} that, for an epoch $m$, $\cM(S_t, \bC(t-1), c_t)$ and $\uM(S_t, \bC(t-1), c_t)$ are given by:
\begin{equation}
    \begin{aligned}\label{eq:deltaM_def}
        &\cM(S_t, \bC(t-1), c_t) = \sum_{i \in S_t}\dfrac{v_{z, i}r_{i, z}(t)}{1+\sum_{i \in S_t}v_{z, i}}, \\
        &\uM(S_t, \bC(t-1), c_t) = \sum_{i \in S_t}\dfrac{v_{i, z, t-1}^{\textup{UCB}}r_{i, z}^{\textup{UCB}}(b_m)}{1+\sum_{i \in S_t}v_{z, i, t-1}^{\textup{UCB}}}, 
    \end{aligned}
\end{equation}
where $r_{i, z}(t)$ and $r_{i, z}^{\textup{UCB}}(b_m)$ are given by \eqref{eq:rit} and \eqref{eq:ucb_reward}, respectively.

By our algorithm, we have:
\begin{align}
    \uM(\pS_t, \bC(t-1), c_t) &= \max_{S_t \in \sS}\bigg\{\uM(S_t, \bC(t-1), c_t) \bigg\}. \label{eq:deltaM_ineq1}
\end{align}

Using $r_{i, z}^{\textup{UCB}}(b_m) \geq r_{i, z}(t)$ (Lemma \ref{lemma_rucb_lower_bound}) and Lemma \ref{lemma_agrawal_2019}, we have the following:
\begin{align}
    \max_{S_t \in \sS}\bigg\{\uM(S_t, \bC(t-1), c_t) \bigg\} \geq \bigg\{\uM(S_t^*, \bC(t-1), c_t) \bigg\} \geq \cM(S_t^*, \bC(t-1), c_t). \label{eq:deltaM_ineq2}
\end{align}

Combining \eqref{eq:deltaM_ineq1} and \eqref{eq:deltaM_ineq2}, we get:
\begin{align*}
    \uM(\pS_t, \bC(t-1), c_t) \geq \uM(S_t^*, \bC(t-1), c_t) \geq \cM(S_t^*, \bC(t-1), c_t).
\end{align*}

\hfill $\square$

\begin{lemma}\label{lemma_deltaM_difference_upperbound}
    Suppose that $u_{i, z, m}^{\textup{UCB}} \geq u_{i, z}$ and $v_{z, i, t}^{\textup{UCB}} \geq v_{z, i}$ for all $i \in \cS$ and $z \in \cC$, where $v_{z, i, t}^{\textup{UCB}}$ and $u_{i, z, m}^{\textup{UCB}}$ are given by \eqref{eq:vucb} and \eqref{eq:uucb}, respectively. If $\pS_t$ is the assortment offered at time $t$ in epoch $m$ under the TWL-UCB algorithm, then we have the following:
    \begin{align*}
        0 \leq \uM(\pS_t, \bC(t-1), c_t) - \cM(\pS_t, \bC(t-1), c_t) \leq \sum_{i \in \pS_t} \bigg\{\dfrac{v_{c_t, i, t-1}^{\textup{UCB}}r_{i, c_t}^{\textup{UCB}}(b_m) - v_{c_t, i} r_{i, c_t}(t)}{1+\sum_{i \in \pS_t}v_{c_t, i}} \bigg\},
    \end{align*}
    where $r_{i, z}(t)$, $r_{i, z}^{\textup{UCB}}(b_m)$, $\cM(S_t, \bC(t-1), c_t)$, and $\uM(S_t, \bC(t-1), c_t)$ are given by \eqref{eq:rit}, \eqref{eq:ucb_reward}, \eqref{eq:deltaM}, and \eqref{eq:deltaM_ucb}, respectively.
\end{lemma}

{\bf Proof of Lemma \ref{lemma_deltaM_difference_upperbound}:} \label{proof_lemma_deltaM_difference_upperbound} Using \eqref{eq:deltaM}, \eqref{eq:reward_epoch}, and \eqref{eq:deltaM_ucb}, at any time $t$ within an epoch $m$, we have the following:
\begin{align}
    &\uM(\pS_t, \bC(t-1), c_t) - \cM(\pS_t, \bC(t-1), c_t) \nonumber\\
    &=\sum_{i \in \pS_t} \bigg\{\dfrac{v_{c_t, i, t-1}^{\textup{UCB}}r_{i, c_t}^{\textup{UCB}}(b_m)}{1+\sum_{i \in \pS_t}v_{c_t, i, t-1}^{\textup{UCB}}} - \dfrac{v_{c_t, i}r_{i, c_t}(t)}{1+\sum_{i \in \pS_t}v_{c_t, i}} \bigg\} \nonumber\\
    &\leq \sum_{i \in \pS_t} \bigg\{\dfrac{v_{c_t, i, t-1}^{\textup{UCB}}r_{i, c_t}^{\textup{UCB}}(b_m) - v_{c_t, i}r_{i, c_t}(t)}{1+\sum_{i \in \pS_t}v_{c_t, i}} \bigg\}. \label{eq:deltaM_ub} \\
    &\text{(using $v_{c_t, i} \leq v_{c_t, i, t}^{\textup{UCB}}$)}\nonumber
\end{align}

Since $u_{i, c_t, m}^{\textup{UCB}} \geq u_{i, c_t}$, by Lemma \ref{lemma_rucb_lower_bound}, we have $r_{i, c_t}^{\textup{UCB}}(b_m) \geq r_{c_t, i}(t)$. Using $r_{i, c_t}^{\textup{UCB}}(b_m) \geq r_{c_t, i}(t)$ and Lemma \ref{lemma_exp_deltaM_monotone}, we have: 
\begin{align}
    &\uM(\pS_t, \bC(t-1), c_t) \geq \cM(S_t^*, \bC(t-1), c_t) \geq \cM(\pS_t, \bC(t-1), c_t). \label{eq:rucb_lb}
\end{align}

From \eqref{eq:deltaM_ub} and \eqref{eq:rucb_lb}, we finally have:
\begin{align*}
    0 \leq \uM(\pS_t, \bC(t-1), c_t) - \cM(\pS_t, \bC(t-1), c_t) \leq \sum_{i \in \pS_t} \bigg\{\dfrac{v_{c_t, i, t-1}^{\textup{UCB}}r_{i, c_t}^{\textup{UCB}}(b_m) - v_{c_t, i} r_{i, c_t}(t)}{1+\sum_{i \in \pS_t}v_{c_t, i}} \bigg\}.
\end{align*}

\hfill $\square$

{\color{black}
\begin{lemma}\label{lemma_optimal_assortment_size}
    Suppose the parameterization corresponding to an arbitrary assortment $\overline{S} \subseteq \cS$ is $\Theta_{\overline{S}}$. Under this parameterization, the ``single best assortment" that maximizes the total reward over the planning horizon among all time-invariant assortments
    is $\overline{S}$, where $|\overline{S}| = B$.
\end{lemma}

{\bf Proof of Lemma \ref{lemma_optimal_assortment_size}:} \label{proof_lemma_optimal_assortment_size} Using \eqref{eq:rit} and \eqref{eq:param_space}, the reward parameter is identical across all sellers and customer types:
\begin{align}
    &r_{i, z}(t) = \dfrac{u}{1+u}, \; \forall \; i \in \cS,\ z \in \cC. \label{eq:reward_parameter_space} 
\end{align}  

Using \eqref{eq:param_space}, \eqref{eq:gamma_epoch}, and \eqref{eq:reward_parameter_space}, the total reward over the horizon corresponding to $\overline{S}$ under $\Theta_{\overline{S}}$ is given by:
\begin{align}  
    \sum_{m=1}^M\sum_{t=b_m}^{mK}\gamma_m\cM(\overline{S}) = \sum_{m=1}^M\sum_{t=b_m}^{mK}(1+u)\cM(\overline{S}) = \sum_{t=1}^{T}\dfrac{u|\overline{S}|(1+\ep_{c_t})v}{1 + |\overline{S}|(1+\ep_{c_t})v}, \label{eq:expected_reward_S}
\end{align}
which is strictly increasing in $|\overline{S}|$.

Consequently, to maximize the total reward over the horizon, $\overline{S}$ must utilize the maximum allowable capacity. Thus, $|\overline{S}| = B$. 

By the construction of $\Theta_{\overline{S}}$ in \eqref{eq:param_space}, $v_{c_t, i} = v(1+\ep_{c_t})$ for $i \in \overline{S}$, and $v_{c_t, i} = v$ for $i \notin \overline{S}$. Since $\ep_{c_t} > 0$, the sellers belonging to $\overline{S}$ have strictly larger preference parameters than the sellers outside of $\overline{S}$. Because $|\overline{S}| = B$, offering $\overline{S}$ strictly maximizes the sum of preference parameters among all assortments of size $B$. Thus, $\overline{S}$ is the ``single best assortment" under $\Theta_{\overline{S}}$, where $|\overline{S}| = B$. 
\hfill $\square$

}

\section{Consistent Estimators}
\label{consistent_estimators}

\begin{theorem}[Consistent Estimators]\label{thm_consistent_estimator}
    For any instances $\bU$ and $\bV$ of the TWL problem and $0 < \delta \leq 1$, we have the following:
    \begin{itemize}
        \item[(A)] $\lim_{n \to \infty}\prob\bigg(\av_{z, i}(n) - v_{z, i} \geq \delta v_{z, i}\bigg) = 0$ for all $i \in \cS, z \in \cC$.
        \item[(B)] $\lim_{n \to \infty}\prob\bigg(\au_{i, z}(n) - u_{i, z} \geq \delta u_{i, z}\bigg) = 0$  for all $i \in \cS, z \in \cC$.
    \end{itemize}
\end{theorem}

{\bf Proof of Theorem \ref{thm_consistent_estimator}:}\label{proof_thm_consistent_estimator}
Proof of part $(A)$: Let $\hv_{z, i, l} = \I\big[c_l = z, \zeta_l(\pS_l) = i\big] \in \set{0, 1}$. Clearly, $\hv_{z, i, 1}, \ldots, \hv_{z, i, n}$ form an independent sequence of $n$ Poisson trials, where $\prob(\hv_{z, i, l}) = p_{z, i}(\pS_l) = v_{z, i}/(1+\sum_{j \in \pS_l}v_{z, j})$. Let $\hv_{z, i}(n) = \sum_{l = 1}^n \hv_{z, i, l}$ and $\mu_n = \E\bigg[\hv_{z, i}(n)\bigg]$, where we have:
\begin{align}
    \mu_n = \E\bigg[\hv_{z, i}(n)\bigg] = \E\bigg[\sum_{l = 1}^n \hv_{z, i, l}\bigg] = \sum_{l=1}^n\dfrac{v_{z, i}}{1+\sum_{j \in \pS_l}v_{z, j}} \leq nv_{z, i}. \label{eq:biased_estimator}  
\end{align}

Using the definition of moment generating function, we have:
\begin{align}
    &\E\bigg[\exp{\lambda\hv_{z, i}(n)}\bigg] \nonumber\\ 
    &= \prod_{l=1}^{n}\E\bigg[\exp{\big(\lambda \hv_{z, i, l}\big)}\bigg] \nonumber\\
    &= \prod_{l=1}^{n} p_{z, i}(\pS_l)e^{\lambda} + (1-p_{z, i}(\pS_l)) \nonumber\\
    &= \prod_{l=1}^{n} 1 + p_{z, i}(\pS_l)(e^{\lambda}-1) \nonumber\\
    &\leq \prod_{l=1}^{n} \exp{\bigg(p_{z, i}(\pS_l)(e^\lambda-1)\bigg)} \tag*{(using \mbox{$e^x \geq 1 + x$})} \nonumber\\
    &= \exp{\Bigg((e^\lambda-1)\sum_{l=1}^{n}p_{z, i}(\pS_l)\Bigg)} \nonumber \\
    &= \exp{\Bigg((e^\lambda-1)\E\bigg[\sum_{l=1}^{n}v_{z, i, l}\bigg]\Bigg)} = \exp{\Bigg((e^\lambda-1)\E\bigg[\hv_{z, i}(n)\bigg]\Bigg)} = \exp{\bigg((e^\lambda-1)\mu_n\bigg)}. \label{eq:mgf}
\end{align}

Using Markov's inequality, for any $\delta > 0$ we get:
\begin{align}
    \prob\bigg(\hv_{z, i}(n) \geq (1+\delta)nv_{z, i} \bigg) &= \prob\bigg(e^{\lambda\hv_{z, i}(n)} \geq e^{\lambda (1+\delta)nv_{z, i}}\bigg) \nonumber\\
    &\leq \dfrac{\E\bigg[e^{\lambda\hv_{z, i}(n)}\bigg]}{e^{\lambda (1+\delta)nv_{z, i}}} \leq \dfrac{e^{(e^\lambda-1)\mu_n}}{e^{\lambda (1+\delta)nv_{z, i}}} \leq \dfrac{e^{(e^\lambda-1)nv_{z, i}}}{e^{\lambda (1+\delta)nv_{z, i}}}. \label{eq:markov_ineq}\\
    &\text{(using \eqref{eq:biased_estimator} and \eqref{eq:mgf})} \nonumber
\end{align}

Setting $\lambda = \log(1+\delta)$ in \eqref{eq:markov_ineq}, we get:
\begin{align}
    \prob\bigg(\hv_{z, i}(n) \geq (1+\delta)nv_{z, i} \bigg) &\leq \bigg(\dfrac{e^{\delta}}{(1+\delta)^{(1+\delta)}} \bigg)^{nv_{z, i}}. \label{eq:markov_ineq2}
\end{align}

Using Theorem 4.4 from \citet[~p. 69]{mitzenmacher2017probability} and \eqref{eq:markov_ineq2}, for $\delta \in (0, 1]$ we get:
\begin{align}
    &\prob\bigg(\hv_{z, i}(n) \geq (1+\delta)nv_{z, i} \bigg) \leq e^{-nv_{z, i}\delta^2/3} \nonumber \\
    \iff \quad &\prob\bigg(\hv_{z, i}(n)/n \geq (1+\delta)v_{z, i} \bigg) \leq e^{-nv_{z, i}\delta^2/3} \nonumber \\
    \iff \quad &\prob\bigg(\av_{z, i}(n) \geq (1+\delta)v_{z, i} \bigg) \leq e^{-nv_{z, i}\delta^2/3}, \label{eq:prob_ineq}
\end{align}
where $\av_{z, i}(n) = \hv_{z, i}(n)/n$. Taking the limiting value in \eqref{eq:prob_ineq} as $n \rightarrow \infty$, we get $\lim_{n \to \infty} \prob\big(\av_{z, i}(n) - v_{z, i} \geq \delta v_{z, i} \big) \leq \lim_{n \to \infty} e^{-nv_{z, i}\delta^2/3} = 0$. This concludes the proof of part $(A)$. 

Similarly, we can prove part $(B)$ to show that $\lim_{n \to \infty}\prob\bigg(\au_{i, z}(n) - u_{i, z} \geq \delta u_{i, z}\bigg) = 0$.

\hfill $\square$ 

\section{Proofs of Lemmas and Theorem in the Main Text: Upper Bound}
\label{ec_thm_1}

\subsection{Proof of Lemma \ref{lemma_ucb}}
\label{proof_lemma_ucb}

Proof of parts $(A.1)$ and $(A.2)$:
From Lemma \ref{lemma_kleinberg_2008} (in \S \ref{supplementary_lemmas}), with probability at least $1 - e^{-\Omega(\alpha)}$, we have:
\begin{align}
    \bigg|\av_{z, i}(t) - \Ev\bigg| \leq r(\alpha, \av_{z, i}(t)) \leq 3r(\alpha, \Ev), \label{eq:lemma_ec1} 
\end{align}
where $\av_{z, i}(t)$ is given by:
\begin{align}
    \av_{z, i}(t) = \dfrac{\sum_{l \leq t}\I\big[c_l = z, i \in \pS_l, \zeta_l(\pS_l) = i\big]}{\sum_{l \leq t}\I\big[c_l = z, i \in \pS_l\big]} = \dfrac{\hv_{z, i}(t)}{\omega_{i, z}(t)} \leq 1. \label{eq:v_bar}
\end{align}

We use the indicator variable $\I\big[c_l = z, i \in \pS_l, \zeta_l(\pS_l) = i\big]$ to estimate the true parameter $v_{z, i}$. The expected value of the indicator variable is given by:
\begin{align}
    &\E\bigg[\I\big[c_l = z, i \in \pS_l, \zeta_l(\pS_l) = i\big]\bigg] = \dfrac{v_{z, i}}{1 + \sum_{j \in \pS_l}v_{z, j}} \leq v_{z, i}. \label{eq:indicator_variable}
\end{align}

We therefore have the following:   
\begin{align}
    &\Ev \leq v_{z, i} \leq 1, \; \text{for all $t$}. \label{eq:biased_estimate} \\
    &\bigg|\av_{z, i}(t) - v_{z, i}\bigg| \leq 1, \; \text{with probability 1}. \label{eq:error_estimate1} \\
    &\bigg|\av_{z, i}(t) - \Ev\bigg| \leq 1, \; \text{with probability 1}. \label{eq:error_estimate2}
\end{align} 

Setting $\Omega(\alpha) = \alpha/C - \log(2)$, where $C \in \R_+$, we have $1 - e^{-\Omega(\alpha)} = 1 - 2e^{-\alpha/C}$. For our analysis, we let $\alpha = C\log(Nt)$. Using $\alpha = C\log(Nt)$, we get $1 - 2e^{-\alpha/C} = 1 - 2e^{-\log(Nt)} = 1 - 2/Nt$. Therefore from \eqref{eq:lemma_ec1} and \eqref{eq:biased_estimate}, with probability $> 1 - 2/Nt$, we get:
\begin{align}
    \bigg|\av_{z, i}(t) - \Ev\bigg| \leq r(\alpha, \av_{z, i}(t)) \leq 3r(\alpha, \Ev) \leq 3r(\alpha, v_{z, i}), \label{eq:kleinberg_ineq}
\end{align}
where $r(\alpha, X)$ is given by:
\begin{align}
    r(\alpha, X) = \dfrac{C\log(Nt)}{\omega_{i, z}(t)} + \sqrt{\dfrac{C\log(Nt)}{\omega_{i, z}(t)}X}. \label{eq:ralpha}
\end{align}

From \eqref{eq:kleinberg_ineq}, \eqref{eq:ralpha}, and Lemma \ref{lemma_alpha_lower_bound} (in \S \ref{supplementary_lemmas}), with probability at least $1 - 2/Nt$, we get:
\begin{align}
    &\omega_{i, z}(t) \leq C\log{(Nt)} \nonumber \\
    \implies \quad &r(\alpha, \av_{z, i}(t)) \geq 1. \label{eq:ralpha_lower_bound}
\end{align}

From \eqref{eq:kleinberg_ineq}, with probability at least $1-2/Nt$, we then have:
\begin{align}
    &\bigg|\av_{z, i}(t) - \Ev\bigg| \leq r(\alpha, \av_{z, i}(t)) \nonumber \\
    \iff \quad &-r(\alpha, \av_{z, i}(t)) \leq \av_{z, i}(t) - \Ev \leq r(\alpha, \av_{z, i}(t)) \label{eq:mean_estimate_bound} \\
    \implies \quad &-r(\alpha, \av_{z, i}(t)) \leq \av_{z, i}(t) - v_{z, i} \leq \av_{z, i}(t) - \Ev \leq r(\alpha, \av_{z, i}(t)) \tag*{(using \eqref{eq:biased_estimate}, \eqref{eq:error_estimate1}, and \eqref{eq:ralpha_lower_bound})} \nonumber \\
    \implies \quad &-\dfrac{C\log(Nt)}{\omega_{i, z}(t)} - \sqrt{\dfrac{C\log(Nt)}{\omega_{i, z}(t)}\av_{z, i}(t)} \leq \av_{z, i}(t) - v_{z, i} \tag*{(using \eqref{eq:ralpha})}  \nonumber \\
    \iff \quad &\av_{z, i}(t) + \dfrac{C\log(Nt)}{\omega_{i, z}(t)} + \sqrt{\dfrac{C\log(Nt)}{\omega_{i, z}(t)}\av_{z, i}(t)} \geq v_{z, i}.  \label{eq:ucb_inequality}
\end{align}

For random variables $X, Y,$ we know that $\prob(X \; \text{or} \; Y) < p$ implies that $\prob(X) < p$ and $\prob(Y) < p$. Using this fact and \eqref{eq:kleinberg_ineq}, we have:
\begin{align}
    &\prob\Bigg(\bigg|\av_{z, i}(t) - \Ev\bigg| > r(\alpha, \av_{z, i}(t))\Bigg) < \dfrac{2}{Nt} \nonumber\\
    \iff \quad &\prob\bigg(\av_{z, i}(t) - \Ev > r(\alpha, \av_{z, i}(t)) \;\; \text{or} \;\; \av_{z, i}(t) - \Ev < -r(\alpha, \av_{z, i}(t))\bigg) < \dfrac{2}{Nt}  \nonumber \\
    \implies \quad &\prob\bigg(\av_{z, i}(t) - \Ev < -r(\alpha, \av_{z, i}(t))\bigg) < \dfrac{2}{Nt} \nonumber \\
    \implies \quad &\prob\bigg(\av_{z, i}(t) - v_{z, i} \leq \av_{z, i}(t) - \Ev < -r(\alpha, \av_{z, i}(t))\bigg) < \dfrac{2}{Nt} \tag*{(using \eqref{eq:biased_estimate})} \nonumber \\
    \iff \quad &\prob\bigg(\av_{z, i}(t) - v_{z, i} < -\dfrac{C\log(Nt)}{\omega_{i, z}(t)} - \sqrt{\dfrac{C\log(Nt)}{\omega_{i, z}(t)}\av_{z, i}(t)}\bigg) < \dfrac{2}{Nt} \tag*{(using \eqref{eq:ralpha})}  \nonumber \\
    \iff \quad &\prob\bigg(\av_{z, i}(t) + \dfrac{C\log(Nt)}{\omega_{i, z}(t)} + \sqrt{\dfrac{C\log(Nt)}{\omega_{i, z}(t)}\av_{z, i}(t)} < v_{z, i}\bigg) < \dfrac{2}{Nt}.  \label{eq:ucb_inequality_v2}
\end{align}

From \eqref{eq:vucb}, we have: 
\begin{align}
    v_{z, i, t}^{\textup{UCB}} = \av_{z, i}(t) + \dfrac{C\log(Nt)}{\omega_{i, z}(t)} + \sqrt{\dfrac{C\log(Nt)}{\omega_{i, z}(t)}\av_{z, i}(t)}. \label{eq:ucb_expression}
\end{align}

Therefore from \eqref{eq:ucb_inequality}, \eqref{eq:ucb_inequality_v2}, and \eqref{eq:ucb_expression}, with probability at least $1-2/Nt$, we get $v_{z, i, t}^{\textup{UCB}} \geq v_{z, i}$, or equivalently, with probability less than $2/Nt$, we have $v_{z, i, t}^{\textup{UCB}} < v_{z, i}$. This concludes the proof of part $(A.1)$ of Lemma \ref{lemma_ucb}. 

From \eqref{eq:mean_estimate_bound}, with probability at least $1-2/Nt$, we have:
\begin{align}
    &-r(\alpha, \av_{z, i}(t)) \leq \av_{z, i}(t) - \Ev \leq r(\alpha, \av_{z, i}(t)) \nonumber \\
    \implies \quad &\av_{z, i}(t) + r(\alpha, \av_{z, i}(t)) - \Ev \leq 2r(\alpha, \av_{z, i}(t)) \nonumber\\
    \implies \quad &v_{z, i, t}^{\textup{UCB}} -\Ev \leq 6r(\alpha, \Ev) \nonumber \tag*{(using \eqref{eq:lemma_ec1} and \eqref{eq:ucb_expression})}\\
    \implies \quad &v_{z, i, t}^{\textup{UCB}} -\Ev \leq \beta_1\dfrac{\log(Nt)}{\omega_{i, z}(t)} + \beta_2\sqrt{\dfrac{\log(Nt)}{\omega_{i, z}(t)}\Ev}, \nonumber
\end{align}
where $\beta_j \in \R_+$ for $j \in \set{1, 2}$. Using a similar approach, we can show that:
\begin{align*}
    \prob\bigg(v_{z, i, t}^{\textup{UCB}} -\Ev > \beta_1\dfrac{\log(Nt)}{\omega_{i, z}(t)} + \beta_2\sqrt{\dfrac{\log(Nt)}{\omega_{i, z}(t)}\Ev}\bigg) < 2/Nt.
\end{align*}

This concludes the proof of part $(A.2)$ of Lemma \ref{lemma_ucb}. 

Proof of parts $(B.1)$ and $(B.2)$: The proof of parts $(B.1)$ and $(B.2)$ are similar to that of $(A.1)$ and $(A.2)$. For completeness, we present the proof below.

Similar to \eqref{eq:lemma_ec1}, with probability at least $1 - e^{-\Omega(\alpha)}$, we have:
\begin{align}
    \bigg|\au_{i, z}(m) - \Eu\bigg| \leq r(\alpha, \au_{i, z}(m)) \leq 3r(\alpha, \Eu), \label{eq:lemma_ec1_v2} 
\end{align}
where $\au_{i, z}(m) = \dfrac{\hu_{i, z}(m)}{\theta_{z, i}(m)} \leq 1$, and $\theta_{z, i}(m)$ and $\hu_{i, z}(m)$ are defined in \eqref{eq:proposal_set_size} and \eqref{eq:match_count}, respectively. 

Similar to \eqref{eq:biased_estimate} through \eqref{eq:error_estimate2} for $v_{z, i}$, we have the following inequalities for $u_{i, z}$: 
\begin{align}
    \Eu \leq u_{i, z} &\leq 1, \; \text{for all $m$}. \label{eq:biased_estimate_u} \\
    \bigg|\au_{i, z}(m) - u_{i, z}\bigg| &\leq 1, \; \text{with probability 1}. \label{eq:error_estimate1_u} \\
    \bigg|\au_{i, z}(m) - \Eu\bigg| &\leq 1, \; \text{with probability 1}. \label{eq:error_estimate2_u}
\end{align} 

In this case, we let $\alpha = C\log(Nm)$. Using $\alpha = C\log(Nm)$, we get $1 - 2e^{-\alpha/C} = 1 - 2/Nm$. Therefore from \eqref{eq:lemma_ec1_v2} and \eqref{eq:biased_estimate_u}, with probability $> 1 - 2/Nm$, we get:
\begin{align}
    \bigg|\au_{i, z}(m) - \Eu\bigg| \leq r(\alpha, \au_{i, z}(m)) \leq 3r(\alpha, \Eu) \leq 3r(\alpha, u_{i, z}), \label{eq:kleinberg_ineq_v2}
\end{align}
where $r(\alpha, Y)$ is given by:
\begin{align}
    r(\alpha, Y) = \dfrac{C\log(Nm)}{\theta_{z, i}(m)} + \sqrt{\dfrac{C\log(Nm)}{\theta_{z, i}(m)}Y}. \label{eq:ralpha1}
\end{align}

From \eqref{eq:kleinberg_ineq_v2}, \eqref{eq:ralpha1}, and Lemma \ref{lemma_alpha_lower_bound} (in \S \ref{supplementary_lemmas}), with probability at least $1 - 2/Nm$, we get:
\begin{align}
    &\theta_{z, i}(m) \leq C\log{(Nm)} \nonumber \\
    \implies \quad &r(\alpha, \au_{i, z}(m)) \geq 1.  \label{eq:ralpha1_lower_bound}
\end{align}

From \eqref{eq:kleinberg_ineq_v2}, with probability at least $1-2/Nm$, we then have:
\begin{align}
    &\bigg|\au_{i, z}(m) - \Eu\bigg| \leq r(\alpha, \au_{i, z}(m)) \nonumber \\
    \iff \quad &-r(\alpha, \au_{i, z}(m)) \leq \au_{i, z}(m) - \Eu \leq r(\alpha, \au_{i, z}(m)) \label{eq:mean_estimate_bound_v2} \\
    \implies \quad &-r(\alpha, \au_{i, z}(m)) \leq \au_{i, z}(m) - u_{i, z} \leq \au_{i, z}(m) - \Eu \leq r(\alpha, \au_{i, z}(m)) \tag*{(using \eqref{eq:biased_estimate_u}, \eqref{eq:error_estimate1_u}, and \eqref{eq:ralpha1_lower_bound})} \nonumber \\
    \implies \quad &-\dfrac{C\log(Nm)}{\theta_{z, i}(m)} - \sqrt{\dfrac{C\log(Nm)}{\theta_{z, i}(m)}\au_{i, z}(m)} \leq \au_{i, z}(m) - u_{i, z} \tag*{(using \eqref{eq:ralpha1})}  \nonumber \\
    \iff \quad &\au_{i, z}(m) + \dfrac{C\log(Nm)}{\theta_{z, i}(m)} + \sqrt{\dfrac{C\log(Nm)}{\theta_{z, i}(m)}\au_{i, z}(m)} \geq u_{i, z}.  \label{eq:ucb_inequality_v3}
\end{align}

For random variables $X, Y,$ we know that $\prob(X \; \text{or} \; Y) < p$ implies that $\prob(X) < p$ and $\prob(Y) < p$. Using this fact and \eqref{eq:kleinberg_ineq_v2}, we have:
\begin{align}
    &\prob\Bigg(\bigg|\au_{i, z}(m) - \Eu\bigg| > r(\alpha, \au_{i, z}(m))\Bigg) < \dfrac{2}{Nm} \nonumber\\
    \iff \quad &\prob\bigg(\au_{i, z}(m) - \Ev > r(\alpha, \au_{i, z}(m)) \;\; \text{or} \;\; \au_{i, z}(m) - \Eu < -r(\alpha, \au_{i, z}(m))\bigg) < \dfrac{2}{Nm}  \nonumber \\
    \implies \quad &\prob\bigg(\au_{i, z}(m) - \Ev < -r(\alpha, \au_{i, z}(m))\bigg) < \dfrac{2}{Nm} \nonumber \\
    \implies \quad &\prob\bigg(\au_{i, z}(m) - u_{i, z} \leq \au_{i, z}(m) - \Eu < -r(\alpha, \au_{i, z}(m))\bigg) < \dfrac{2}{Nm} \tag*{(using \eqref{eq:biased_estimate_u})} \nonumber \\
    \iff \quad &\prob\bigg(\au_{i, z}(m) - u_{i, z} < -\dfrac{C\log(Nm)}{\theta_{z, i}(m)} - \sqrt{\dfrac{C\log(Nm)}{\theta_{z, i}(m)}\au_{i, z}(m)}\bigg) < \dfrac{2}{Nm} \tag*{(using \eqref{eq:ralpha1})}  \nonumber \\
    \iff \quad &\prob\bigg(\au_{i, z}(m) + \dfrac{C\log(Nm)}{\theta_{z, i}(m)} + \sqrt{\dfrac{C\log(Nm)}{\theta_{z, i}(m)}\au_{i, z}(m)} < u_{i, z}\bigg) < \dfrac{2}{Nm}.  \label{eq:ucb_inequality_v4}
\end{align}

From \eqref{eq:uucb}, we have: 
\begin{align}
    u_{i, z, m}^{\textup{UCB}} = \au_{i, z}(m) + \dfrac{C\log(Nm)}{\theta_{z, i}(m)} + \sqrt{\dfrac{C\log(Nm)}{\theta_{z, i}(m)}\au_{i, z}(m)}. \label{eq:ucb_expression_v2}
\end{align}

Therefore from \eqref{eq:ucb_inequality_v3}, \eqref{eq:ucb_inequality_v4}, and \eqref{eq:ucb_expression_v2}, with probability at least $1-2/Nm$, we get $u_{i, z, m}^{\textup{UCB}} \geq u_{i, z}$, or equivalently, with probability less than $2/Nm$, we have $u_{i, z, m}^{\textup{UCB}} < u_{i, z}$. This concludes the proof of part $(B.1)$ of Lemma \ref{lemma_ucb}. 

From \eqref{eq:mean_estimate_bound_v2}, with probability at least $1-2/Nm$, we have:
\begin{align}
    &-r(\alpha, \au_{i, z}(m)) \leq \au_{i, z}(m) - \Eu \leq r(\alpha, \au_{i, z}(m)) \nonumber \\
    \implies \quad &\au_{i, z}(m) + r(\alpha, \au_{i, z}(m)) - \Eu \leq 2r(\alpha, \au_{i, z}(m)) \nonumber\\
    \implies \quad &u_{i, z, m}^{\textup{UCB}} -\Eu \leq 6r(\alpha, \Eu) \nonumber \tag*{(using \eqref{eq:lemma_ec1_v2} and \eqref{eq:ucb_expression_v2})}\\
    \implies \quad &u_{i, z, m}^{\textup{UCB}} -\Eu \leq \beta_3\dfrac{\log(Nm)}{\theta_{z, i}(m)} + \beta_4\sqrt{\dfrac{\log(Nm)}{\theta_{z, i}(m)}\Eu}, \nonumber
\end{align}
where $\beta_j \in \R_+$ for $j \in \set{3, 4}$. Using a similar approach, we can show that:
\begin{align*}
    \prob\bigg(u_{i, z, m}^{\textup{UCB}} -\Eu > \beta_3\dfrac{\log(Nm)}{\theta_{z, i}(m)} + \beta_4\sqrt{\dfrac{\log(Nm)}{\theta_{z, i}(m)}\Eu}\bigg) < 2/Nm.
\end{align*}

This concludes the proof of part $(B.2)$ of Lemma \ref{lemma_ucb}.

\hfill $\square$

\subsection{Proof of Lemma \ref{lemma_optimal_value_bound}}
\label{proof_lemma_optimal_value_bound}

Recall the definitions of $\omega_{i, z}(t), \av_{z, i}(t), v_{z, i, t}^{\textup{UCB}}, \theta_{z, i}(m), \au_{i, z}(m)$, $u_{i, z, m}^{\textup{UCB}}$, $\cM(S_t, \bC(t-1), c_t)$, $\uM(S_t, \bC(t-1), c_t)$, $S_t^*$, and $\pS_t$ from \eqref{eq:assortment_count}, \eqref{eq:avg_proposal_count}, \eqref{eq:vucb}, \eqref{eq:proposal_set_size}, \eqref{eq:avg_match_count}, \eqref{eq:uucb}, \eqref{eq:deltaM}, \eqref{eq:deltaM_ucb}, \eqref{eq:optimal_assortment_oracle}, and \eqref{eq:optimistic_assortment}, respectively. By Lemma \ref{lemma_ucb}, we have the following:
\begin{equation}
    \begin{aligned}\label{eq:ucb_ub}
        &\prob\bigg(v_{z, i, t}^{\textup{UCB}} < v_{z, i}\bigg) < \dfrac{2}{Nt},\\
        &\prob\bigg(u_{i, z, m}^{\textup{UCB}} < u_{i, z}\bigg) < \dfrac{2}{Nm}.
    \end{aligned}
\end{equation}
\begin{equation}
        \begin{aligned}\label{eq:ucb_lb}
        &\prob \bigg(v_{z, i, t}^{\textup{UCB}} - \Ev > \beta_1\dfrac{\log(Nt)}{\omega_{i, z}(t)} + \beta_2\sqrt{\dfrac{\log(Nt)}{\omega_{i, z}(t)}\Ev} \bigg) < \dfrac{2}{Nt},\\
        &\prob \bigg(u_{i, z, m}^{\textup{UCB}} - \Eu > \beta_3\dfrac{\log(Nm)}{\theta_{z, i}(m)} + \beta_4\sqrt{\dfrac{\log(Nm)}{\theta_{z, i}(m)}\Eu} \bigg) < \dfrac{2}{Nm}.
    \end{aligned}
\end{equation}

We define $q(\omega_{i, z}(t), \Ev)$ and $q(\theta_{z, i}(m), \Eu)$ for $\Ev$ and $\Eu$, respectively, as follows:
\begin{equation}
    \begin{aligned}\label{eq:q_definition_v2}
        q(\omega_{i, z}(t), \Ev) &\doteq \beta_1\dfrac{\log(Nt)}{\omega_{i, z}(t)} + \beta_2\sqrt{\dfrac{\log(Nt)}{\omega_{i, z}(t)}\Ev}, \\
        q(\theta_{z, i}(m), \Eu) &\doteq \beta_3\dfrac{\log(Nm)}{\theta_{z, i}(m)} + \beta_4\sqrt{\dfrac{\log(Nm)}{\theta_{z, i}(m)}\Eu}.
    \end{aligned}
\end{equation}

In the proof of Lemma \ref{lemma_ucb} (see \eqref{eq:ralpha_lower_bound} and \eqref{eq:ralpha1_lower_bound} in \S \ref{proof_lemma_ucb}), we prove the following:
\begin{equation}
    \begin{aligned}\label{eq:theta_lower_bound}
        &\prob\bigg(\omega_{i, z}(t) \leq \beta_1\log{(Nt)} \bigg) \geq 1-\dfrac{2}{Nt}, \\
        &\prob\bigg(\theta_{z, i}(m) \leq \beta_3\log{(Nm)}\bigg) \geq 1-\dfrac{2}{Nm}.
    \end{aligned}
\end{equation}

Using \eqref{eq:q_definition_v2} and \eqref{eq:theta_lower_bound}, we get:
\begin{equation}
    \begin{aligned}\label{eq:q_lower_bound}
        &\prob\bigg(q(\omega_{i, z}(t), \Ev) \geq 1 \bigg) \geq 1 - \dfrac{2}{Nt}, \\
        &\prob\bigg(q(\theta_{z, i}(m), \Eu) \geq 1 \bigg) \geq 1 - \dfrac{2}{Nm}. 
    \end{aligned}
\end{equation}

Define $E$ and $G$ as follows:
\begin{align}
    E &\doteq \bigcup_{i = 1}^N\left[\left\{q(\theta_{z, i}(m), \Eu) < 1\right\} \bigcup \left\{v_{z, i, t}^{\textup{UCB}} - \Ev > q(\omega_{i, z}(t), \Ev)\right\}\right], \nonumber \\
    G &\doteq \bigcup_{i=1}^N \left[\left\{u_{i, z, m}^{\textup{UCB}} < u_{i, z}\right\} \bigcup \set{v_{z, i, t}^{\textup{UCB}} < v_{i, z}}\right]. \nonumber 
\end{align}

Using \eqref{eq:ucb_lb}, \eqref{eq:q_lower_bound}, and $m \leq t$, we get $\prob(E) \leq 4/m$ by union bound. Similarly, using \eqref{eq:ucb_ub} and $m \leq t$, we get $\prob(G) \leq 4/m$ by union bound. We therefore get $\prob(E \cup G) \leq 8/m$ by union bound.  

Using the definitions of $E$ and $G$, we have the following:
\begin{align}
    \prob\big(E^{\complement}\big) &= \prob\left(\bigcap_{i=1}^N\left[\left\{1 \leq q(\theta_{z, i}(m), \Eu)\right\} \bigcap \left\{v_{z, i, t}^{\textup{UCB}} -\Ev \leq q(\omega_{i, z}(t), \Ev)\right\}\right]\right) \nonumber\\
    &\geq 1-\dfrac{4}{m} \label{eq:prob_ec} \\
    \prob\big(G^{\complement}\big) &= \prob\left(\bigcap_{i=1}^N \left[\left\{u_{i, z, m}^{\textup{UCB}} \geq u_{i, z}\right\} \bigcap \left\{v_{z, i, t}^{\textup{UCB}} \geq v_{z, i}\right\}\right]\right) \nonumber\\
    &\geq 1-\dfrac{4}{m}. \label{eq:prob_gc}
\end{align}

Using the above expressions and De Morgan's Law, we get:
\begin{align}
    \prob\big(E^{\complement} \cap G^{\complement}\big) = 1 - \prob\big(E \cup G\big) \geq 1 - 8/m. \label{eq:lemma2_ineq1}
\end{align}

Using \eqref{eq:prob_gc} and Lemma \ref{lemma_rucb_lower_bound} (in \S \ref{supplementary_lemmas}), we have:
\begin{align}
    &\prob\left(\bigcap_{i=1}^N\left[\left\{r_{i, z}^{\textup{UCB}}(b_m) \geq r_{i, z}(t)\right\} \bigcap \left\{v_{z, i, t}^{\textup{UCB}} \geq v_{z, i}\right\}\right]\right) \geq 1 - \dfrac{4}{m}, \label{eq:lemma2_ineq2}
\end{align}
where $r_{i, z}(t)$ and $r_{i, z}^{\textup{UCB}}(b_m)$ are given by \eqref{eq:rit} and \eqref{eq:ucb_reward}, respectively. 

We know that $r_{i, z}(t), r_{i, z}^{\textup{UCB}}(b_m) < 1$ (by definition), $\Ev \leq v_{z, i}$ (from \eqref{eq:biased_estimate}), and $\Eu \leq u_{i, z}$ (from \eqref{eq:biased_estimate_u}). Using this and \eqref{eq:prob_ec}, we have
\begin{align}
    &\prob\left(\bigcap_{i=1}^N\left[\left\{r_{i, z}^{\textup{UCB}}(b_m) -r_{i, z}(t) \leq q(\theta_{z, i}(m), u_{i, z})\right\} \bigcap \left\{v_{z, i, t}^{\textup{UCB}} - v_{z, i} \leq q(\omega_{i, z}(t), v_{z, i})\right\}\right]\right) \geq 1-\dfrac{4}{m}, \nonumber
\end{align}
which leads to the following:
\begin{align}
    \prob\left(\bigcap_{i=1}^N\bigg[v_{z, i, t}^{\textup{UCB}}r_{i, z}^{\textup{UCB}}(b_m) \leq \big(r_{i, z}(t) + q(\theta_{z, i}(m), u_{i, z})\big)\big(v_{z, i} + q(\omega_{i, z}(t), v_{z, i})\big)\bigg]\right) \geq 1-\dfrac{4}{m}. \nonumber 
\end{align}

Simplifying the above expression and using $r_{i, z}(t), r_{i, z}^{\textup{UCB}}(b_m) < 1$, we get:
\begin{align}
    &\prob\left(\bigcap_{i=1}^N\bigg[v_{z, i, t}^{\textup{UCB}}r_{i, z}^{\textup{UCB}}(b_m) - v_{z, i} r_{i, z}(t) \leq q(\theta_{z, i}(m), u_{i, z}) + q(\omega_{i, z}(t), v_{z, i}) + q(\theta_{z, i}(m), u_{i, z})q(\omega_{i, z}(t), v_{z, i})\bigg]\right) \nonumber\\
    &\geq 1-\dfrac{4}{m}. \label{eq:lemma2_ineq3}
\end{align}

From \eqref{eq:lemma2_ineq1} and \eqref{eq:lemma2_ineq2}, we know that $\bigcap_{i=1}^N\left[\left\{r_{i, z}^{\textup{UCB}}(b_m) \geq r_{i, z}(t)\right\} \bigcap \left\{v_{z, i, t}^{\textup{UCB}} \geq v_{z, i}\right\}\right] \supseteq G^{\complement}$. Similarly, from \eqref{eq:lemma2_ineq2} and \eqref{eq:lemma2_ineq3}, we have $\bigcap_{i=1}^N\bigg[v_{z, i, t}^{\textup{UCB}}r_{i, z}^{\textup{UCB}}(b_m) - v_{z, i} r_{i, z}(t) \leq q(\theta_{z, i}(m), u_{i, z}) + q(\omega_{i, z}(t), v_{z, i}) + q(\theta_{z, i}(m), u_{i, z})q(\omega_{i, z}(t), v_{z, i})\bigg] \supseteq E^{\complement}$. Thus, we have the following: 
\begin{align*}
    \prob\Bigg(&\bigcap_{i=1}^N\bigg[v_{z, i, t}^{\textup{UCB}}r_{i, z}^{\textup{UCB}}(b_m) - v_{z, i} r_{i, z}(t) \leq q(\theta_{z, i}(m), u_{i, z}) + q(\omega_{i, z}(t), v_{z, i}) + q(\theta_{z, i}(m), u_{i, z})q(\omega_{i, z}(t), v_{z, i})\bigg], \\
    &\bigcap_{i=1}^N\left[\left\{r_{i, z}^{\textup{UCB}}(b_m) \geq r_{i, z}(t)\right\} \bigcap \left\{v_{z, i, t}^{\textup{UCB}} \geq v_{z, i}\right\}\right] \Bigg) \\
    &\geq \prob\bigg(E^{\complement} \cap G^{\complement}\bigg) \geq 1-\dfrac{8}{m}.
\end{align*}

Using the above information and combining the results from \eqref{eq:lemma2_ineq1}, \eqref{eq:lemma2_ineq2}, and \eqref{eq:lemma2_ineq3}, with probability at least $1-8/m$, we get the following:
\begin{align}
    \bigcap_{i=1}^N\bigg[0 &\leq v_{z, i, t}^{\textup{UCB}}r_{i, z}^{\textup{UCB}}(b_m) - v_{z, i} r_{i, z}(t)
    \leq q(\theta_{z, i}(m), u_{i, z}) + q(\omega_{i, z}(t), v_{z, i}) + q(\theta_{z, i}(m), u_{i, z})q(\omega_{i, z}(t), v_{z, i})\bigg]. \nonumber
\end{align}

Finally, combining the above result and Lemma \ref{lemma_deltaM_difference_upperbound} (in \S \ref{supplementary_lemmas}), with probability at least $1-8/m$, we get:
\begin{align}
    0 \leq &\uM(\pS_t, \bC(t-1), c_t = z) - \cM(\pS_t, \bC(t-1), c_t = z) \nonumber\\
    &\leq \sum_{i \in \pS_t} \bigg\{\dfrac{v_{z, i, t-1}^{\textup{UCB}}r_{i, z}^{\textup{UCB}}(b_m) - v_{z, i}r_{i, z}(t)}{1+\sum_{i \in \pS_t}v_{z, i}} \bigg\} \nonumber \\
    &\leq \sum_{i \in \pS_t}\bigg(q(\theta_{z, i}(m), u_{i, z}) + q(\omega_{i, z}(t), v_{z, i}) + q(\theta_{z, i}(m), u_{i, z})q(\omega_{i, z}(t), v_{z, i})\bigg). \nonumber
\end{align}

\hfill $\square$

\subsection{Proof of Theorem \ref{thm_regret_upper_bound}}
\label{proof_thm_regret_upper_bound}

Recall the definitions of $\omega_{i, z}(t), \av_{z, i}(t), v_{z, i, t}^{\textup{UCB}}, \theta_{z, i}(m), \au_{i, z}(m)$, $u_{i, z, m}^{\textup{UCB}}$, $\cM(S_t, \bC(t-1), c_t)$, $\uM(S_t, \bC(t-1), c_t)$, $S_t^*$, and $\pS_t$ from \eqref{eq:assortment_count}, \eqref{eq:avg_proposal_count}, \eqref{eq:vucb}, \eqref{eq:proposal_set_size}, \eqref{eq:avg_match_count}, \eqref{eq:uucb}, \eqref{eq:deltaM}, \eqref{eq:deltaM_ucb}, \eqref{eq:optimal_assortment_oracle}, and \eqref{eq:optimistic_assortment}, respectively.

Using \eqref{eq:regret_def}, the regret of the TWL-UCB algorithm is given by: 
\begin{align}
    Reg{_{_\textup{TWL-UCB}}}(T) &= \sum_{m = 1}^M \E \bigg[\gamma_m\sum_{t = b_m}^{mK} \bigg\{\cM(S_t^*, \bC(t-1), c_t) - \cM(\pS_t, \bC(t-1), c_t)\bigg\}\bigg] \nonumber \\ 
    &= \sum_{m=1}^M \E\bigg[\Delta V_m\big(b_m, \bC(b_m-1)\big) \bigg], \label{eq:regret_revisited}
\end{align}
where $b_m \doteq (m-1)K + 1$, and $\Delta V_m\big(b_m, \bC(b_m-1)\big)$ is defined as:
\begin{align}
    \Delta V_m\big(b_m, \bC(b_m-1)\big) &\doteq \gamma_m\sum_{t = b_m}^{mK} \bigg\{\cM(S_t^*, \bC(t-1), c_t) - \cM(\pS_t, \bC(t-1), c_t)\bigg\}. \label{eq:deltaV}
\end{align}

{\color{black} Before diving into the detailed proof of Theorem \ref{thm_regret_upper_bound}, we provide an overview of the main steps below. To bound $Reg{_{_\textup{TWL-UCB}}}(T)$ in \eqref{eq:regret_revisited}, we compute $\E\bigg[\Delta V_m\big(b_m, \bC(b_m-1)\big) \bigg]$ under two mutually exclusive and collectively exhaustive probability regimes, $\cH_m$ and $\cH_m^{\complement}$, as: 
\[\E\bigg[\Delta V_m\big(b_m, \bC(b_m-1)\big) \bigg] = \E\bigg[\Delta V_m\big(b_m, \bC(b_m-1)\big)\I\big[\cH_m\big] \bigg] + \E\bigg[\Delta V_m\big(b_m, \bC(b_m-1)\big)\I\big[\cH_m^{\complement}\big] \bigg].\]

For the above analysis, we identify the regimes $\cH_m$ and $\cH_m^{\complement}$ and the corresponding probabilities in \eqref{eq:e_t} through \eqref{eq:theta_m_bound}. Then we segregate $\E\bigg[\Delta V_m\big(b_m, \bC(b_m-1)\big) \bigg]$ under $\cH_m$ and $\cH_m^{\complement}$, as shown in \eqref{eq:E_deltaV} through \eqref{eq:regret_epoch}. For $\cH_m$, we show that: 
\[\E\bigg[\Delta V_m\big(b_m, \bC(b_m-1)\big)\I\big[\cH_m\big] \bigg] \leq \sum_{t=b_m}^{mK}\gamma_m \dfrac{8}{m}.\] 

To upper bound $\E\bigg[\Delta V_m\big(b_m, \bC(b_m-1)\big)\I\big[\cH_m^{\complement}\big] \bigg]$, we use Lemma \ref{lemma_exp_deltaM_monotone} (in \S \ref{supplementary_lemmas}) to transform the problem of bounding $\Delta V_m\big(b_m, \bC(b_m-1)\big)\I\big[\cH_m^{\complement}\big]$  as explained in \eqref{eq:deltaM_monotone_eq} through \eqref{eq:deltaV_first}: 
\begin{align*}
    \Delta V_m\big(b_m, \bC(b_m-1)\big)\I\big[\cH_m^{\complement}\big] &= \gamma_m\sum_{t = b_m}^{mK} \big\{\cM(S_t^*, \bC(t-1), c_t) - \cM(\pS_t, \bC(t-1), c_t)\big\}\I\big[\cH_m^{\complement}\big]  \\
    &\leq \gamma_m\sum_{t = b_m}^{mK} \big\{\uM(\pS_t, \bC(t-1), c_t) - \cM(\pS_t, \bC(t-1), c_t)\big\}\I\big[\cH_m^{\complement}\big]. 
\end{align*}

We combine the bounds on $\E\bigg[\Delta V_m\big(b_m, \bC(b_m-1)\big)\I\big[\cH_m\big] \bigg]$ and $\E\bigg[\Delta V_m\big(b_m, \bC(b_m-1)\big)\I\big[\cH_m^{\complement}\big] \bigg]$ to obtain the bound on $\E\bigg[\Delta V_m\big(b_m, \bC(b_m-1)\big)\bigg]$ for each $m$ as shown in \eqref{eq:exp_regret_epoch}. Finally, we aggregate the bounds on $\E\bigg[\Delta V_m\big(b_m, \bC(b_m-1)\big)\bigg]$ for all $m \in \sM$, shown in \eqref{eq:regret_semifinal}, and simplify thereafter to obtain the bound on $Reg{_{_\textup{TWL-UCB}}}(T)$ in \eqref{eq:regret_revisited}. 

We now present the formal, detailed proof of Theorem \ref{thm_regret_upper_bound} below.
}

Define $\cE_t$ and $\cG_m$ as follows:
\begin{align}
    &\cE_t \doteq \bigcup_{i=1}^N \left\{v_{z, i, t}^{\textup{UCB}} < v_{z, i}\right\} \cup \left\{v_{z, i, t}^{\textup{UCB}} > \Ev + \beta_1\dfrac{\log(Nt)}{\omega_{i, z}(t)} + \beta_2\sqrt{\dfrac{\log(Nt)}{\omega_{i, z}(t)}\Ev}\right\}, \label{eq:e_t} \\
    &\cG_m \doteq \bigcup_{i=1}^N \left\{u_{i, z, m}^{\textup{UCB}} < u_{i, z}\right\} \cup \left\{u_{i, z, m}^{\textup{UCB}} > \Eu + \beta_3\dfrac{\log(Nm)}{\theta_{z, i}(m)} + \beta_4\sqrt{\dfrac{\log(Nm)}{\theta_{z, i}(m)}\Eu}\right\}. \label{eq:g_m}
\end{align}

From Lemma \ref{lemma_ucb}, we have:
\begin{equation}
    \begin{aligned}\label{eq:ucb_v}
        &\prob \bigg(v_{z, i, t}^{\textup{UCB}} - \Ev > \beta_1\dfrac{\log(Nt)}{\omega_{i, z}(t)} + \beta_2\sqrt{\dfrac{\log(Nt)}{\omega_{i, z}(t)}\Ev} \bigg) \leq \dfrac{2}{Nt},\\
        &\prob(v_{z, i, t}^{\textup{UCB}} < v_{z, i}) \leq \dfrac{2}{Nt},
    \end{aligned}
\end{equation}
\begin{equation}
    \begin{aligned}\label{eq:ucb_u}
        &\prob \bigg(u_{i, z, m}^{\textup{UCB}} - \Eu > \beta_3\dfrac{\log(Nm)}{\theta_{z, i}(m)} + \beta_4\sqrt{\dfrac{\log(Nm)}{\theta_{z, i}(m)}\Eu} \bigg) \leq \dfrac{2}{Nm},\\
        &\prob(u_{i, z, m}^{\textup{UCB}} < u_{i, z}) \leq \dfrac{2}{Nm}.
    \end{aligned}
\end{equation}

By union bound and \eqref{eq:ucb_v}, we get \(\prob(\cE_t) \leq 4/t\), and by union bound and \eqref{eq:ucb_u}, we get \(\prob(\cG_m) \leq 4/m\). Additionally, using \eqref{eq:ralpha_lower_bound} and \eqref{eq:ralpha1_lower_bound} from \S \ref{proof_lemma_ucb}, and by union bound, we get:
\begin{equation}
    \begin{aligned}\label{eq:theta_lower_bound_v2}
        &\prob\left(\bigcap_{i=1}^N \left\{ \omega_{i, z}(t) \leq \beta_1\log{(Nt)}\right\}\right) \geq 1-\dfrac{2}{t}, \\
        &\prob\left(\bigcap_{i=1}^N\left\{\theta_{z, i}(m) \leq \beta_3\log{(Nm)}\right\}\right) \geq 1-\dfrac{2}{m}.
    \end{aligned}
\end{equation}

Define $\cH_m \doteq \cE_t \cup \cG_m$. Then by union bound and $m \leq t$, $\prob(\cH_m) \leq 8/m$, which leads to $\prob\big(\cH_m^{\complement}\big) \geq 1 - 8/m$. Using \eqref{eq:theta_lower_bound_v2} and the definitions of $\cE_t, \cG_m$, and $\cH_m$, we know that $\bigcap_{i=1}^N\left\{\theta_{z, i}(m) \leq \beta_3\log{(Nm)}\right\} \supseteq \cH_m^{\complement}$. Therefore, we get:
\begin{align}
    \prob\left(\cH_m^{\complement}, \;\; \bigcap_{i=1}^N\left\{\theta_{z, i}(m) \leq \beta_3\log{(Nm)}\right\}\right) \geq 1 - \dfrac{8}{m}. \label{eq:theta_m_bound}
\end{align}

We now split $\E\bigg[\Delta V_m\big(b_m, \bC(b_m-1)\big)\bigg]$ from \eqref{eq:regret_revisited} into probability regimes of two complementary scenarios, $\cH_m$ and $\cH_m^{\complement}$, as shown below: 
\begin{align}
    &\E\bigg[\Delta V_m\big(b_m, \bC(b_m-1)\big)\bigg] \nonumber\\ 
    &= \E\bigg[\Delta V_m\big(b_m, \bC(b_m-1)\big) \I\big[\cH_m\big] + \Delta V_m\big(b_m, \bC(b_m-1)\big) \I\big[\cH_m^{\complement}\big] \bigg] \nonumber \\
    &= \E\bigg[\Delta V_m\big(b_m, \bC(b_m-1)\big) \I\big[\cH_m\big]\bigg] + \E\bigg[\Delta V_m\big(b_m, \bC(b_m-1)\big) \I\big[\cH_m^{\complement}\big] \bigg]. \label{eq:E_deltaV}
\end{align}

Note that $V_m\big(b_m, \bC(b_m-1)\big) = \E\left[\sum_{t=b_m}^{mK} \gamma_m\cM(S_t, \bC(t-1), c_t)\right] \leq \sum_{t=b_m}^{mK}\gamma_m$ since $\cM(S_t, \bC(t-1), c_t) \leq 1$ for all $t$. Therefore, $\Delta V_m\big(b_m, \bC(b_m-1)\big) \leq \sum_{t=b_m}^{mK}\gamma_m$ and therefore from \eqref{eq:E_deltaV} we get:
\begin{align}
    \E\bigg[\Delta V_m\big(b_m, \bC(b_m-1)\big)\bigg] &\leq \sum_{t=b_m}^{mK}\gamma_m\E\bigg[\I\big[\cH_m\big]\bigg] + \E\bigg[\Delta V_m\big(b_m, \bC(b_m-1)\big) \I\big[\cH_m^{\complement}\big] \bigg] \nonumber \\ 
    &\leq \sum_{t=b_m}^{mK}\gamma_m\dfrac{8}{m} + \E\bigg[\Delta V_m\big(b_m, \bC(b_m-1)\big) \I\big[\cH_m^{\complement}\big] \bigg]. \label{eq:regret_epoch}
\end{align}

When $\cH_m^{\complement}$ is true, then from Lemma \ref{lemma_exp_deltaM_monotone} (in \S \ref{supplementary_lemmas}), we have:
\begin{align}
    \uM(\pS_t, \bC(t-1), c_t) \geq \uM(S_t^*, \bC(t-1), c_t) \geq \cM(S_t^*, \bC(t-1), c_t). \label{eq:deltaM_monotone_eq}
\end{align}

Using \eqref{eq:deltaV}, \eqref{eq:deltaM_monotone_eq}, and Lemma \ref{lemma_optimal_value_bound}, we can therefore express $\Delta V_m\big(b_m, \bC(b_m-1)\big) \I\big[\cH_m^{\complement}\big]$ in \eqref{eq:regret_epoch} as follows:
\begin{align}
    &\Delta V_m\big(b_m, \bC(b_m-1)\big) \I\big[\cH_m^{\complement}\big] \nonumber \\
    &= \gamma_m\sum_{t = b_m}^{mK} \big\{\cM(S_t^*, \bC(t-1), c_t) - \cM(\pS_t, \bC(t-1), c_t)\big\}\I\big[\cH_m^{\complement}\big] \nonumber  \\
    &\leq \gamma_m\sum_{t = b_m}^{mK} \big\{\uM(\pS_t, \bC(t-1), c_t) - \cM(\pS_t, \bC(t-1), c_t)\big\}\I\big[\cH_m^{\complement}\big] \nonumber\\
    &\leq \gamma_m\sum_{t =b_m}^{t = mK}\sum_{i \in \pS_t}\bigg(q(\omega_{i, c_t}(t), v_{c_t, i}) + q(\theta_{c_t, i}(m), u_{i, c_t}) + q(\omega_{i, c_t}(t), v_{c_t, i})q(\theta_{c_t, i}(m), u_{i, c_t})\bigg) \nonumber\\
    &= \gamma_m\sum_{t =b_m}^{t = mK} \sum_{i \in \pS_t}\bigg[\beta_1\dfrac{\log(Nt)}{\omega_{i, c_t}(t)} + \beta_2\sqrt{\dfrac{\log(Nt)}{\omega_{i, c_t}(t)}v_{c_t, i}} + \beta_3\dfrac{\log(Nm)}{\theta_{c_t, i}(m)} + \beta_4\sqrt{\dfrac{\log(Nm)}{\theta_{c_t, i}(m)}u_{i, c_t}}\bigg] \nonumber \\
    &+ \gamma_m\sum_{t =b_m}^{t = mK} \sum_{i \in \pS_t}\bigg[\beta_1\dfrac{\log(Nt)}{\omega_{i, c_t}(t)} + \beta_2\sqrt{\dfrac{\log(Nt)}{\omega_{i, c_t}(t)}v_{c_t, i}}\bigg]\bigg[\beta_3\dfrac{\log(Nm)}{\theta_{c_t, i}(m)} + \beta_4\sqrt{\dfrac{\log(Nm)}{\theta_{c_t, i}(m)}u_{i, c_t}}\bigg] \nonumber \\
    &\leq \beta \gamma_m\sum_{t =b_m}^{t = mK} \sum_{i \in \pS_t}\Bigg[\dfrac{\log(Nt)}{\omega_{i, c_t}(t)} + \sqrt{\dfrac{\log(Nt)}{\omega_{i, c_t}(t)}v_{c_t, i}} + \dfrac{\log(Nt)}{\theta_{c_t, i}(m)} + \sqrt{\dfrac{\log(Nt)}{\theta_{c_t, i}(m)}u_{i, c_t}} \nonumber \\
    &+\dfrac{\log^2(Nt)}{\theta_{c_t, i}(m)\omega_{i, c_t}(t)} + \dfrac{\log^{1.5}(Nt)\sqrt{v_{c_t, i}}}{\theta_{c_t, i}(m)\sqrt{\omega_{i, c_t}(t)}} + \dfrac{\log^{1.5}(Nt)\sqrt{u_{i, c_t}}}{\sqrt{\theta_{c_t, i}(m)}\omega_{i, c_t}(t)} + \dfrac{\log(Nt)\sqrt{u_{i, c_t}v_{c_t, i}}}{\sqrt{\theta_{c_t, i}(m)\omega_{i, c_t}(t)}}\Bigg], \label{eq:deltaV_first}
\end{align}
where the last inequality follows from $m \leq t$, and \mbox{$\beta = \max\set{\beta_i, \prod_{i \neq j}\beta_i\beta_j \; \forall i, j \in \set{1, \ldots, 4}}$}. Using Assumption \ref{assumption:mnl}, and $\theta_{c_t, i}(m) \leq \omega_{i, c_t}(t))$, we can  simplify \eqref{eq:deltaV_first} to get the following:
\begin{align}
    &\Delta V_m\big(b_m, \bC(b_m-1)\big) \I\big[\cH_m^{\complement}\big] \nonumber \\
    &\leq \beta \gamma_m\sum_{t =b_m}^{t = mK} \sum_{i \in \pS_t}\bigg[\dfrac{\log^2(Nt)}{\theta_{c_t, i}^2(m)} + \dfrac{2\log^{1.5}(Nt)}{\theta_{c_t, i}^{1.5}(m)} + \dfrac{3\log(Nt)}{\theta_{c_t, i}(m)} + 2\sqrt{\dfrac{\log(Nt)}{\theta_{c_t, i}(m)}}\bigg] \nonumber\\
    &\leq \beta \gamma_m\sum_{t =b_m}^{t = mK} \sum_{i \in \pS_t}\bigg[\dfrac{\log^2(mKN)}{\theta_{c_t, i}^2(m)} + \dfrac{2\log^{1.5}(mKN)}{\theta_{c_t, i}^{1.5}(m)} + \dfrac{3\log(mKN)}{\theta_{c_t, i}(m)} + 2\sqrt{\dfrac{\log(mKN)}{\theta_{c_t, i}(m)}}\bigg], \label{eq:delta_v_complement}
\end{align}
where the last inequality follows from $t \leq mK$. Substituting $\Delta V_m\big(b_m, \bC(b_m-1)\big) \I\big[\cH_m^{\complement}\big]$ from \eqref{eq:delta_v_complement} in \eqref{eq:regret_epoch} and considering $\gamma_m \leq \bR$, we get: 
\begin{align}
    \E\bigg[\Delta V_m\big(b_m, \bC(b_m-1)\big)\bigg] &\leq \bR\sum_{t=b_m}^{mK}\dfrac{8}{m} \nonumber \\ 
    &+ \beta\bR \E\left[\log^2(mKN)\sum_{t =b_m}^{t = mK}\sum_{i \in \pS_t}\dfrac{1}{\theta_{c_t, i}^2(m)} + 2\log^{1.5}(mKN)\sum_{t =b_m}^{t = mK}\sum_{i \in \pS_t}\dfrac{1}{\theta_{c_t, i}^{1.5}(m)} \nonumber \right. \\ 
    &\left. + 3\log(mKN)\sum_{t =b_m}^{t = mK}\sum_{i \in \pS_t}\dfrac{1}{\theta_{c_t, i}(m)} + 2\sqrt{\log(mKN)}\sum_{t =b_m}^{t = mK}\sum_{i \in \pS_t}\sqrt{\dfrac{1}{\theta_{c_t, i}(m)}} \right]. \label{eq:exp_regret_epoch}
\end{align}

Substituting $\E\bigg[\Delta V_m\big(b_m, \bC(b_m-1)\big)\bigg]$ from \eqref{eq:exp_regret_epoch} in \eqref{eq:regret_revisited}, we get: 
\begin{align}
    Reg{_{_\textup{TWL-UCB}}}(T) &\leq 8K\bR\sum_{m = 1}^M\dfrac{1}{m} \nonumber \\ 
    &+ \beta\bR \sum_{m = 1}^M\E\left[\sum_{t =b_m}^{t = mK}\sum_{i \in \pS_t}\dfrac{\log^2(mKN)}{\theta_{c_t, i}^2(m)} + 2\sum_{t =b_m}^{t = mK}\sum_{i \in \pS_t}\dfrac{\log^{1.5}(mKN)}{\theta_{c_t, i}^{1.5}(m)} \nonumber \right. \\
    &\left. + 3\sum_{t =b_m}^{t = mK}\sum_{i \in \pS_t}\dfrac{\log(mKN)}{\theta_{c_t, i}(m)} + 2\sum_{t =b_m}^{t = mK}\sum_{i \in \pS_t}\sqrt{\dfrac{\log(mKN)}{\theta_{c_t, i}(m)}} \right]. \label{eq:regret_semifinal}
\end{align}

{\color{black} Under the MNL model, whenever seller $i$ is included in an offered assortment, the probability that customer type $c_t$ extends a proposal to seller $i$ is always at least:
\begin{align*}
p _{min} \doteq \min_{z \in \cC, j \in \cS} \frac{v_{z, j}}{1 + \sum_{h \in \pS_t} v_{z, h}} \ge \frac{\min_{z \in \cC, j \in \cS} v_{z, j}}{1 + B} > 0.
\end{align*}

To evaluate the sums involving $\theta_{c_t, i}(m)$ in \eqref{eq:regret_semifinal}, we must account for the fact that $\theta_{c_t, i}(m)$ only increments when customer $c_t$ extends a proposal. Suppose that $\theta_{c_t, i}(m) = k_i$. Given $k_i$, let $\Delta \omega_{i}^{(k_i)}$ denote the number of periods in which seller $i$ is included in the offered assortment before customer type $c_t$ extends the next proposal to seller $i$. 

Let $\mathcal{X}_{i, l}$ denote the event that seller $i$ does not receive a proposal in the $l$-th assortment inclusion. 
Because the MNL model guarantees that the probability of a proposal is always at least $p_{min}$, regardless of previous outcomes, we have:
\begin{align}
\mathbb{P}\big(\mathcal{X}_{i, l} \mid \mathcal{X}_{i, 1}, \ldots, \mathcal{X}_{i, l-1}\big) 
= 1 - \mathbb{P}\big(\mathcal{X}_{i, l}^{\complement} \mid \mathcal{X}_{i,1}, \ldots, \mathcal{X}_{i,l-1}\big) 
\le 1 - p_{min}. \nonumber
\end{align}

For any integer $L \geq 1$, we can write:
\begin{align*}
    \mathbb{P}\big(\Delta \omega_{i}^{(k_i)} > L\big) &\doteq \mathbb{P}\big(\text{no proposal to $i$ in next $L$ assortment inclusions}\big) \nonumber\\
    &=\mathbb{P}\big(\mathcal{X}_{i, 1}, \ldots, \mathcal{X}_{i, L}\big)\\
    &= \prod_{l=1}^L \mathbb{P}\big(\mathcal{X}_{i, l} \mid \mathcal{X}_{i, 1} \cap \cdots \cap \mathcal{X}_{i, l-1}\big) \leq \prod_{l=1}^L (1-p_{min}) = (1-p_{min})^L \nonumber.
\end{align*}

Thus, for any $i \in \cS$ and $c_t \in \cC$, the expected number of assortment inclusions made before the next proposal is received is bounded by:
\begin{align}
    \E\left[\Delta \omega_{i}^{(k_i)}\right] &= \sum_{L=0}^{\infty} \mathbb{P}\big(\Delta \omega_{i}^{(k_i)} > L\big) \leq \sum_{L=0}^{\infty} (1-p_{min})^L = \dfrac{1}{p_{min}}. \label{eq:proposal_timeline}
\end{align}

Using \eqref{eq:theta_m_bound}, \eqref{eq:proposal_timeline}, $\sum_{1 \leq x < \infty} \dfrac{1}{x^p} < \infty \;\forall p > 1$, and $\sum_{m = 1}^M \dfrac{1}{m} \leq 1 + \log(M) \leq 1 + \log(T)$, we can therefore bound $Reg{_{_\textup{TWL-UCB}}}(T)$ in \eqref{eq:regret_semifinal} as:
\begin{align}
    &Reg{_{_\textup{TWL-UCB}}}(T) \nonumber\\ 
    &\leq 8K\bR[1 + \log(T)] \nonumber \\ 
    &+ \beta \bR\left\{\E\left[\sum_{m = 1}^M\sum_{t =b_m}^{t = mK}\sum_{i \in \pS_t}\dfrac{\log^2(NT)}{\theta_{c_t, i}^2(m)}\right] + 2\E\left[\sum_{m = 1}^M\sum_{t =b_m}^{t = mK}\sum_{i \in \pS_t}\dfrac{\log^{1.5}(NT)}{\theta_{c_t, i}^{1.5}(m)}\right] \nonumber \right.  \\
    &\left. + 3\E\left[\sum_{m = 1}^M\sum_{t =b_m}^{t = mK}\sum_{i \in \pS_t}\dfrac{\log(NT)}{\theta_{c_t, i}(m)}\right] + 2\E\left[\sum_{m = 1}^M\sum_{t =b_m}^{t = mK}\sum_{i \in \pS_t}\sqrt{\dfrac{\log(NT)}{\theta_{c_t, i}(m)}}\right]\right\} \nonumber \tag*{(using \mbox{$mK \leq MK = T$})} \\
    &\leq \alpha_1\log(T) \nonumber \\ 
    &+ \beta \bR \left\{\E\left[\sum_{i = 1}^N\sum_{m = 1}^M\sum_{t =b_m}^{t = mK}\dfrac{\log^2(NT)}{\theta_{c_t, i}^2(m)}\right] + 2\E\left[\sum_{i = 1}^N\sum_{m = 1}^M\sum_{t =b_m}^{t = mK}\dfrac{\log^{1.5}(NT)}{\theta_{c_t, i}^{1.5}(m)}\right] \nonumber \right. \\ 
    &\left. + 3\E\left[\sum_{i = 1}^N\sum_{m = 1}^M\sum_{t =b_m}^{t = mK}\dfrac{\log(NT)}{\theta_{c_t, i}(m)}\right] + 2\E\left[\sum_{i = 1}^N\sum_{m = 1}^M\sum_{t =b_m}^{t = mK}\sqrt{\dfrac{\log(NT)}{\theta_{c_t, i}(m)}}\right]\right\} \nonumber \tag*{(using \mbox{$|\pS_t| \leq B \leq N$} and \mbox{$\alpha_1 \in \R_+$})}  \\
    &\leq \alpha_1\log(T) \nonumber \\ 
    &+ \beta\bR \sum_{i = 1}^N \sum_{z \in \cC} \left\{\E\left[\sum_{k_{i} = 1}^{\beta_3\log(NT)}\Delta \omega_{i}^{(k_i)}\dfrac{\log^2(NT)}{k_{i}^2}\right] + 2\E\left[\sum_{k_{i} = 1}^{\beta_3\log(NT)}\Delta \omega_{i}^{(k_i)}\dfrac{\log^{1.5}(NT)}{k_{i}^{1.5}}\right] \nonumber \right. \\ 
    &\left. + 3\E\left[\sum_{k_{i} = 1}^{\beta_3\log(NT)}\Delta \omega_{i}^{(k_i)}\dfrac{\log(NT)}{k_{i}}\right] + 2\E\left[\sum_{k_{i} = 1}^{\beta_3\log(NT)}\Delta \omega_{i}^{(k_i)}\sqrt{\dfrac{\log(NT)}{k_{i}}}\right]\right\}.  \label{eq:transition_equation}
\end{align}  

{\color{black} In \eqref{eq:transition_equation}, we transform the summation involving $\theta_{c_t, i}(m)$ over $t$ into a summation over its discrete values $k_i$ for all $i \in \cS$; the rationale for this transformation is explained further. The inner sum over $t$ evaluates the terms at the current value of $\theta_{c_t, i}(m)$. However, $\theta_{c_t, i}(m)$ may not increment at every $t$; it remains fixed at some value $k_i$ for exactly $\Delta \omega_{i}^{(k_i)}$ periods, which is the number of times seller $i$ is offered to customer type $c_t$ before receiving the next proposal. Consequently, instead of adding the terms sequentially over time $t$, we can group them by $k_i$. Because the exact sequence of arriving customer types $c_t$ is unknown, we can safely upper-bound the regret by evaluating this grouped sum over all possible customer types $z \in \cC$. For any given customer type $z \in \cC$, the term $1/k_i^p$, where $p \geq 1,$ appears in this bound exactly $\Delta \omega_{i}^{(k_i)}$ times. Furthermore, under the high-probability event established in \eqref{eq:theta_m_bound}, the maximum value of $\theta_{c_t, i}(m)$ is bounded above by $\beta_3\log(NT)$. 

Finally, we simplify \eqref{eq:transition_equation} as follows:}
\begin{align*}
    &Reg{_{_\textup{TWL-UCB}}}(T) \nonumber\\
    &\leq \alpha_1\log(T) \nonumber \\ 
    &+ \beta\bR \sum_{i = 1}^N \sum_{z \in \cC} \left\{\sum_{k_{i} = 1}^{\beta_3\log(NT)}\dfrac{\log^2(NT)}{k_{i}^2}\E\left[\Delta \omega_{i}^{(k_i)}\right] + 2\sum_{k_{i} = 1}^{\beta_3\log(NT)}\dfrac{\log^{1.5}(NT)}{k_{i}^{1.5}}\E\left[\Delta \omega_{i}^{(k_i)}\right] \nonumber \right. \\ 
    &\left. + 3\sum_{k_{i} = 1}^{\beta_3\log(NT)}\dfrac{\log(NT)}{k_{i}}\E\left[\Delta \omega_{i}^{(k_i)}\right] + 2\sum_{k_{i} = 1}^{\beta_3\log(NT)}\sqrt{\dfrac{\log(NT)}{k_{i}}}\E\left[\Delta \omega_{i}^{(k_i)}\right]\right\}  \nonumber \\
    &\leq \alpha_1\log(T) \nonumber \\ 
    &+ \dfrac{\beta\bR }{p_{min}} \sum_{i = 1}^N \sum_{z \in \cC} \left\{\sum_{k_{i} = 1}^{\beta_3\log(NT)}\dfrac{\log^2(NT)}{k_{i}^2} + 2\sum_{k_{i} = 1}^{\beta_3\log(NT)}\dfrac{\log^{1.5}(NT)}{k_{i}^{1.5}} \nonumber \right. \\ 
    &\left. + 3\sum_{k_{i} = 1}^{\beta_3\log(NT)}\dfrac{\log(NT)}{k_{i}} + 2\sum_{k_{i} = 1}^{\beta_3\log(NT)}\sqrt{\dfrac{\log(NT)}{k_{i}}}\right\}  \nonumber \\
    &\stackrel{(a)}{\leq} \alpha_1\log(T) \nonumber\\
    &+ \alpha_2 \beta \bR |\cC|N \bigg\{\log^2(NT) + \log^{1.5}(NT) + \log(NT)\log\big(\log(NT)\big) + \sqrt{\log(NT)}\sqrt{\log(NT)}\bigg\} \nonumber \\
    &\leq \alpha\log^2(NT), \; \text{where $\alpha \in \R_+$, and inequality $(a)$ follows from the fact that \(\sum_{x = 1}^{n} \dfrac{1}{\sqrt{x}} < 2\sqrt{n}\).} \nonumber
\end{align*}

}

\hfill $\square$ 

\section{Proofs of Lemmas and Theorem in the Main Text: Lower bound}
\label{ec_thm_2}

\subsection{Proof of Lemma \ref{lemma_inst_regret_suboptimal}}
\label{proof_lemma_inst_regret_suboptimal}
 
Suppose that the arriving customer in period $t$ in epoch $m$ is of type $c_{t} = z \in \cC$. Under $\Theta_{\overline{S}}$, $v_{z, i} = v(1+\ep_z)$ for $i \in {\overline{S}}$ and $v_{z, i} = v$ for $i \notin {\overline{S}}$ (see \eqref{eq:param_space} for the definition of $\Theta_{\overline{S}}$). Therefore, under $\Theta_{\overline{S}}$, ${\overline{S}}$ is the ``single best assortment" by Lemma \ref{lemma_optimal_assortment_size} (in \S \ref{supplementary_lemmas}). The one-step regret under $\Theta_{\overline{S}}$ is then given by:
\begin{align*}
    \gamma_m \left(\cM({\overline{S}}) - \cM(S_t)\right) &= \dfrac{(1+u)\sum_{i \in {\overline{S}}}v_{z, i}r_{i, z}(t)}{1 + \sum_{i \in {\overline{S}}}v_{z, i}} - \dfrac{(1+u)\sum_{i \in S_{t}}v_{z, i}r_{i, z}(t)}{1 + \sum_{i \in S_t}v_{z, i}} \\
    &= u \bigg[\dfrac{Bv(1+\ep_z)}{1 + Bv(1+\ep_z)} - \dfrac{|S_t \cap {\overline{S}}|v(1+\ep_z) + v(B-|S_t \cap {\overline{S}}|)}{1 + |S_t \cap {\overline{S}}|v(1+\ep_z) + v(B-|S_t \cap {\overline{S}}|)} \bigg] \tag*{(using \eqref{eq:reward_parameter_space})} \\
    &= uv \bigg[\dfrac{B(1+\ep_z)}{1 + Bv(1+\ep_z)} - \dfrac{B + \ep_z|S_t \cap {\overline{S}}|}{1 + v(B + \ep_z|S_t \cap {\overline{S}}|)} \bigg] \\
    &= uv \bigg[\dfrac{\ep_z(B - |S_t \cap {\overline{S}}|)}{(1 + Bv(1+\ep_z))(1 + v(B + \ep_z|S_t \cap {\overline{S}}|))}\bigg] \\
    &\geq \dfrac{uv\ep_z(B - |S_t \cap {\overline{S}}|)}{(1 + B)^2}. \tag*{(using \mbox{$0 < v \leq 0.5, 0 < \ep_z \leq 1$} , and \mbox{$|S_t \cap {\overline{S}}| \leq B$})}
\end{align*}

\hfill $\square$ 

\subsection{Proof of Lemma \ref{lemma_pinskers_ineq}}
\label{proof_lemma_pinskers_ineq}

Using $\eta_i(m) \doteq \sum_{t = b_m}^{mK}\I\big[i \in S_t\big]$ for $S_t \in \sS_B$, we have:
\begin{align*}
    &\Bigg|\E_{\tS}\bigg[\sum_{t = b_m}^{mK}\I\big[i \in S_t\big]\bigg] - \E_{\overline{S}}\bigg[\sum_{t = b_m}^{mK}\I\big[i \in S_t\big]\bigg]\Bigg|\\ 
    &= \bigg|\E_{\tS}\big[\eta_i(m)\big] - \E_{\overline{S}}\big[\eta_i(m)\big]\bigg| \\ 
    &= \Bigg|\sum_{x = 0}^K \bigg[x\prob_{\tS}\bigg(\eta_i(m) = x\bigg) - x\prob_{\overline{S}}\bigg(\eta_i(m) = x\bigg)\bigg]\Bigg|\\
    &\leq \sum_{x = 0}^K x\Bigg|\prob_{\tS}\bigg(\eta_i(m) = x\bigg) - \prob_{\overline{S}}\bigg(\eta_i(m) = x\bigg)\Bigg| \tag*{(using triangle inequality)} \\
    &\leq K\sum_{x = 0}^K \Bigg|\prob_{\tS}\bigg(\eta_i(m) = x\bigg) - \prob_{\overline{S}}\bigg(\eta_i(m) = x\bigg)\Bigg|\\
    &= K||\prob_{\tS} - \prob_{\overline{S}}||_1 \leq  K\sqrt{2\textup{KL}_m\big(\prob_{\tS}\parallel \prob_{\overline{S}}\big)} = K\sqrt{2\textup{KL}_m\big(\prob_{\tS} \parallel \prob_{\tS \cup \set{i}}\big)}.  \tag*{(by Pinsker's Inequality)}
\end{align*}

\hfill $\square$ 

\subsection{Proof of Lemma \ref{lemma_kl_divergence_neighbor}}
\label{proof_lemma_kl_divergence_neighbor}

Let $\hS_t \in \sS$ be the assortment offered at time $t$ to an incoming customer of type $c_t$, where $|\hS_t| \leq B$. If $i \notin \hS_t$, then $\textup{KL}_m\big(\prob_{\tS}(\cdot \mid \hS_t)\parallel \prob_{\overline{S}}(\cdot \mid \hS_t)\big) = 0$ at time $t$ since $\prob_{\tS}(\cdot \mid \hS_t)$ and $\prob_{\overline{S}}(\cdot \mid \hS_t)$ have the same distributions. Therefore, we focus on the case when $i \in \hS_t$. 

Let $|\hS_t| = B_{\pi} \leq B$. For any $\tS \in \sS_{B-1}^{-i}$, we have $|\hS_t \cap \tS| \leq B-1$. Let $|\hS_t \cap \tS| = J$. Then, for ${\overline{S}} \in \sS_B$, we get $|\hS_t \cap {\overline{S}}| = J + 1 \leq B$. Let $v_{c_t, k}^{\overline{S}}$ and $v_{c_t, k}^{\tS}$ represent the MNL parameters of customer of type $c_t \in \cC$ for seller $k \in \cS$ under $\Theta_{\overline{S}}$ and $\Theta_{\tS}$, respectively. Then the probability that customer of type $c_t$ chooses $j \in \hS_t \cup \set{0}$, where $``\set{0}"$ represents an outside option or no choice on the platform, is $p_{c_t, j}^{\overline{S}} = v_{c_t, j}^{\overline{S}}\big/(1+\sum_{k \in \hS_t}v_{c_t, k}^{\overline{S}})$ and $p_{c_t, j}^{\tS} = v_{c_t, j}^{\tS}\big/(1+\sum_{k \in \hS_t}v_{c_t, k}^{\tS})$ under $\prob_{\overline{S}}$ and $\prob_{\tS}$, respectively. Under $\prob_{\overline{S}}$ and $\prob_{\tS}$, $p_{c_t, j}^{\overline{S}}$ and $p_{c_t, j}^{\tS}$ are, respectively, given by:
\begin{align}
    p_{c_t, j}^{\overline{S}} &= \dfrac{v_{c_t, j}^{\overline{S}}}{1+\sum_{k \in \hS_t \cap {\overline{S}}}v_{c_t, k}^{\overline{S}} + \sum_{k \notin \hS_t \cap {\overline{S}}}v_{c_t, k}^{\overline{S}}} \nonumber  \\ 
    &= \dfrac{v_{c_t, j}^{\overline{S}}}{1+|\hS_t\cap {\overline{S}}|v(1+\ep_{c_t}) + v(|\hS_t| - |\hS_t\cap {\overline{S}}|)} = \dfrac{v_{c_t, j}^{\overline{S}}}{1+v(B_{\pi} + \ep_{c_t}(J+1))}. \label{eq:prob_S} \\
    p_{c_t, j}^{\tS} &= \dfrac{v_{c_t, j}^{\tS}}{1+\sum_{k \in \hS_t \cap \tS}v_{c_t, k}^{\tS} + \sum_{k \notin \hS_t \cap \tS}v_{c_t, k}^{\tS}} \nonumber \\ 
    &= \dfrac{v_{c_t, j}^{\tS}}{1+|\hS_t\cap \tS|v(1+\ep_{c_t}) + v(|\hS_t| - |\hS_t\cap \tS|)} = \dfrac{v_{c_t, j}^{\tS}}{1+v(B_{\pi} + \ep_{c_t}J)}. \label{eq:prob_Stilde}
\end{align}

Case 1: $j = 0$.
\begin{align}
    |p_{c_t, 0}^{\tS} - p_{c_t, 0}^{\overline{S}}| &= \bigg|\dfrac{1}{1+v(B_{\pi} + \ep_{c_t}J)} - \dfrac{1}{1+v(B_{\pi} + \ep_{c_t}(J+1))}\bigg| \nonumber  \\ 
    &= \bigg|\dfrac{v\ep_{c_t}}{(1+v(B_{\pi} + \ep_{c_t}J))(1+v(B_{\pi} + \ep_{c_t}(J+1)))}\bigg| \leq v\ep_{c_t}.  \label{eq:case1}
\end{align}

Case 2: $j = i$, where $i \notin \tS$ by design.
\begin{align}
    |p_{c_t, j}^{\tS} - p_{c_t, j}^{\overline{S}}| &= \bigg|\dfrac{v}{1+v(B_{\pi} + \ep_{c_t}J)} - \dfrac{v(1+\ep_{c_t})}{1+v(B_{\pi} + \ep_{c_t}(J+1))}\bigg| \nonumber  \\
    &= v\bigg|\dfrac{1}{1+v(B_{\pi} + \ep_{c_t}J)} - \dfrac{1}{1+v(B_{\pi} + \ep_{c_t}(J+1))} + \dfrac{-\ep_{c_t}}{1+v(B_{\pi} + \ep_{c_t}(J+1))}\bigg| \nonumber  \\
    &\leq v\bigg\{\bigg|\dfrac{1}{1+v(B_{\pi} + \ep_{c_t}J)} - \dfrac{1}{1+v(B_{\pi} + \ep_{c_t}(J+1))}\bigg| + \bigg|\dfrac{-\ep_{c_t}}{1+v(B_{\pi} + \ep_{c_t}(J+1))}\bigg|\bigg\} \nonumber \\ 
    &\text{(by triangle inequality)} \nonumber \\
    &= v(|p_{c_t, 0}^{\tS} - p_{c_t, 0}^{\overline{S}}| + \ep_{c_t}) \leq v(v\ep_{c_t} + \ep_{c_t}) = v^2\ep_{c_t} + v\ep_{c_t} \leq 2v\ep_{c_t}. \label{eq:case2} \\
    &\text{(using \eqref{eq:case1} and \mbox{$v \in (0, 0.5]$})} \nonumber
\end{align}

Case 3: $j \neq \set{i, 0}$, where $i \notin \tS$ by design.
\begin{align}
    |p_{c_t, j}^{\tS} - p_{c_t, j}^{\overline{S}}| &= \bigg|\dfrac{v(1+\ep_{c_t})}{1+v(B_{\pi} + \ep_{c_t}J)} - \dfrac{v(1+\ep_{c_t})}{1+v(B_{\pi} + \ep_{c_t}(J+1))}\bigg| \nonumber  \\ 
    &= v(1+\ep_{c_t})|p_{c_t, 0}^{\tS} - p_{c_t, 0}^{\overline{S}}| \leq v^2(1+\ep_{c_t})\ep_{c_t} \nonumber\\ 
    &= v^2\ep_{c_t} + v^2\ep^2_{c_t} \leq 2v^2\ep_{c_t}. \label{eq:case3}\\
    &\text{(using \eqref{eq:case1} and \mbox{$\ep_{c_t} \in (0, 1]$})} \nonumber
\end{align}

Since $B_{\pi} \leq B, J+1 \leq B$, $0 < v \leq 0.5$, and $0 < \ep_{c_t} \leq 1$, we have: 
\begin{equation}
    \begin{aligned}\label{eq:prob_bound}
        &p_{c_t, 0}^{\overline{S}} = \dfrac{1}{1 + v(B_{\pi} + \ep_{c_t}(J+1))} \geq \dfrac{1}{B+1}, \\
        &p_{c_t, j}^{\overline{S}} = \dfrac{v_{c_t, j}^{\overline{S}}}{1 + v(B_{\pi} + \ep_{c_t}(J+1))} \geq \dfrac{v(1+\ep_{c_t})}{B+1} \geq \dfrac{v}{B+1}, \;\; \text{where $j \neq 0$}.
    \end{aligned}
\end{equation}

From Lemma \ref{lemma_chen_2018} (in \S \ref{supplementary_lemmas}), we have:
\begin{align}
    \textup{KL}_m\big(\prob_{\tS}(\cdot \mid \hS_t)\parallel \prob_{\overline{S}}(\cdot \mid \hS_t)\big) &\leq \sum_{j} \dfrac{(p_{c_t, j}^{\tS} - p_{c_t, j}^{\overline{S}})^2}{p_{c_t, j}^{\overline{S}}} \nonumber \\
    &= \dfrac{(p_{c_t, 0}^{\tS} - p_{c_t, 0}^{\overline{S}})^2}{p_{c_t, 0}^{\overline{S}}} + \dfrac{(p_{c_t, i}^{\tS} - p_{c_t, i}^{\overline{S}})^2}{p_{c_t, i}^{\overline{S}}} + \sum_{j \neq i}\dfrac{(p_{c_t, j}^{\tS} - p_{c_t, j}^{\overline{S}})^2}{p_{c_t, j}^{\overline{S}}} \nonumber \\
    &\leq \dfrac{v^2\ep_{c_t}^2}{1/(B+1)} + \dfrac{4v^2\ep_{c_t}^2}{v/(B+1)} + \sum_{j \neq i}\dfrac{4v^4\ep_{c_t}^2}{v/(B+1)}  \tag*{(using \eqref{eq:case1} through \eqref{eq:prob_bound})} \nonumber \\
    &\leq (B+1)v^2\ep_{c_t}^2 + 4(B+1)v\ep_{c_t}^2 + 4B(B+1)v^3\ep_{c_t}^2 \nonumber \\ 
    &\leq \dfrac{(B+1)\ep_{c_t}^2}{4} + 2(B+1)\ep_{c_t}^2 + \dfrac{B(B+1)}{2}\ep_{c_t}^2 \;\; \text{(using $0 < v \leq 0.5$)} \nonumber \\
    &\leq (B+1)\ep_{c_t}^2 (\dfrac{1}{4} + 2 + \dfrac{B}{2}) \leq \dfrac{(B+1)(2B + 9)}{4}\max_{z \in \cC}\ep_{z}^2.  \label{eq:kl_div}
\end{align}

Using divergence decomposition \citep{lattimore2020bandit} (see, for e.g., \cite{chen2018note, miao2022online} for related applications) and \eqref{eq:kl_div}, KL divergence measured in epoch $m$ is therefore given by:
\begin{align*}
    \textup{KL}_m\big(\prob_{\tS} \parallel \prob_{\overline{S}}\big) = \textup{KL}_m\big(\prob_{\tS} \parallel \prob_{\tS \cup \set{i}}\big) 
    &\leq \dfrac{(B+1)(2B + 9)}{4}\max_{z \in \cC}\ep_{z}^2.
\end{align*}

\hfill $\square$ 

\subsection{Proof of Theorem \ref{thm_regret_lower_bound}}
\label{proof_thm_regret_lower_bound}

{\color{black}
Let $\min\set{\ep_z}_{z \in \cC} = \ep$. Recall from \S \ref{lb_proof_outline} that $S_t \in \sS_B$ represents a full-capacity assortment such that $S_t \supseteq \hS_t$, where $\hS_t$ is the actual assortment offered by policy $\pi$. By the monotonicity of the $\cM(S)$ function with respect to $|S|$, as established in \eqref{eq:expected_reward_S} in Lemma \ref{lemma_optimal_assortment_size} (in \S \ref{supplementary_lemmas}), the sequence $\set{\hS_t}_{t\leq T}$ suffers higher regret compared to $\set{S_t}_{t \leq T}$. Moreover, $\gamma_m = (1 + u) \in (1, 2]$ for $u \in (0, 1]$ (using \eqref{eq:param_space} and \eqref{eq:gamma_epoch}). 

The worst-case regret over all possible parameterizations is lower-bounded by the worst-case regret over our constructed family of adversarial environments $\set{\Theta_{\overline{S}} \mid \overline{S} \in \sS_B}$. Furthermore, the maximum regret over this family is lower-bounded by the average regret. Therefore, we can lower-bound the worst-case regret as:}
\begin{align}
    \sup_{(\bU, \bV)} Reg_{\pi, (\bU, \bV)}(T) &\geq \max_{\overline{S} \in \sS_B} Reg_{\pi, \Theta_{\overline{S}}}(T) \nonumber\\ 
    &= \max_{\overline{S} \in \sS_B} \sum_{m=1}^M\gamma_m\E_{\overline{S}}\left[ \sum_{t=b_m}^{mK}\bigg\{\cM(\overline{S}) - \cM(\hS_t)\bigg\} \right] \nonumber\\
    &\geq \max_{\overline{S} \in \sS_B} \sum_{m=1}^M\gamma_m\E_{\overline{S}}\left[ \sum_{t=b_m}^{mK}\bigg\{\cM(\overline{S}) - \cM(S_t)\bigg\} \right] \nonumber\\
    &\geq \dfrac{1}{|\sS_B|}\sum_{\overline{S} \in \sS_B} \sum_{m=1}^M \gamma_m \E_{\overline{S}}\left[ \sum_{t=b_m}^{mK} \bigg\{\cM(\overline{S}) - \cM(S_t)\bigg\} \right] \nonumber\\
    &\geq \dfrac{1}{|\sS_B|}\sum_{\overline{S} \in \sS_B} \sum_{m = 1}^M  \E_{\overline{S}}\left[\sum_{t = b_m}^{mK}\dfrac{uv\ep_{c_t}(B - |S_t \cap \overline{S}|)}{(1 + B)^2}\right] \tag*{(from Lemma \ref{lemma_inst_regret_suboptimal})} \nonumber\\
    &\geq \dfrac{uv\ep}{(B+1)^2}\sum_{m = 1}^M \dfrac{1}{|\sS_B|}\sum_{\overline{S} \in \sS_B} \E_{\overline{S}}\left[ \sum_{t = b_m}^{mK} (B - |S_t \cap {\overline{S}}|)\right] \tag*{(using \mbox{$\ep_{c_t} \geq \ep$} for all \mbox{$c_t$})} \nonumber\\
    &= \dfrac{uv\ep}{(B+1)^2}\sum_{m = 1}^M \dfrac{1}{|\sS_B|}\sum_{{\overline{S}} \in \sS_B} \sum_{t = b_m}^{mK} \E_{\overline{S}}\bigg[B - |S_t \cap {\overline{S}}|\bigg] \nonumber\\
    &= \dfrac{uv\ep}{(B+1)^2}\sum_{m = 1}^M \dfrac{1}{|\sS_B|}\sum_{{\overline{S}} \in \sS_B} \sum_{t = b_m}^{mK} \bigg\{B - \sum_{i\in {\overline{S}}}\E_{\overline{S}}\bigg[\I\big[i \in S_t\big]\bigg]\bigg\} \nonumber\\
    &= \dfrac{uv\ep}{(B+1)^2}\sum_{m = 1}^M \left\{\sum_{t = b_m}^{mK}B - \dfrac{1}{|\sS_B|}\sum_{{\overline{S}} \in \sS_B} \sum_{i \in {\overline{S}}}\E_{\overline{S}}\left[\sum_{t = b_m}^{mK}\I\big[i \in S_t\big]\right]\right\} \nonumber\\
    &= \dfrac{uv\ep}{(B+1)^2}\sum_{m = 1}^M \bigg\{KB - \dfrac{1}{|\sS_B|}\sum_{{\overline{S}} \in \sS_B} \sum_{i\in {\overline{S}}}\E_{\overline{S}}\big[\eta_i(m)\big]\bigg\}, \label{eq:reg_lb_first_exp}
\end{align}
where $\eta_i(m) \doteq \sum_{t = b_m}^{mK}\I\big[i \in S_t\big]$. From Lemma \ref{lemma_pinskers_ineq}, we get:
\begin{align}
    \E_{\overline{S}}\big[\eta_i(m)\big] \leq \E_{\tS}\big[\eta_i(m)\big] + K\sqrt{2\textup{KL}(\prob_{\tS}\parallel \prob_{\tS \cup \set{i}})}. \label{eq:lemma5_inference} 
\end{align}

Using \eqref{eq:lemma5_inference}, we can rewrite \(\dfrac{1}{|\sS_B|}\sum_{{\overline{S}} \in \sS_B} \sum_{i\in {\overline{S}}}\E_{\overline{S}}\big[\eta_i(m)\big]\) from \eqref{eq:reg_lb_first_exp} as:
\begin{align}
    &\dfrac{1}{|\sS_B|}\sum_{{\overline{S}} \in \sS_B}\sum_{i \in {\overline{S}}}\E_{\overline{S}}\big[\eta_i(m)\big] \nonumber \\ 
    &= \dfrac{1}{|\sS_B|}\sum_{i=1}^N\sum_{{\overline{S}} \in \sS_B, i \in {\overline{S}}} \E_{\overline{S}}\big[\eta_i(m)\big] \nonumber \\ 
    &= \dfrac{1}{|\sS_B|}\sum_{i=1}^N\sum_{\tS \in \sS_{B-1}^{-i}}\E_{\tS \cup \set{i}}\big[\eta_i(m)\big] \nonumber\\
    &\leq \dfrac{1}{|\sS_B|}\sum_{i=1}^N\sum_{\tS \in \sS_{B-1}^{-i}}\bigg[\E_{\tS}\big[\eta_i(m)\big] +  K\sqrt{2\textup{KL}(\prob_{\tS}\parallel \prob_{\tS \cup \set{i}})} \bigg] \nonumber\\
    &= \dfrac{1}{|\sS_B|}\sum_{\tS \in \sS_{B-1}}\sum_{i \notin \tS}\E_{\tS}\big[\eta_i(m)\big] +  \dfrac{K}{|\sS_B|}\sum_{\tS \in \sS_{B-1}}\sum_{i = 1}^N\sqrt{2\textup{KL}(\prob_{\tS}\parallel \prob_{\tS \cup \set{i}})} \nonumber\\
    &\leq \dfrac{1}{|\sS_B|}\sum_{\tS \in \sS_{B-1}}\sum_{i = 1}^N\E_{\tS}\big[\eta_i(m)\big] +  \dfrac{K}{|\sS_B|}\sum_{\tS \in \sS_{B-1}}\sum_{i=1}^N\sqrt{2\textup{KL}(\prob_{\tS}\parallel \prob_{\tS \cup \set{i}})} \nonumber\\
    &= \dfrac{|\sS_{B-1}|}{|\sS_B|}\E_{\tS}\bigg[KB\bigg] +  \dfrac{K}{|\sS_B|}\sum_{i=1}^N\sum_{\tS \in \sS_{B-1}}\sqrt{2\textup{KL}(\prob_{\tS}\parallel \prob_{\tS \cup \set{i}})} \tag*{(using \mbox{$\sum_{t=b_m}^{mK}\sum_{i=1}^N\I\big[i \in S_t\big] = KB$})} \nonumber\\
    &= \dfrac{KB^2}{N-B+1} + \dfrac{K}{|\sS_B|}\sum_{i=1}^N\sum_{\tS \in \sS_{B-1}^{-i}}\sqrt{2\textup{KL}(\prob_{\tS}\parallel \prob_{\tS \cup \set{i}})}. \label{eq:reg_lb_second_exp}
\end{align}

Using Lemma \ref{lemma_kl_divergence_neighbor}, \eqref{eq:reg_lb_first_exp}, and \eqref{eq:reg_lb_second_exp}, we get:
\begin{align}
    &\max_{{\overline{S}}\in \sS_B} Reg_{\pi, \Theta_{\overline{S}}}(T) \nonumber\\
    &\geq \dfrac{uv\ep}{(B+1)^2}\sum_{m=1}^M\left[KB - \dfrac{KB^2}{N - B+1} - \dfrac{K}{|\sS_B|}\sum_{i=1}^N\sum_{\tS \in \sS_{B-1}}\sqrt{2\textup{KL}(\prob_{\tS}\parallel \prob_{\tS \cup \set{i}})} \right] \nonumber\\
    &\geq \dfrac{uv\ep}{(B+1)^2}\sum_{m=1}^M\left[\dfrac{KB(N - 2B + 1)}{N - B + 1} - \dfrac{NK|\sS_{B-1}|}{|\sS_B|}\sqrt{2(B+1)\big(\dfrac{2B + 9}{4}\big)\max_{z \in \cC}\ep_{z}^2} \right] \nonumber\\
    &\geq \dfrac{uv\ep}{(B+1)^2}\sum_{m=1}^M\left[\dfrac{B(N - 2B + 1)}{N - B + 1} - \dfrac{NB}{2(N-B+1)}\sqrt{\dfrac{(B+1)(2B + 9)}{2}}\max_{z \in \cC}\ep_{z}\right] \nonumber  \\
    &\text{(using \mbox{$\ep_{z} \leq 1 \; \forall z \in \cC, K = 1$})} \nonumber\\
    &= \dfrac{uv\ep B(N - 2B + 1)}{(B+1)^2(N-B+1)}\sum_{m=1}^M\bigg[1 - \dfrac{N}{2(N-2B+1)}\sqrt{\dfrac{(B+1)(2B + 9)}{2}}\max_{z \in \cC}\ep_{z}\bigg]. \label{eq:reg_lb_third_exp}
\end{align}

Recall that $\min\set{\ep_z}_{z \in \cC} = \ep$. We set $\ep = \dfrac{N-2B+1}{2\sqrt{2}N\sqrt{(B+1)(2B + 9)}}$. Then, we can bound $\ep_z$ as: $\ep \leq \ep_{z} \leq 2\ep \leq 1$ for all $z \in \cC$. Therefore, from \eqref{eq:reg_lb_third_exp} and \(\max_{z \in \cC}\ep_{z} = 2\ep = \dfrac{N-2B+1}{\sqrt{2}N\sqrt{(B+1)(2B + 9)}}\), we get:
\begin{align}
    &\max_{{\overline{S}}\in \sS_B} Reg_{\pi, \Theta_{\overline{S}}}(T) \nonumber\\ 
    &\geq \dfrac{uvB(N - 2B + 1)}{(B+1)^2(N-B+1)} \dfrac{N-2B+1}{2\sqrt{2}N\sqrt{(B+1)(2B + 9)}}\sum_{m=1}^M\bigg[1 - \dfrac{1}{2} \bigg] \nonumber\\
    &\geq \dfrac{uvB(N-2B+1)^2 M}{4\sqrt{2}N(B+1)^2(N-B+1)\sqrt{(B+1)(2B+9)}} \nonumber\\
    &= \dfrac{uvB(N-2B+1)^2 NT}{4\sqrt{2}N^2(B+1)^2(N-B+1)\sqrt{(B+1)(2B+9)}} \tag*{(using \mbox{$M = T$})} \nonumber \\
    &> \dfrac{uv B(N-2B+1)^2\log^2(NT)}{4\sqrt{2}N^2(B+1)^2(N-B+1)\sqrt{(B+1)(2B+9)}}\tag*{(using \mbox{$NT > \log^2(NT)$})}\nonumber\\
    &= \bigg(\dfrac{1}{N}\bigg) \bigg(\dfrac{N-2B+1}{4\sqrt{2}(B+1)}\bigg)u \bigg(\dfrac{1}{N(B+1)}\bigg)\bigg(\dfrac{N-2B+1}{\sqrt{(B+1)(2B+9)}}\bigg)v B\log^2(NT).\nonumber
\end{align}

Finally, we set $u = \min\bigg\{1, \dfrac{4\sqrt{2}(B+1)}{N-2B+1}\bigg\}$, and $v = \min\bigg\{\dfrac{1}{2}, \dfrac{\sqrt{(B+1)(2B+9)}}{N-2B+1}\bigg\}$ in the above expression to obtain the following lower bound on $\sup_{(\bU, \bV)} Reg_{\pi, (\bU, \bV)}(T)$:
\begin{align*}
    &\sup_{(\bU, \bV)} Reg_{\pi, (\bU, \bV)}(T) \nonumber\\
    &\geq \max_{\overline{S}\in \sS_B} Reg_{\pi, \Theta_{\overline{S}}}(T) \nonumber\\
    &\geq \min\bigg\{\dfrac{1}{N}, \dfrac{N-2B+1}{4\sqrt{2}N(B+1)}\bigg\} \min\bigg\{\dfrac{1}{N(B+1)}, \dfrac{N-2B+1}{2N(B+1)\sqrt{(B+1)(2B+9)}}\bigg\}B\log^2(NT)  \\ 
    &\geq \mathcal{B} \log^2(NT), \;\; \text{where $\mathcal{B} \in \R_+$.}
\end{align*}

\hfill $\square$ 

\section{Proof of Theorem \ref{thm_info_transparency}}
\label{ec_info_robust}

For the TWL problem, we have $\exp{(\mathbf{x}_{i, z}^{\top}\boldsymbol{\vartheta})} \leq 1$ under Assumption \ref{assumption:mnl}. This leads to $\mathbf{x}_{i, z}^{\top}\boldsymbol{\vartheta} \leq 0$, which is equivalent to $\sum_{j=1}^d x_{i, z}^j\vartheta_j \leq 0$. Note that for Assumption \ref{assumption:mnl} to be true for any value of $d$, we must have $x_{i, z}^j\vartheta_j \leq 0$ for each $j$. Let $\cx_{i, z}$ be the modified feature vector such that there is at least one feature from $\mathbf{x}_{i, z}$ which is now hidden by the platform manager in $\cx_{i, z}$. Suppose that the $k$-th feature is hidden, i.e., $x_{i, z}^k = 0$. Clearly, $x_{i, z}^k\vartheta_k = 0$ for $\cx_{i, z}$. Therefore, we get $(\cx_{i, z})^{\top}\boldsymbol{\vartheta} = \sum_{j=1}^d x_{i, z}^j\vartheta_j = \sum_{j \neq k} x_{i, z}^j\vartheta_j + x_{i, z}^k\vartheta_k = \sum_{j \neq k} x_{i, z}^j\vartheta_j \leq 0$, leading to $\exp{\left((\cx_{i, z})^{\top}\boldsymbol{\vartheta}\right)} \leq 1$. The above result is true even when several features are masked. This indicates that when the platform manager reveals only a fraction of the customer-specific information vector to the sellers, then also the sellers' revised preference parameters will continue to be bounded by 1, i.e., $\Tilde{u}_{i, z} = \exp{\left((\cx_{i, z})^{\top}\boldsymbol{\vartheta}\right)} \leq 1$, where $\Tilde{u}_{i, z}$ represents seller $i$'s revised preference parameter for customer of type $z$. Using this revised framework, we proceed to prove Theorem \ref{thm_info_transparency} below.

Recall the definitions of $\omega_{i, z}(t), \av_{z, i}(t), v_{z, i, t}^{\textup{UCB}}, \theta_{z, i}(m), \au_{i, z}(m)$, $u_{i, z, m}^{\textup{UCB}}$, $\cM(S_t, \bC(t-1), c_t)$, $\uM(S_t, \bC(t-1), c_t)$, $S_t^*$, and $\pS_t$ from \eqref{eq:assortment_count}, \eqref{eq:avg_proposal_count}, \eqref{eq:vucb}, \eqref{eq:proposal_set_size}, \eqref{eq:avg_match_count}, \eqref{eq:uucb}, \eqref{eq:deltaM}, \eqref{eq:deltaM_ucb}, \eqref{eq:optimal_assortment_oracle}, and \eqref{eq:optimistic_assortment}, respectively.

Using Lemma \ref{lemma_ucb}, we have each with probability at least $1-2/Nm$,
\begin{align*}
    \Tilde{u}_{i, z, m}^{\textup{UCB}} &\geq \Tilde{u}_{i, z},\\
    \Tilde{u}_{i, z, m}^{\textup{UCB}} - \E\big[\Bar{\Tilde{u}}_{i, z}(m)\big] &\leq \Tilde{\beta_3} \dfrac{\log{Nm}}{\theta_{z, i}(m)} + \Tilde{\beta_4} \sqrt{\dfrac{\log{Nm}}{\theta_{z, i}(m)}\E\big[\Bar{\Tilde{u}}_{i, z}(m)\big]}, 
\end{align*}
where $\Tilde{\beta_3}, \Tilde{\beta_4} \in \mathbb{R}_+$.

From Lemma \ref{lemma_optimal_value_bound} and \eqref{eq:deltaM_monotone_eq}, we have with probability at least $1-8/m$, 
\begin{align*}
    &\cM(S_t^*, \bC(t-1), c_t) - \cM(\pS_t, \bC(t-1), c_t) \\
    &\leq \sum_{i \in \pS_t}\bigg(q(\omega_{i, c_t}(t), v_{c_t, i}) + q(\theta_{c_t, i}(m), \Tilde{u}_{i, c_t}) + q(\omega_{i, c_t}(t), v_{c_t, i}) q(\theta_{c_t, i}(m), \Tilde{u}_{i, c_t})\bigg),  
\end{align*}
where $q(\theta_{c_t, i}(m), \Tilde{u}_{i, c_t}) \doteq \Tilde{\beta_3} \dfrac{\log{Nm}}{\theta_{c_t, i}(m)} + \Tilde{\beta_4} \sqrt{\dfrac{\log{Nm}}{\theta_{c_t, i}(m)}\Tilde{u}_{i, c_t}}$, and $q(\omega_{i, c_t}(t), v_{c_t, i})$ is defined in Lemma \ref{lemma_optimal_value_bound}. 

Incorporating the above results in \eqref{eq:regret_epoch} and simplifying as shown earlier in \S \ref{proof_thm_regret_upper_bound} , we get:
\begin{align*}
    &\E\bigg[\Delta V_m\big(b_m, \bC(b_m-1)\big)\bigg] \\
    &\leq  \sum_{t=b_m}^{mK}\gamma_m\dfrac{8}{m} + \gamma_m\sum_{t=b_m}^{mK}\big\{\cM(S_t^*, \bC(t-1), c_t) - \cM(\pS_t, \bC(t-1), c_t)\big\}\I\big[\cH_m^{\complement}\big]\\ 
    &\leq \dfrac{8\bR K}{m} + \bR\E\Bigg[\sum_{t=b_m}^{mK}\sum_{i \in \pS_t}\bigg\{q(\omega_{i, c_t}(t), v_{c_t, i}) + q(\theta_{c_t, i}(m), \Tilde{u}_{i, c_t}) + q(\omega_{i, c_t}(t), v_{c_t, i})q(\theta_{c_t, i}(m), \Tilde{u}_{i, c_t})\bigg\}\Bigg]\\
    &\leq \dfrac{8\bR K}{m} \\  
    &+ \bR\E\Bigg[\sum_{t=b_m}^{mK}\sum_{i \in \pS_t}\bigg\{\beta_1 \dfrac{\log{Nt}}{\omega_{i, c_t}(t)} + \beta_2 \sqrt{\dfrac{\log{Nt}}{\omega_{i, c_t}(t)}v_{c_t, i}} + \Tilde{\beta_3} \dfrac{\log{Nm}}{\theta_{c_t, i}(m)} + \Tilde{\beta_4} \sqrt{\dfrac{\log{Nm}}{\theta_{c_t, i}(m)}\Tilde{u}_{i, c_t}} \\
    &+ \bigg(\beta_1 \dfrac{\log{Nt}}{\omega_{i, c_t}(t)} + \beta_2 \sqrt{\dfrac{\log{Nt}}{\omega_{i, c_t}(t)}v_{c_t, i}}\bigg)\bigg(\Tilde{\beta_3} \dfrac{\log{Nm}}{\theta_{c_t, i}(m)} + \Tilde{\beta_4} \sqrt{\dfrac{\log{Nm}}{\theta_{c_t, i}(m)}\Tilde{u}_{i, c_t}}\bigg)\bigg\}\Bigg],
\end{align*}
which on further simplification, as shown in \S \ref{proof_thm_regret_upper_bound}, will yield the following:
\begin{align*}
    &\E\bigg[\Delta V_m\big(b_m, \bC(b_m-1)\big)\bigg] \nonumber\\ 
    &\leq \dfrac{8\bR K}{m}\nonumber \\ 
    &+ \beta \bR \E\Bigg[\log^2(mKN)\sum_{t =b_m}^{t = mK}\sum_{i \in \pS_t}\dfrac{1}{\theta_{c_t, i}^2(m)} + 2\log^{1.5}(mKN)\sum_{t =b_m}^{t = mK}\sum_{i \in \pS_t}\dfrac{1}{\theta_{c_t, i}^{1.5}(m)}\\ 
    &+ 3\log(mKN)\sum_{t =b_m}^{t = mK}\sum_{i \in \pS_t}\dfrac{1}{\theta_{c_t, i}(m)} + 2\sqrt{\log(mKN)}\sum_{t =b_m}^{t = mK}\sum_{i \in \pS_t}\dfrac{1}{\sqrt{\theta_{c_t, i}(m)}}\Bigg], 
\end{align*}
which is exactly the same as $\E\bigg[\Delta V_m\big(b_m, \bC(b_m-1)\big)\bigg]$ under complete information transparency (see \eqref{eq:exp_regret_epoch}). Upon substituting the above inequality in the expression of regret in \eqref{eq:regret_revisited} and simplifying the resulting regret expression as explained in \eqref{eq:regret_semifinal} and onward, we get:
\begin{align*}
    Reg{_{_\textup{TWL-UCB}}}(T) = \mathcal{O}\big(\log^2(NT)\big).
\end{align*}

Thus, \(Reg{_{_\textup{TWL-UCB}}}(T) = \sum_{m=1}^M\E\bigg[\Delta V_m\big(b_m, \bC(b_m-1)\big)\bigg]\) under partial information transparency, which is the same as in Theorem \ref{thm_regret_upper_bound}.

\hfill $\square$ 

\section{Proof of Theorem \ref{thm_cust_reward}}
\label{ec_customer_reward_robust}

{\color{black} 



Define $\underline{S}_t^*$ and $\underline{\pS}_t$ as the optimal assortments under the clairvoyant policy and the TWL-UCB policy, respectively, under the customer-dependent reward structure, such that $\underline{S}_t^*$ and $\underline{\pS}_t$ are, respectively, given by:
\begin{align}
    \underline{S}_t^* &= \argmax_{S_{t} \in \sS} \Bigg\{\gamma_{m, c_t}\cM(S_{t}, \bC(t-1), c_{t}) + \E\bigg[\underline{V}_m^*\big(t+1, \bC(t)\big) \mid t, \bC(t-1)\bigg]\Bigg\}, \label{eq:S_star} \\
    \underline{\pS}_t &= \argmax_{S_t \in \sS} \Bigg\{\sum_{i \in S_{t}}\gamma_{m, c_t}\dfrac{v_{c_{t}, i, t-1}^{\textup{UCB}}r_{i, c_{t}}^{\textup{UCB}}(b_m)}{1 + \sum_{i \in S_{t}}v_{c_t, i, t-1}^{\textup{UCB}}}\Bigg\}, \label{eq:S_hat}
\end{align}
where $\underline{V}_m^*\big(t+1, \bC(t)\big)$ is given by:
\begin{align}
    \underline{V}_m^*\big(t+1, \bC(t)\big) = \max_{S_{t+1} \in \sS} \Bigg\{ \E\bigg[\sum_{l = t+1}^{mK}\gamma_{m, c_l}\cM(S_{l}, \bC(l-1), c_{l}) \mid t+1, \bC(t)\bigg]\Bigg\}. \nonumber 
\end{align}

Using \eqref{eq:reward_epoch} and the arguments in the proof of Lemma \ref{lemma_exp_reward} (in \S \ref{supplementary_lemmas}), the total reward received by the platform in epoch $m$ under the customer-dependent reward structure is given by:
\begin{align}
    \mathcal{R}(m, \bU, \bV) &= \sum_{t=b_m}^{mK}\gamma_{m, c_t}\cM(S_t, \bC(t-1), c_t). \label{eq:reward_customer_dependent}
\end{align}

Using \eqref{eq:regret_def}, \eqref{eq:S_star}, \eqref{eq:S_hat}, and \eqref{eq:reward_customer_dependent}, the regret over a horizon of $T$ periods under the customer-dependent reward can therefore be bounded as follows: 
\begin{align}
    &Reg{_{_\textup{TWL-UCB}}}(T) \nonumber \\ 
    &= \sum_{m=1}^M\E\left[\sum_{t=b_m}^{mK}\gamma_{m, c_t}\left\{\cM(\underline{S}_t^*, \bC(t-1), c_t) - \cM(\underline{\pS}_t, \bC(t-1), c_t)\right\}\right] \nonumber\\
    &\leq \max_{m, z}\gamma_{m, z} \sum_{m=1}^M\E\left[\sum_{t=b_m}^{mK}\left\{\cM(\underline{S}_t^*, \bC(t-1), c_t) - \cM(\underline{\pS}_t, \bC(t-1), c_t)\right\}\right] \nonumber\\
    &= \max_{m, z}\gamma_{m, z}\sum_{m = 1}^M \E\Bigg[\sum_{t = b_m}^{mK} \bigg\{\cM(\underline{S}_t^*, \bC(t-1), c_t) - \cM(\underline{\pS}_t, \bC(t-1), c_t)\bigg\}\I\big[\cH_m + \cH_m^{\complement}\big]\Bigg] \nonumber \\
    &= \max_{m, z}\gamma_{m, z}\sum_{m=1}^M \E\left[\sum_{t = b_m}^{mK}  \bigg\{\cM(\underline{S}_t^*, \bC(t-1), c_t) - \cM(\underline{\pS}_t, \bC(t-1), c_t)\bigg\}\I\big[\cH_m\big] \right] \nonumber \\
    &+ \max_{m, z}\gamma_{m, z}\sum_{m=1}^M\E\left[\sum_{t = b_m}^{mK}\bigg\{\cM(\underline{S}_t^*, \bC(t-1), c_t) - \cM(\underline{\pS}_t, \bC(t-1), c_t)\bigg\}\I\big[\cH_m^{\complement}\big] \right], \label{eq:reg_v2_exp1}
\end{align}
where we split the total expected regret into two complementary probability regimes as discussed in \S \ref{proof_thm_regret_upper_bound}. 
}

In \S \ref{proof_thm_regret_upper_bound}, we show that $\prob(\cH_m) \leq 8/m$, so $\prob(\cH_m^{\complement}) \geq 1 - 8/m$. Therefore, from \eqref{eq:reg_v2_exp1}, we have:
\begin{align}
    &Reg{_{_\textup{TWL-UCB}}}(T) \nonumber \\ 
    &\leq \max_{m, z}\gamma_{m, z}\sum_{m=1}^M \E\Bigg[\sum_{t = b_m}^{mK}  \bigg\{\cM(\underline{S}_t^*, \bC(t-1), c_t) - \cM(\underline{\pS}_t, \bC(t-1), c_t)\bigg\}\I\big[\cH_m\big] \Bigg] \nonumber \\
    &+ \max_{m, z}\gamma_{m, z}\sum_{m=1}^M \E\Bigg[\sum_{t = b_m}^{mK}  \bigg\{\cM(\underline{S}_t^*, \bC(t-1), c_t) - \cM(\underline{\pS}_t, \bC(t-1), c_t)\bigg\}\I\big[\cH_m^{\complement}\big] \Bigg] \nonumber \\
    &\leq \max_{m, z}\gamma_{m, z}\sum_{m=1}^M \sum_{t=b_m}^{mK}\prob(\cH_m) \nonumber \\ 
    &+ \max_{m, z}\gamma_{m, z}\E\Bigg[ \sum_{m=1}^M\sum_{t=b_m}^{mK}\bigg(\sum_{i \in \pS_t} q(\omega_{i, c_t}(t), v_{c_t, i}) + q(\theta_{c_t, i}(m), u_{i, c_t}) + q(\omega_{i, c_t}(t), v_{c_t, i})q(\theta_{c_t, i}(m), u_{i, c_t})\bigg) \Bigg] \nonumber\\ 
    &\text{(using \mbox{$\cM(S_t, \bC(t-1), c_t) \leq 1 \; \forall S_t$}, Lemma \ref{lemma_optimal_value_bound}, and \eqref{eq:deltaM_monotone_eq} as discussed in \S \ref{proof_thm_regret_upper_bound})} \nonumber \\
    &\leq \max_{m, z}\gamma_{m, z}\sum_{m=1}^M \sum_{t=b_m}^{mK}\dfrac{8}{m} \nonumber \\ 
    &+ \max_{m, z}\gamma_{m, z}\E\Bigg[\sum_{m=1}^M\sum_{t=b_m}^{mK}\bigg(\sum_{i \in \pS_t}q(\omega_{i, c_t}(t), v_{c_t, i}) + q(\theta_{c_t, i}(m), u_{i, c_t}) + q(\omega_{i, c_t}(t), v_{c_t, i})q(\theta_{c_t, i}(m), u_{i, c_t})\bigg) \Bigg] \nonumber\\
    &\leq \max_{m, z}\gamma_{m, z}8K(1 + \log{T}) \nonumber\\ 
    &+ \max_{m, z}\gamma_{m, z}\E\Bigg[\sum_{m=1}^M\sum_{t=b_m}^{mK}\bigg(\sum_{i \in \pS_t}q(\omega_{i, c_t}(t), v_{c_t, i}) + q(\theta_{c_t, i}(m), u_{i, c_t}) + q(\omega_{i, c_t}(t), v_{c_t, i})q(\theta_{c_t, i}(m), u_{i, c_t})\bigg) \Bigg]. \nonumber\\
    &\text{\Bigg(using \mbox{$\sum_{m=1}^{M} \dfrac{1}{m} \leq 1 + \log(M) \leq 1 + \log(T)$}\Bigg)} \nonumber
\end{align}

In the above expression, we can simplify the term, 
\begin{align*}
    \E\Bigg[\sum_{m=1}^M\sum_{t=b_m}^{mK}\sum_{i \in \pS_t}q(\omega_{i, c_t}(t), v_{c_t, i}) + q(\theta_{c_t, i}(m), u_{i, c_t}) + q(\omega_{i, c_t}(t), v_{c_t, i})q(\theta_{c_t, i}(m), u_{i, c_t})\Bigg], \nonumber 
\end{align*} 
similar to the way presented in \S \ref{proof_thm_regret_upper_bound} (see \eqref{eq:regret_epoch} and thereafter for detailed analysis). Combining all the results and setting $\max_{m, z}\gamma_{m, z} = \Hat{\gamma}$, we get the following:
\begin{align}
    Reg{_{_\textup{TWL-UCB}}}(T) = \mathcal{O}\big(\log^2(NT)\big).\nonumber
\end{align}

\hfill $\square$ 

\end{document}